\documentclass{article}

\usepackage{microtype}
\usepackage{graphicx}
\usepackage{subfigure}
\usepackage{booktabs} 

\usepackage{hyperref}

\usepackage{natbib}

\usepackage{algorithm}
\usepackage{algorithmic}

\usepackage[usenames,dvipsnames]{xcolor}
\usepackage{tikz,pgfplots}
\pgfplotsset{compat=newest}
\usetikzlibrary{arrows,automata}
\usetikzlibrary{arrows.meta,calc,decorations.markings,math,arrows.meta}

\usepackage{amsmath,amsfonts,amssymb}
\usepackage{mathtools}
\usepackage{amsthm}
\usepackage{bm}
\usepackage{bbm}
\usepackage{cases}

\usepackage{enumitem}
\usepackage{tabularx}
    \newcolumntype{L}{>{\raggedright\arraybackslash}X}
    \newcolumntype{C}{>{\centering\arraybackslash}X}

\usepackage{xspace}
\usepackage{wrapfig}

\usepackage[colorinlistoftodos, textwidth=20mm]{todonotes}
\definecolor{citrine}{rgb}{0.89, 0.82, 0.04}
\definecolor{blued}{RGB}{70,197,221}

\newcommand{\M}{\mathcal M}
\newcommand{\A}{\mathcal A}
\renewcommand{\O}{\mathcal O}

\newcommand{\calS}{\mathcal S}

\renewcommand{\Re}{\mathbb R}
\newcommand{\nextstates}{\Gamma}

\newcommand{\MR}{D^\text{MR}}
\newcommand{\MD}{D^\text{MD}}
\newcommand{\SR}{\Pi^\text{SR}}
\newcommand{\SD}{\Pi^\text{SD}}
\newcommand{\PiC}{\Pi_c}

\newcommand{\SP}[1]{sp\left\{{#1}\right\}}
\newcommand{\proj}[1]{\Gamma_c{#1}}

\newcommand{\norm}[2][\infty]{\left\|#2\right\|_{#1}}
\DeclareMathOperator*{\argmax}{\arg\,\max}
\DeclareMathOperator*{\argmin}{\arg\,\min}

\newcommand{\transp}{\mathsf{T}}
\newcommand{\opT}[1]{T_c#1}
\newcommand{\opN}[1]{N_c{#1}}
\newcommand{\opG}[1]{G_c{#1}}

\newcommand{\evi}{{\small\textsc{EVI}}\xspace}
\newcommand{\regal}{{\small\textsc{Regal}}\xspace}
\newcommand{\regalc}{{\small\textsc{Regal.C}}\xspace}
\newcommand{\regald}{{\small\textsc{Regal.D}}\xspace}
\newcommand{\regopt}{{\small\textsc{ScOpt}}\xspace}
\newcommand{\ucrl}{{\small\textsc{UCRL}}\xspace}
\newcommand{\psrl}{{\small\textsc{PSRL}}\xspace}

\newcommand{\scal}{{\small\textsc{SCAL}}\xspace}

\newcommand{\rmaxbound}{r_{\max}}

\newcommand{\wt}[1]{\widetilde{#1}}
\newcommand{\wh}[1]{\widehat{#1}}

\newcommand{\wb}[1]{\overline{#1}}

\usepackage{accents}
\DeclareMathAccent{\wtilde}{\mathord}{largesymbols}{"65}
\newcommand{\uwidetilde}[1]{\underaccent{\wtilde}{#1}}
\newcommand{\uwt}[1]{\underaccent{\wtilde}{#1}}

\newcommand{\pluseq}{\mathrel{+}=}

\DeclareMathOperator*{\supp}{supp}

\DeclareRobustCommand{\eg}{e.g.,\@\xspace}
\DeclareRobustCommand{\ie}{i.e.,\@\xspace}
\DeclareRobustCommand{\wrt}{w.r.t.\@\xspace}

\DeclareRobustCommand{\st}{s.t.\@\xspace}

\newtheorem{theorem}{Theorem}
\newtheorem{lemma}[theorem]{Lemma}
\newtheorem{proposition}[theorem]{Proposition}

\newtheorem*{remark}{Remark}

\theoremstyle{remark}

\theoremstyle{theorem}
\newtheorem{definition}{Definition}

\newtheorem{example}{Example}
\newtheorem{assumption}[theorem]{Assumption}

\newlength{\minipagewidth}
\newlength{\minipagewidthx}
\newcommand{\bookboxx}[1]{\small
\par\medskip\noindent
\framebox[0.99\textwidth]{
\begin{minipage}{0.97\dimexpr\textwidth-\parindent\relax} {#1} \end{minipage} } \par\medskip }

\usepackage[normalem]{ulem} 

\usepackage[accepted]{icml2018}

\icmltitlerunning{Efficient Bias-Span-Constrained Exploration-Exploitation}

\makeatletter
\g@addto@macro\normalsize{%
  \setlength\abovedisplayskip{5pt}
  \setlength\belowdisplayskip{5pt}
}
\makeatother

\begin{document} 

\twocolumn[
\icmltitle{Efficient Bias-Span-Constrained Exploration-Exploitation\\ in Reinforcement Learning}



\icmlsetsymbol{equal}{*}

\begin{icmlauthorlist}
        \icmlauthor{Ronan Fruit}{equal,inria}
        \icmlauthor{Matteo Pirotta}{equal,inria}
        \icmlauthor{Alessandro Lazaric}{fair}
        \icmlauthor{Ronald Ortner}{leoben}
\end{icmlauthorlist}

\icmlaffiliation{inria}{SequeL Team, INRIA Lille, France}
\icmlaffiliation{fair}{Facebook AI Research, Paris, France}
\icmlaffiliation{leoben}{Montanuniversit\"{a}t Leoben, Austria}

\icmlcorrespondingauthor{Ronan Fruit}{ronan.fruit@inria.fr}

\icmlkeywords{boring formatting information, machine learning, ICML}

\vskip 0.3in
]



\printAffiliationsAndNotice{\icmlEqualContribution} 

\begin{abstract}
We introduce \scal, an algorithm designed to perform efficient exploration-exploitation in any \emph{unknown weakly-communicating} Markov decision process (MDP) for which an upper bound $c$ on the span of the \emph{optimal bias function} is known. For an MDP with $S$ states, $A$ actions and $\nextstates \leq S$ possible next states, we prove a regret bound of $\widetilde{O}(c\sqrt{\Gamma SAT})$, which significantly improves over existing algorithms (e.g., \ucrl and \psrl), whose regret scales linearly with the MDP \emph{diameter} $D$. In fact, the optimal bias span is finite and often much smaller than $D$ (e.g., $D=\infty$ in non-communicating MDPs). A similar result was originally derived by~\citet{Bartlett2009regal} for \regalc, for which no tractable algorithm is available. In this paper, we \emph{relax} the optimization problem at the core of \regalc, we carefully analyze its properties, and we provide the first \emph{computationally efficient algorithm} to solve it. 
Finally, we report numerical simulations supporting our theoretical findings and showing how \scal significantly outperforms \ucrl in MDPs with \emph{large} diameter and \emph{small} span.

\end{abstract}


\vspace{-0.3in}
\section{Introduction}
\vspace{-0.05in}

While learning in an unknown environment, a reinforcement learning (RL) agent must trade off the \textit{exploration} needed to collect information about the dynamics and reward, and the \textit{exploitation} of the experience gathered so far to gain as much reward as possible. In this paper, we focus on the regret framework~\citep{Jaksch10}, which evaluates the exploration-exploitation performance by comparing the rewards accumulated by the agent and an optimal policy. A common approach to the exploration-exploitation dilemma is the \emph{optimism in face of uncertainty} (OFU) principle: the agent maintains optimistic estimates of the value function and, at each step, it executes the policy with highest optimistic value~\citep[\eg][]{Brafman:2003:RGP:944919.944928,Jaksch10,Bartlett2009regal}. 
An alternative approach is posterior sampling~\citep{thompson1933likelihood}, which maintains a Bayesian distribution over MDPs (\ie dynamics and expected reward) and, at each step, samples an MDP and executes the corresponding optimal policy~\citep[\eg][]{Osband2013more,Abbasi-Yadkori2015bayesian,Osband2017posterior,Ouyang2017learning,Agrawal2017posterior}.

Given a finite MDP with $S$ states, $A$ actions, and diameter~$D$ (i.e., the time needed to connect any two states), \citet{Jaksch10} proved that no algorithm can achieve regret smaller than $\Omega(\sqrt{DSAT})$. While recent work successfully closed the gap between upper and lower bounds w.r.t.\ the dependency on the number of states \citep[\eg][]{Agrawal2017posterior,pmlr-v70-azar17a}, 
relatively little attention has been devoted to the dependency on $D$. While the diameter quantifies the number of steps needed to ``recover'' from a bad state in the worst case, the actual regret incurred while ``recovering'' is related to the difference in potential reward between ``bad'' and ``good'' states, which is accurately measured by the span (\ie the range) $\SP{h^*}$ of the optimal bias function~$h^*$. While the diameter is an upper bound on the bias span, it could be arbitrarily larger (e.g., weakly-communicating MDPs may have finite span and infinite diameter) thus suggesting that algorithms whose regret scales with the span may perform significantly better.\footnote{The proof of the lower bound relies on the construction of an MDP whose diameter actually coincides with the bias span (up to a multiplicative numerical constant), thus leaving the open question whether the ``actual'' lower bound depends on $D$ or the bias span. See \citep{Osband2016} for a more thorough discussion.} Building on the idea that the OFU principle should be \textit{mitigated} by the bias span of the optimistic solution, \citet{Bartlett2009regal} proposed three different algorithms (referred to as \regal) achieving regret scaling with $\SP{h^*}$ instead of~$D$. The first algorithm defines a span regularized problem, where the regularization constant needs to be carefully tuned depending on the state-action pairs visited in the future, which makes it unfeasible in practice. Alternatively, they propose a constrained variant, called \regalc, where the regularized problem is replaced by a constraint on the span. Assuming that an upper-bound $c$ on the bias span of the optimal policy is known (\ie $\SP{h^*} \leq c$), \regalc achieves regret upper-bounded by $\wt{\mathcal{O}}(\min\{D,c\}S\sqrt{AT})$. Unfortunately, they do not propose any computationally tractable algorithm solving the constrained optimization problem, which may even be ill-posed in some cases. Finally, \regald avoids the need of knowing the \textit{future} visits by using a doubling trick, but still requires solving a regularized problem, for which no computationally tractable algorithm is known.

In this paper, we build on \regalc and propose a constrained optimization problem for which we derive a computationally efficient algorithm, called \regopt. We identify conditions under which \regopt converges to the optimal solution and propose a suitable stopping criterion to achieve an $\varepsilon$-optimal policy. Finally, we show that using a slightly modified optimistic argument, the convergence conditions are always satisfied and the learning algorithm obtained by integrating \regopt into a \ucrl-like scheme (resulting into \scal) achieves regret scaling as $\wt{\mathcal{O}}(\min\{D,c\}\sqrt{\nextstates SAT})$ when an upper-bound $c$ on the optimal bias span is available, thus providing the first computationally tractable algorithm that can solve weakly-communicating MDPs.


\vspace{-0.1in}
\section{Preliminaries}
\vspace{-0.05in}


We consider a finite \emph{weakly-communicating} Markov decision process \citep[Sec. 8.3]{puterman1994markov} $M = \langle \calS, \A, r, p\rangle$ with a set of states $\calS$ and a set of actions $\A=\bigcup_{s \in \calS} \A_s$. Each state-action pair $(s,a)\in \mathcal{S}\times \mathcal{A}_s$ is characterized by a reward distribution with mean $r(s,a)$ and support in $[0, \rmaxbound]$ as well as a transition probability distribution $p(\cdot|s,a)$ over next states.
We denote by $S = |\calS|$ and $A = \max_{s\in\mathcal{S}}|\A_s|$ the number of states and actions, and by $\nextstates$ the maximum support of all transition probabilities.
A Markov randomized \emph{decision rule} $d: \calS \rightarrow \mathcal{P}(\A)$ maps states to distributions over actions. The corresponding set is denoted by $\MR$, while the subset of Markov deterministic decision rules is $\MD$. A stationary \emph{policy} $\pi = (d, d, \ldots)$ $=:d^\infty$ repeatedly applies the same decision rule $d$ over time. 
The set of stationary policies defined by Markov randomized (resp. deterministic) decision rules is denoted by $\SR(M)$ (resp. $\SD(M)$). The \emph{long-term average reward} (or \textit{gain}) of a policy $\pi\in\SR(M)$ starting from $s\in\calS$ is
\begin{align*}
        g^\pi_M(s) := \lim_{T\to +\infty} \mathbb{E}_{\mathbb{Q}}\Bigg[ \frac{1}{T}\sum_{t=1}^T r(s_t,a_t) \Bigg],
\end{align*}
where $\mathbb{Q} := \mathbb{P}\left(\cdot|a_t \sim \pi(s_t); s_0=s; M\right)$. Any stationary policy $\pi\in\SR$ has an associated bias function defined as
\begin{align*}
        h^\pi_M(s) := \underset{T\to +\infty}{C\text{-}\lim}~\mathbb{E}_{\mathbb{Q}}\Bigg[\sum_{t=1}^{T} \big(r(s_t,a_t) - g_M^\pi(s_t)\big)\Bigg],
\end{align*}
that measures the expected total difference between the reward and the stationary reward in \emph{Cesaro-limit}\footnote{For policies with an aperiodic chain, the standard limit exists.} (denoted $C\text{-}\lim$). Accordingly, the difference of bias values $h^\pi_M(s)-h^\pi_M(s')$ quantifies the (dis-)advantage of starting in state $s$ rather than $s'$. In the following, we drop the dependency on $M$ whenever clear from the context and
denote by $\SP{h^\pi} := \max_s h^\pi(s) - \min_{s} h^\pi(s)$ the \emph{span} of the bias function. 
In weakly communicating MDPs, any optimal policy $\pi^* \in \argmax_\pi g^\pi(s)$ has \emph{constant} gain, \ie $g^{\pi^*}(s) = g^*$ for all $s\in\calS$. Let $P_d\in \Re^{S\times S}$ and $r_d \in \Re^S$ be the transition matrix and reward vector associated with decision rule $d\in\MR$. We denote by $L_d$ and $L$ the Bellman operator associated with $d$ and \emph{optimal} Bellman operator
\begin{align*}
       \forall v \in \Re^S,~~ L_dv := r_d + P_d v; \quad Lv := \max_{d\in\MR} \big\{ r_d + P_d v \big\}.
\end{align*}
%
For any policy $\pi = d^\infty \in \Pi^{SR}$, the gain $g^\pi$ and bias $h^\pi$ satisfy the following system of \emph{evaluation equations}
\begin{equation}\label{eq:eval.equations}
g^\pi = P_d g^\pi; \quad h^\pi = L_d h^\pi - g^\pi.
\end{equation}
Moreover, there exists a policy $\pi^* \in \argmax_\pi g^\pi(s)$ for which $(g^*,h^*) = (g^{\pi^*},h^{\pi^*})$ satisfy the \emph{optimality equation}
\begin{equation}\label{eq:optimality.equation}
                h^* = L h^* - g^* e, \quad\text{ where }\;\; e = (1,\dots,1)^\intercal.
\end{equation}
Finally, we denote by 
	$D := \max_{(s,s') \in \calS\times \calS, s\neq s'} \{\tau_{M}(s \to s')\}$
the diameter of $M$,
where $\tau_{M}(s\to s')$ is the minimal expected number of steps needed to reach $s'$ from $s$ in $M$.
 
\textbf{Learning problem.} Let $M^*$ be the true \emph{unknown} MDP. We consider the learning problem where $\mathcal{S}$, $\mathcal{A}$ and $\rmaxbound$ are \emph{known}, while rewards $r$ and transition probabilities $p$ are \emph{unknown} and need to be estimated on-line. We evaluate the performance of a learning algorithm $\mathfrak{A}$ after $T$ time steps by its cumulative \emph{regret} $\Delta(\mathfrak{A},T) = T g^* - \sum_{t=1}^T r_t(s_t,a_t)$.

\vspace{-0.05in}
\section{Optimistic Exploration-Exploitation}\label{sec:problem.statement}
\vspace{-0.05in}

Since our proposed algorithm \scal (Sec.~\ref{sec:regret}) is a tractable variant of \regalc and thus a modification of \ucrl, we first recall their common structure summarized in Fig.~\ref{fig:ucrl.constrained}.

\vspace{-0.05in}
\subsection{Upper-Confidence Reinforcement Learning}
\vspace{-0.05in}

\ucrl proceeds through episodes $k=1,2\dots$ At the beginning of each episode $k$, \ucrl computes a set of plausible MDPs defined as
$
\mathcal M_k = \big\{ M = \langle \calS, \A, \wt{r}, \wt{p}\rangle : 
        \wt{r}(s,a) \in B_r^k(s,a),~
        \wt{p}(s'|s,a) \in B_p^k(s,a,s'),
        \sum_{s'} \wt{p}(s'|s,a) = 1
\big\}$, where $B_r^k$ and $B_p^k$ are high-probability confidence intervals on the rewards and transition probabilities of the true MDP $M^*$, which guarantees that $M^* \in \mathcal{M}_k$ w.h.p.
We use confidence intervals constructed using empirical Bernstein's inequality~\citep{audibert2007tuning,Maurer2009empirical}
\begin{align*}
\beta_{r,k}^{sa}
&:=
         \sqrt{\frac{14 
                         \wh{\sigma}_{r,k}^2(s,a)
                        b_{k,\delta}
                }{\max\lbrace 1, N_k(s,a)\rbrace}} + \frac{\frac{49}{3}
\rmaxbound
        b_{k,\delta}
        }{\max\lbrace 1, N_k(s,a)-1 \rbrace},\\
\beta_{p,k}^{sas'}
&:= 
\sqrt{\frac{14 \wh{\sigma}_{p,k}^2(s'|s,a) 
                        b_{k,\delta}
                }{\max\lbrace 1, N_k(s,a)\rbrace}} +\frac{\frac{49}{3}
        b_{k,\delta}
        }{\max\lbrace 1,N_k(s,a)-1 \rbrace},
\end{align*}
where $N_k(s,a)$ is the number of visits in $(s, a)$ before episode $k$, $\wh{\sigma}_{r,k}^2(s,a)$ and $\wh{\sigma}_{p,k}^2(s'|s,a)$ are the empirical variances of $r(s,a)$ and ${p}(s'|s,a)$ 
and $b_{k,\delta} = \ln(2 SA t_k/\delta)$.
Given the empirical averages $\wh{r}_k(s,a)$ and $\wh{p}_k(s'|s,a)$ of rewards and transitions, we define $\mathcal{M}_k$ by $B_r^k(s,a) := [\wh{r}_k(s,a) - \beta_{r,k}^{sa}, \wh{r}_k(s,a) + \beta_{r,k}^{sa}] \cap [0, \rmaxbound]$ and $B_p^k(s,a,s') := [\wh{p}_k(s'|s,a) - \beta_{p,k}^{sas'},\wh{p}_k(s'|s,a) + \beta_{p,k}^{sas'}] \cap [0,1]$.

Once $\mathcal{M}_k$ has been computed, \ucrl finds an approximate solution $(\wt{M}_k^*,\wt\pi_k^*)$ to the optimization problem
\begin{align}\label{eq:optimistic.policy}
( \wt{M}^*_k, \wt\pi^*_k ) \in \argmax_{M\in\mathcal{M}_k, \pi\in\SD(M)} g^\pi_M.
\end{align}
Since $M^* \in \mathcal{M}_k$ w.h.p., it holds that ${g}^*_{\wt{M}_k} \geq  {g}^*_{M^*}$.
As noticed by~\citet{Jaksch10}, problem~\eqref{eq:optimistic.policy} is equivalent to finding $\wt{\mu}^* \in \argmax_{\mu \in \SD(\wt{\mathcal{M}}_k)} \big\lbrace g^\mu_{\wt{\mathcal{M}}_k} \big\rbrace$ where $\wt{\mathcal{M}}_k$ is the \emph{extended} MDP (sometimes called \emph{bounded-parameter} MDP) implicitly defined by $\mathcal{M}_k$. More precisely, in $\wt{\mathcal{M}}_k$ the (finite) action space $\A$ is ``extended'' to a compact action space $\wt{\A}_k$ by considering every possible value of the confidence intervals $B_r^k(s,a)$ and $B_p^k(s,a,s')$ as fictitious actions. The equivalence between the two problems comes from the fact that for each $\wt{\mu} \in \SD(\wt{\mathcal{M}}_k)$ there exists a pair ($\wt{M}, \wt{\pi}$) such that the policies $\wt{\pi}$ and $\wt{\mu}$ induce the same Markov reward process on respectively $\wt{M}$ and $\wt{\mathcal{M}}_k$, and conversely. Consequently, \eqref{eq:optimistic.policy} can be solved by running so-called \textit{extended} value iteration (\evi): starting from an initial vector $u_0 =0$, \evi recursively computes
\begin{align}\label{eq:evi}
u_{n+1}(s) \!=\! \max_{a,\wt{r}, \wt{p}} \big[ \wt{r}(s,a) + \wt{p}(\cdot|s,a)^\transp u_n\big] \!=\! \wt{L} u_n(s),
\end{align}
%
where $\wt{L}$ is the \textit{optimistic} optimal Bellman operator associated to $\wt{\mathcal{M}}_k$. 
%
If \evi is stopped when $\SP{u_{n+1}-u_n} \leq \varepsilon_k$, then the greedy policy $\wt{\mu}_{k}$ \wrt $u_{n}$ is guaranteed to be $\varepsilon_k$-optimal, \ie $g^{\wt{\mu}_{k}}_{\wt{\mathcal{M}}_k} \geq {g}^*_{\wt{\mathcal{M}}_k} - \varepsilon_k \geq {g}^*_{M^*} - \varepsilon_k $. Therefore, the policy $\wt{\pi}_k$ associated to $\wt{\mu}_{k}$ is an \emph{optimistic} $\varepsilon_k$\emph{-optimal} policy, and \ucrl executes $\wt{\pi}_k$ until the end of episode $k$.

\begin{figure}[t]
\renewcommand\figurename{\small Figure}
\begin{minipage}{\columnwidth}
\bookboxx{
\textbf{Input:} Confidence $\delta \in ]0,1[$, $r_{\max}$, $\calS$, $\A$, a constant $c \geq 0$


\noindent \textbf{For} episodes $k=1, 2, ...$ \textbf{do}

\begin{enumerate}[leftmargin=4mm,itemsep=0mm]
\item Set $t_k = t$ and episode counters $\nu_k (s,a) = 0$.

\item Compute estimates $\wh{p}_k(s' | s,a)$, $\wh{r}_k(s,a)$ and a confidence set~$\mathcal{M}_k$ (\ucrl, \regalc), resp.~$\mathcal{M}_k^\ddagger$ (\scal).

\item
Compute an $\rmaxbound/\sqrt{t_k}$-approximation $\widetilde{\pi}_k$ of the solution of Eq.~\ref{eq:optimistic.policy} (\ucrl), resp. Eq.~\ref{eq:optimistic.policy.regalc} (\regalc), resp. Eq.~\ref{eq:optimistic.policy.scal} (\scal).

%
\item Sample action $a_t \sim \wt{\pi}_k(\cdot|s_t)$.

\item \textbf{While} $\nu_k(s_t,a_t) \leq \max\{1, N_k(s_t,a_t)\}$ \textbf{do}
\begin{enumerate}[leftmargin=4mm,itemsep=-1mm]
        \item Execute $a_t$, obtain reward $r_{t}$, and observe next state $s_{t+1}$.
        \item Set $\nu_k (s_t,a_t) \pluseq 1$.
        \item Sample action $a_{t+1} \sim \wt{\pi}_k(\cdot|s_{t+1})$ and set $t \pluseq 1$.
\end{enumerate}

\item Set $N_{k+1}(s,a) = N_{k}(s,a)+ \nu_k(s,a)$.
\end{enumerate}
}
 \vspace{-0.1in}
\caption{\small The general structure of optimistic algorithms for RL.}
\label{fig:ucrl.constrained}
\end{minipage}
\vspace{-0.2in}
\end{figure}

\vspace{-0.05in}
\subsection{A first relaxation of \regalc}
\vspace{-0.05in}
\regalc follows the same steps as \ucrl but instead of solving problem~\eqref{eq:optimistic.policy}, it tries to find the best \emph{optimistic} model $\wt{M}^*_{\textsc{RC}} \in \mathcal{M}_{\textsc{RC}}$ having constrained \emph{optimal} bias span \ie
\begin{align}\label{eq:optimistic.policy.regalc}
        ( \wt{M}^*_{\textsc{RC}},\wt{\pi}^*_{\textsc{RC}}) = \argmax_{M \in \mathcal{M}_{\textsc{RC}}, \pi \in \SD(M)} g^{\pi}_M,
\end{align}
where $\mathcal{M}_{\textsc{RC}} := \{ M \in \mathcal{M}_k  :  \SP{h^{*}_M} \leq c\}$ is the set of plausible MDPs with bias span of the \emph{optimal} policy bounded by $c$.
Under the assumption that $\SP{h^*_{M^*}} \leq c$, \regalc discards any MDP $M \in \mathcal{M}_k$ whose \textit{optimal} policy has a span larger than $c$ (\ie $\SP{h^{*}_M} > c$) and
otherwise looks for the MDP with highest \emph{optimal} gain $g^{*}(M)$. Unfortunately, there is no guarantee that all MDPs in $\mathcal{M}_{\textsc{RC}}$ are weakly communicating and thus have constant gain. As a result, we suspect this problem to be ill-posed (\ie the maximum is most likely not well-defined).
Moreover, even if it is well-posed, searching the space $\mathcal{M}_{\textsc{RC}}$ seems to be computationally intractable. Finally, for any $M \in \mathcal{M}_k$, there may be several optimal policies with different bias spans and some of them may not satisfy the optimality equation~\eqref{eq:optimality.equation} and are thus difficult to compute.

In this paper, we slightly modify problem~\eqref{eq:optimistic.policy.regalc} as follows:
%
\begin{align}\label{eq:optimistic.policy.relaxedregalc}
        (\wt{M}_c^*, \wt\pi_c^*) \in \argmax_{M\in\mathcal{M}_k, \pi\in\PiC(M)} g^\pi_{M},
\end{align}

\vspace{-0.1in}
where the search space of policies is defined as
\begin{equation*}\label{E:Z_C}
        \PiC(M) := \left\{ \pi\in \SR: \SP{h^\pi_M} \leq c \; \wedge \;  \SP{g^\pi_M} = 0  \right\},
\end{equation*}
and $\max_{\pi \in \PiC(M)} \{g^\pi_M\} = -\infty$ if $\PiC(M) = \emptyset$.
Similarly to \eqref{eq:optimistic.policy}, problem~\eqref{eq:optimistic.policy.relaxedregalc} is equivalent to solving $\wt{\mu}^*_c \in \argmax_{\mu \in \PiC(\wt{\mathcal{M}}_k)} \big\lbrace g^\mu_{\wt{\mathcal{M}}_k} \big\rbrace$. Unlike~\eqref{eq:optimistic.policy.regalc}, for \emph{every} MDP in $\mathcal{M}_k$ (not just those in $\mathcal{M}_{\textsc{RC}}$), \eqref{eq:optimistic.policy.relaxedregalc} considers \emph{all} (stationary) policies with \emph{constant gain} satisfying the span constraint (not just the deterministic optimal policies).
Since $g^\pi_M$ and $\SP{h^\pi_M}$ are in general non-continuous functions of ($M$, $\pi$), the argmax in \eqref{eq:optimistic.policy.regalc} and \eqref{eq:optimistic.policy.relaxedregalc} may not exist. 
Nevertheless, by reasoning in terms of supremum value, we can show that~\eqref{eq:optimistic.policy.relaxedregalc} is always a \emph{relaxation} of~\eqref{eq:optimistic.policy.regalc} (where we enforce the additional constraint of constant gain).
\begin{proposition}\label{prop:relaxation}
        Define the  following restricted set of MDPs $\mathcal{E}_k = \mathcal{M}_{\textsc{RC}} \cap \{M \in \mathcal{M}_k:~ \SP{g^{*}_M} = 0\}$.
        Then
        \[
                \sup_{M\in\mathcal{E}_k, \pi \in \SD} g^{\pi}_M \leq \sup_{M \in \mathcal{M}_k, \pi \in \PiC(M)} g^\pi_M.
        \]
\end{proposition}
\vspace{-0.2in}
\begin{proof}
        The result follows from the fact that $\mathcal{E}_k \subseteq \mathcal{M}_k$ and $\forall M \in \mathcal{E}_k$, $\argmax_{\pi \in \SD} \{g^{\pi}_M\} \subseteq\PiC(M)$.
\end{proof}
\vspace{-0.2in}
As a result, the \emph{optimism} principle is preserved when moving from~\eqref{eq:optimistic.policy.regalc} to~\eqref{eq:optimistic.policy.relaxedregalc} and since the set of admissible MDPs $\M_k$ is the same, any algorithm solving~\eqref{eq:optimistic.policy.relaxedregalc} would enjoy the same regret guarantees as \regalc. In the following we further characterise problem~\eqref{eq:optimistic.policy.relaxedregalc}, introduce a \emph{truncated} value iteration algorithm to solve it, and finally integrate it into a \ucrl-like scheme to recover \regalc regret guarantees.

\vspace{-0.1in}
\section{The Optimization Problem}\label{sec:optimization}
\vspace{-0.05in}

In this section we analyze some properties of the following optimization problem, of which~\eqref{eq:optimistic.policy.relaxedregalc} is an instance,
%
\begin{equation}\label{P:opt_probl_well_posed}
\begin{aligned}
        \sup_{\pi \in \PiC(M)} \left\{ g^\pi_M \right\},
\end{aligned}
\end{equation}
where $M$ is any MDP (with discrete or compact action space) \st $\Pi_c(M) \neq \emptyset$.
%
Problem~\eqref{P:opt_probl_well_posed} aims at finding a policy that maximizes the gain $g^\pi_M$ within the set of randomized policies with constant gain (\ie $\SP{g^\pi_M} = 0$) and bias span smaller than $c$ (\ie $\SP{h^\pi_M} \leq c$). Since $g^\pi_M \in [0, \rmaxbound]$ the supremum always exists and we denote it by $g^*_c(M)$. The set of maximizers is denoted by $\PiC^*(M) \subseteq \PiC(M)$, with elements $\pi_c^*(M)$ (if $\PiC^*(M)$ is non-empty).

In order to give some intuition about the solutions of problem~\eqref{P:opt_probl_well_posed}, we introduce the following illustrative MDP.

%
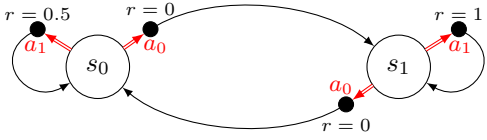
\begin{figure}[t]
        \centering
\begin{tikzpicture}
	\tikzset{VertexStyle/.style = {draw, 
									shape          = circle,
	                                text           = black,
	                                inner sep      = 2pt,
	                                outer sep      = 0pt,
	                                minimum size   = 24 pt}}
	\tikzset{VertexStyle2/.style = {shape      $\SR$ to $\Re^{S}$    = circle,
	                                text           = black,
	                                inner sep      = 2pt,
	                                outer sep      = 0pt,
	                                minimum size   = 14 pt}}
	\tikzset{Action/.style = {draw, 
                					shape          = circle,
	                                text           = black,
	                                fill           = black,
	                                inner sep      = 2pt,
	                                outer sep      = 0pt}}
	                                 
	\node[VertexStyle](s0) at (0,0) {$ s_{0} $};
	\node[Action](a0s0) at (.7,.5){};
	\node[Action](a1s0) at (-.8,0.5){};
	\node[VertexStyle](s1) at (4,0){$s_1$};
	\node[Action](a0s1) at (3.3,-0.5){};
	\node[Action](a1s1) at (4.8,0.5){};
    
	\draw[->, >=latex, double, color=red](s0) to node[midway, right, yshift=-0.2em]{{\small $a_0$}} (a0s0);
	\draw[->, >=latex](a0s0) to [out=30,in=140,looseness=0.8] (s1);
	\draw[->, >=latex, double, color=red](s0) to node[midway, left, yshift=-0.2em]{{\small $a_1$}} (a1s0);
	\draw[->, >=latex](a1s0) to [out=210,in=210,looseness=2] (s0);

	\draw[->, >=latex, double, color=red](s1) to node[midway, left, yshift=.2em]{{\small $a_0$}} (a0s1);
	\draw[->, >=latex](a0s1) to [out=210,in=-40,looseness=0.8] (s0);
	\draw[->, >=latex, double, color=red](s1) to node[midway, right, yshift=-0.2em]{{\small $a_1$}} (a1s1);
	\draw[->, >=latex](a1s1) to [out=-30,in=-30,looseness=2] (s1);
    
    \node [above, yshift=0.2em] at (a0s0) {\scriptsize $r=0$};
    \node [below, yshift=-0.2em] at (a0s1) {\scriptsize $r=0$};
    \node [above] at (a1s0) {\scriptsize $r=0.5$};
    \node [above] at (a1s1) {\scriptsize $r=1$};
\end{tikzpicture}
\vspace{-0.1in}
\caption{\label{F:2states_nondetpol}
        Toy example with deterministic transitions and reward for all actions.
}
\vspace{-0.1in}
\end{figure}

\begin{example} \label{E:twostates}
Consider the two-states MDP depicted in Fig.~\ref{F:2states_nondetpol}. For a generic stationary policy $\pi\in\SR$ with decision rule $d \in \MR$ we have that
        \begin{equation*}
                d = \begin{bmatrix}
                        x & 1-x\\
                        y & 1-y
                \end{bmatrix};
                \enspace
                P_d = \begin{bmatrix}
                        1-x & x\\
                        y & 1-y
                \end{bmatrix}\!,\;
                r_d = \begin{bmatrix}
                        \frac{1-x}{2}\\
                        1-y
                \end{bmatrix}.
        \end{equation*}

We can compute the gain $g = [g_1, g_2]$ and the bias $h = [h_1, h_2]$ by solving the linear system~\eqref{eq:eval.equations}.
For any $x>0$ or $y>0$, we obtain
        \begin{align*}
                g_1 \!=\! g_2 \!=\! \frac{1}{2} + x\frac{1-3y}{2(x+y)}; \quad h_2 - h_1 \!=\! \frac{1}{2} + \frac{1-3y}{2(x+y)},
        \end{align*}
while for $x=0$, $y=0$, we have $g_1 = 1/2$ and $g_2 = 1$, with $h_2=h_1 = 0$. Note that $0 \leq \SP{h^\pi} \leq 1$ for any $\pi \in\SR$. 
In the following, we will use this example choosing particular values for $x$, $y$, and $c$ to illustrate some important properties of optimization problem ~\eqref{P:opt_probl_well_posed}.
\end{example}

\textbf{Randomized policies.} 
The following lemma shows that, unlike in unconstrained gain maximization where there always exists an optimal deterministic policy, the solution of~\eqref{P:opt_probl_well_posed} may indeed be a randomized policy.

\begin{lemma}\label{lem:stoch_opt_policy}
There exists an MDP $M$ and a scalar $c \geq 0$, such that $\PiC^*(M) \neq \emptyset$ and $\PiC^*(M) \cap \SD(M) = \emptyset$.
\end{lemma}

\begin{proof}
Consider Ex.~\ref{E:twostates} with constraint $1/2 < c < 1$. The only deterministic policy $\pi_D$ with constant gain and bias span smaller than $c$ is defined by the decision rule with $x=0$ and $y=1$, which leads to $g^{\pi_D} = 1/2$ and $\SP{h^{\pi_D}} = 1/2$. On the other hand, a randomized policy $\pi_R$ can satisfy the constraint and maximize the gain by taking $x=1$ and $y = (1-c)/(1+c)$, which gives $\SP{h^{\pi_R}} = c$ and $g^{\pi_R} = c > g^{\pi_D}$, thus proving the statement.
\end{proof}

\textbf{Constant gain.} The following lemma shows that if we consider non-constant gain policies, the supremum in~\eqref{P:opt_probl_well_posed} may not be well defined, as no \emph{dominating} policy exists. A policy $\pi\in\SR$ is \emph{dominating} if for any policy $\pi'\in\SR$, $g^\pi(s) \geq g^{\pi'}(s)$ in all states $s\in\calS$.

\begin{lemma}\label{lem:constant.gain}
There exists an MDP $M$ and a scalar $c \geq 0$, such that there exists no dominating policy $\pi$ in $\SR$ with constrained bias span (i.e., $\SP{h^\pi} \leq c$).
\end{lemma}

\begin{proof}
Consider Ex.~\ref{E:twostates} with constraint $1/2 < c < 1$. As shown in the proof of Lem.~\ref{lem:stoch_opt_policy}, the optimal stationary policy $\pi_R$ with constant gain has $g^*_c =[c,c]$. On the other hand, the only policy $\pi$ with non-constant gain is $x=0$, $y=0$, which has $\SP{h^\pi}=0 < c$ and $g^\pi(s_0) = 1/2 < c = g^*_c$ and $g^\pi(s_1) = 1 > c = g^*_c$, thus proving the statement.
\end{proof}

On the other hand, when the search space is restricted to policies with constant gain, the optimization problem is well posed. 
Whether problem~\eqref{P:opt_probl_well_posed} always admits a maximizer is left as an open question. The main difficulty 
comes from the fact that, in general, $\pi \mapsto g^\pi$ is not a continuous map and $\PiC$ is not a closed set. For instance in Ex.~\ref{E:twostates}, although the maximum is attained, the point $x=0$, $y=0$ does not belong to $\PiC$ (\ie $\PiC$ is not closed) and $ g^\pi$ is not continuous at this point. Notice that when the MDP is \emph{unichain} \citep[Sec. 8.3]{puterman1994markov}, $\PiC$ is compact, $g^\pi$ is continuous, and we can prove the following lemma (see App.~\ref{app:optim_problem_formulation}):
\begin{lemma}\label{prop:P.well.defined}
If $M$ is unichain then $\PiC^*(M) \neq \emptyset$.
\end{lemma}

We will later show that for the specific instances of \eqref{P:opt_probl_well_posed} that are encountered by our algorithm \scal, Lem.~\ref{prop:P.well.defined} holds.
\vspace{-0.05in}
\section{Planning with \regopt} \label{sec:planning}
\vspace{-0.05in}

In this section, we introduce \regopt and derive sufficient conditions for its convergence to the solution of~\eqref{P:opt_probl_well_posed}. In the next section, we will show that these assumptions always hold when \regopt is carefully integrated into \ucrl (while in App.~\ref{app:counterexamples} we show that they may not hold in general). 

\vspace{-0.05in}
\subsection{Span-constrained value and policy operators}\label{ssec:operators}
\vspace{-0.05in}

\begin{figure}[t]
\renewcommand\figurename{\small Figure}
\begin{minipage}{\columnwidth}
\bookboxx{
\textbf{Input:} Initial vector $v_0 \in \mathbb{R}^S$, reference state $\overline{s}\in \mathcal{S}$, contractive factor $\gamma \in (0,1)$, accuracy $\varepsilon \in (0,+\infty)$\\
\textbf{Output:} Vector $v_n \in \mathbb{R}^S$, policy $\pi_n=(G_c v_n)^\infty$
 \vspace{-0.05in}
\begin{enumerate}
 \item Initialize $n =0$ and $v_{1} = \opT{v_{0}} - (\opT{v_{0}})(\wb{s})e$,
 \vspace{-0.05in}
 \item \noindent \textbf{While} $\SP{v_{n+1}-v_n} + \frac{2\gamma^n}{1-\gamma}\SP{v_1-v_0} > \varepsilon$ \textbf{do}
 \vspace{-0.05in}
 \begin{enumerate}
 \item $n \pluseq 1$.
 \item $v_{n+1} = \opT{v_{n}} - (\opT{v_{n}})(\wb{s})e$.
 \end{enumerate}
\end{enumerate}
}
 \vspace{-0.1in}
\caption{\small Algorithm \regopt.}
\label{fig:regopt}
\vspace{-0.2in}
\end{minipage}
\end{figure}

\regopt is a version of (relative) value iteration~\citep{puterman1994markov,bertsekas1995dynamic}, where the optimal Bellman operator is modified to return value functions with span bounded by $c$, and the stopping condition is tailored to return a \emph{constrained greedy} policy with near-optimal gain. We first introduce a \emph{constrained} version of the optimal Bellman operator $L$.

 \begin{definition}\label{def:opT}
Given $v \in \Re^{S}$ and $c \geq 0$, we define the value operator $\opT{}:\Re^{S}\rightarrow \Re^{S}$ as
        \begin{align}\label{eq:opT}
                \opT{v} = \begin{cases}
                        Lv(s) & \forall s \in \overline{\calS}(c,v),\\
                        c + \min_s \{Lv(s)\} & \forall s \in \mathcal{S} \setminus \overline{\calS}(c,v),
                \end{cases}
        \end{align}
        where $\overline{\calS}(c,v) = \left\{s \in \mathcal{S} | Lv(s) \leq \min_s \{Lv(s)\} + c \right\}$. 
\end{definition}

In other words, operator $\opT{}$ applies a \emph{span truncation} to the one-step application of $L$, that is, for any state $s \in \mathcal{S}$, $\opT{v}(s) = \min\{Lv(s), \min_{x}Lv(x) + c\}$, which guarantees that $\SP{\opT{v}} \leq c$. Unlike $L$, operator $\opT{}$ is not always associated with a decision rule $d$ \st $\opT{v} =L_d v$ (see App.~\ref{app:counterexamples}). We say that $\opT{}$ is \emph{feasible} at $v \in \mathbb{R}^S$ and $s\in\calS$ if there exists a distribution $\delta^+_v(s) \in \mathcal{P}(A)$ such that
\begin{equation}\label{eq:dec.rule.opT}
\opT{v}(s) = \sum_{a\in \mathcal{A}_s} \delta_v^+(s,a) \big[ r(s,a) + p(\cdot|s,a)^\transp v \big].
\end{equation}
When a distribution $\delta^+_v(s)$ exists in all states, we say that $\opT{}$ is \emph{globally feasible} at $v$, and $\delta^+_v$ is its associated decision rule, i.e., $\opT{v} =L_{\delta^+_v} v$.
In the following lemma, we identify sufficient and necessary conditions for (global) \emph{feasibility}.

\begin{lemma}\label{L:TRpolicyexistence}
Operator $\opT{}$ is \emph{feasible} at $v\in \mathbb{R}^S$ and $s\in\calS$ if and only if
\begin{equation}\label{E:Top_policy_cond}
\min_{a\in \A_s} \{ r(s,a) + p(\cdot|s,a)^\transp v \} \leq \min_{s'} \{L v(s') \}+ c.
\end{equation}
Furthermore, let 
\begin{equation}
D(c,v) := \left\{d \in\MR \;|\; \SP{L_d v} \leq c \right\}
\end{equation}
be the set of randomized decision rules $d$ whose associated operator $L_d$ returns a span-constrained value function when applied to $v$. Then, $\opT{v}$ is \emph{globally feasible} if and only if $D(c,v) \neq \emptyset$, in which case we have
\begin{equation}\label{eq:opT.LP}
\opT{v} = \max_{\delta\in D(c,v)} L_{\delta} v, \quad\text{ and }\quad \delta^+_v \in \argmax_{\delta\in D(c,v)} L_{\delta} v.
\end{equation}
%
%
%
\end{lemma}

\vspace{-0.1in}
The last part of this lemma shows that when $\opT{}$ is globally feasible at $v$ (\ie $D(c,v) \neq \emptyset$), $\opT{v} = L_{\delta_v^+}v$ is the \emph{componentwise maximal} value function of the form $L_\delta v$ with decision rule $\delta \in \MR$ satisfying $\SP{L_\delta v } \leq c$. Surprisingly, even in the presence of a constraint on the one-step value span, such a \emph{componentwise} maximum still exists (which is not as straightforward as in the case of the greedy operator $L$). Therefore, whenever $D(c,v) \neq \emptyset$, optimization problem~\eqref{eq:opT.LP} can be seen as an LP-problem (see App.~\ref{app:optim_problem_formulation2}).

%
\begin{definition}\label{def:Gop}
Given $v \in \Re^{S}$ and $c \geq 0$, let $\widetilde{\calS}(c,v)$ be the set of states where $\opT{v}$ is \emph{feasible} (condition~\eqref{E:Top_policy_cond}) with $\delta^+_v(s)$ be the associated decision rule (Eq.~\ref{eq:dec.rule.opT}). We define the operator $\opG{}: \Re^{S}\rightarrow \MR$ as\footnote{When there are several policies $\delta_v^+$ achieving $\opT{v}(s) = L_{\delta_v^+} v(s)$ in state $s\in \mathcal{S}$, $\opG{}$ chooses an arbitrary decision rule.}
        \begin{align*}
                \opG{v} \!=\! \begin{cases}
                        \delta^+_v(s) & \!\!s \in \widetilde{\calS}(C,v),\\
                        \argmin\limits_{a \in \mathcal{A}_s} \left\{ r(s,a) + p(\cdot|s,a)^\transp v\right\} & \!\! s \in \mathcal{S} \!\setminus\! \widetilde{\calS}(C,v).
                \end{cases}
        \end{align*}
\end{definition}
%
As a result, if $\opT{}$ is globally feasible at $v$, by definition $\opG{v} = \delta^+_v$. Note that computing $\delta^+_v$ is \emph{not} significantly more difficult than computing a greedy policy (see App.~\ref{sec:T.policy} for an \emph{efficient implementation}).

We are now ready to introduce \regopt (Fig.~\ref{fig:regopt}). Given a vector $v_0\in\Re^S$ and a reference state $\wb s$, \regopt implements relative value iteration where $L$ is replaced by $T_c$, \ie
 \begin{align}\label{eq:rel.val.iteration}
 v_{n+1} = \opT{v_{n}} - \opT{v_{n}}(\wb{s})e.
 \end{align}
Notice that the term $(\opT{v_{n}})(\wb{s})e$ subtracted at any iteration~$n$ prevents $v_n$ from increasing linearly with $n$ and thus avoids numerical instability. However, the subtraction can be dropped without affecting the convergence properties of \regopt. If the stopping condition is met at iteration $n$, \regopt returns policy $\pi_n = d_n^\infty$ where $d_n = \opG{v_n}$.


\vspace{-0.05in}
\subsection{Convergence and Optimality Guarantees}\label{ssec:guarantees}
\vspace{-0.05in}

In order to derive convergence and optimality guarantees for \regopt we need to analyze the properties of operator~$\opT{}$. 
We start by proving that $\opT{}$ preserves the one-step \emph{span contraction} properties of $L$.


\begin{assumption}\label{asm:L.contraction}
The optimal Bellman operator $L$ is a 1-step $\gamma$-span-contraction, \ie there exists a $\gamma <1$ such that for any vectors $u,v \in \Re^{S}$, $\SP{Lu - Lv} \leq \gamma \SP{u-v}$.\footnote{In the undiscounted setting, if the MDP is unichain, $L$ is a $J$-stage contraction with $S\geq J \geq 1$.}
\end{assumption}


\begin{lemma}\label{lem:contraction}
Under Asm.~\ref{asm:L.contraction}, $\opT$ is a $\gamma$-span contraction.
\end{lemma}

The proof of Lemma~\ref{lem:contraction} relies on the fact that the truncation of $L$ in the definition of $\opT{}$ is non-expansive in span semi-norm. 
Details are given in App.~\ref{app:operators}, where it is also shown that $\opT{}$ preserves other properties of $L$ such as \emph{monotonicity} and \emph{linearity}. It then follows that $\opT{}$ admits a fixed point solution to an optimality equation (similar to $L$) and thus \regopt converges to the corresponding bias and gain, the latter being an upper-bound on the optimal solution of~\eqref{P:opt_probl_well_posed}. We formally state these results in Lem.~\ref{lem:optimality.equation.T}.

\begin{lemma}\label{lem:optimality.equation.T}
Under Asm.~\ref{asm:L.contraction}, the following properties hold:
 \begin{enumerate}
 \item \emph{Optimality equation and uniqueness:} There exists a solution $(g^+,h^+) \in \Re \times \Re^{S}$ to the optimality equation
  \begin{align}\label{E:optimality_eq_average_T}
   \opT{h^+} = h^+ + g^+ e.
  \end{align}   
  If $(g,h) \in \Re \times \Re^{S}$ is another solution of~\eqref{E:optimality_eq_average_T}, then 
  $g = g^+ $ and there exists $\lambda  \in \Re$ \st $h = h^+ + \lambda  e$.
  \item \emph{Convergence:} For any initial vector $v_0 \in \Re^{S}$, the sequence $(v_n)$ generated by \regopt converges to a solution vector $h^+$ of the optimality equation~\eqref{E:optimality_eq_average_T}, and
  \[ \lim_{n\to +\infty} \opT{}^{n+1}v_0-\opT{}^n v_0 = g^+e.\]
  \item \emph{Dominance:}
  The gain $g^+$ is an upper-bound on the supremum of~\eqref{P:opt_probl_well_posed}, i.e., $g^+ \geq g^*_c$.
 \end{enumerate}
\end{lemma}

A direct consequence of point 2 of Lem.~\ref{lem:optimality.equation.T} (convergence)
is that \regopt always stops after a finite number of iterations. Nonetheless, $\opT{}$ may not always be globally feasible at $h^+$ (see App.~\ref{app:counterexamples}) and thus there may be no policy associated to optimality equation~\eqref{E:optimality_eq_average_T}. Furthermore, even when there is one, Lem.~\ref{lem:optimality.equation.T} provides no guarantee on the performance of the policy returned by \regopt after a finite number of iterations. To overcome these limitations, we introduce an additional assumption, which leads to stronger performance guarantees for \regopt.

\begin{assumption}\label{asm:feasibility}
        Operator $\opT{}$ is globally feasible at any vector $v \in \mathbb{R}^S$ such that $\SP{v} \leq c$.
\end{assumption}

\begin{theorem}\label{thm:opT.optimality}
        Assume Asm.~\ref{asm:L.contraction} and~\ref{asm:feasibility} hold and let $\gamma$ denote the contractive factor of $\opT{}$ (Asm.~\ref{asm:L.contraction}). For any $v_0 \in \mathbb{R}^S$ such that $\SP{v_0}\leq c$, any $\overline{s} \in \mathcal{S}$ and any $\varepsilon >0$, the policy $\pi_n$ output by \regopt{$(v_0,\overline{s},\gamma,\varepsilon)$} is such that $\| g^+e - g^{\pi_n} \|_{\infty} \leq \varepsilon$. Furthermore, if in addition the policy $\pi^+ = (\opG{h^+})^\infty$ is \emph{unichain}, $g^+$ is the solution to optimization problem~\eqref{P:opt_probl_well_posed} \ie $g^+ = g^*_c$ and $\pi^+ \in \Pi^*_c$.
%
%
\end{theorem}
%

The first part of the theorem shows that the stopping condition used in Fig.~\ref{fig:regopt} ensures that \regopt returns an $\varepsilon$-optimal policy $\pi_n$.
Notice that while $\SP{h^+} = \SP{\opT{h^+}} \leq c$ by definition of $\opT$, in general when the policy $\pi^+ = (\opG{h^+})^\infty$ associated to $h^+$ is \emph{not unichain}, we might have $\SP{h^+} < sp\{h^{\pi^+}\}$. On the other hand, Corollary 8.2.7. of \citet{puterman1994markov} ensures that if $\pi^+$ is unichain then $\SP{h^+} = sp\{h^{\pi^+}\}$, hence the second part of the theorem. Notice also that even if $\pi^+$ is unichain, we cannot guarantee that $\pi_n$ satisfies the span constraint, \ie $\SP{h^{\pi_n}}$ may be arbitrary larger than $c$. Nonetheless, in the next section, we show that the definition of $\opT{}$ and Thm.~\ref{thm:opT.optimality} are sufficient to derive regret bounds when \regopt is integrated into \ucrl.


\vspace{-0.05in}
\section{Learning with \scal}\label{sec:regret}
\vspace{-0.05in}
In this section we introduce \scal, an optimistic online RL algorithm that employs \regopt to compute policies that efficiently balance exploration and exploitation.
We prove that the assumptions stated in Sec.~\ref{ssec:guarantees} hold when \regopt is integrated into the optimistic framework.
Finally, we show that \scal enjoys the same regret guarantees as \regalc, while being the first implementable and efficient algorithm to solve bias-span constrained exploration-exploitation.

Based on Def.~\ref{def:opT}, we define $\wt{\opT{}}$ as the \emph{span truncation} of the optimal Bellman operator $\wt{L}$ of the bounded-parameter MDP $\wt{\mathcal{M}}_k$  (see Sec.~\ref{sec:problem.statement}). 
Given the structure of problem~\eqref{eq:optimistic.policy.relaxedregalc}, one might consider applying \regopt (using $\wt{\opT{}}$) to the extended MDP $\wt{\mathcal{M}}_k$.
Unfortunately, in general $\wt L$ does not satisfy Asm.~\ref{asm:L.contraction} and~\ref{asm:feasibility} and thus $\wt{\opT{}}$ may not enjoy the properties of Lem.~\ref{lem:optimality.equation.T} and Thm.~\ref{thm:opT.optimality}. To overcome this problem, we slightly modify $\wt{\mathcal{M}}_k$ as described in Def.~\ref{def:augmented.perturbed.mdp}.

\begin{definition}\label{def:augmented.perturbed.mdp}
Let $\wt{\mathcal{M}}$ be a bounded-parameter (extended) MDP. Let $1 \geq \eta >0$ and $\overline{s} \in \mathcal{S}$ an arbitrary state. We define the ``modified'' MDP $\wt{\mathcal{M}}^{\ddagger}$ associated to $\wt{\mathcal{M}}$ by\footnote{For any closed interval $[a,b] \subset \Re$, $\max \lbrace[a,b]\rbrace := b$ and $\min\lbrace[a,b]\rbrace := a$}
%
 \begin{align*}
         B^{\ddagger}_r(s,a) &= [0, \max\{B_r(s,a)\}], \\
 B_p^{\ddagger}(s,a,s') &= \begin{cases}B_p(s,a,s')& \text{if}~~s'\neq \overline{s}, \\ B_p(s,a,\overline{s}) \cap [\eta, 1] &\text{otherwise}, \end{cases}
 \end{align*}
where we assume that $\eta$ is small enough so that: $B_p(s,a,\overline{s}) \cap [\eta, 1] \neq \emptyset$, $\sum_{s' \in \mathcal{S}} \min\{B_p^{\ddagger}(s,a,s')\} \leq 1$, and $\sum_{s' \in \mathcal{S}} \max\{B^{\ddagger}_p(s,a,s')\} \geq 1$.
 We denote by $\wt{L}^\ddagger$ the optimal Bellman operator of $\wt{\mathcal{M}}^{\ddagger}$ (cf.\ Eq.~\ref{eq:evi}) and by $\wt{T}_c^\ddagger$ the span truncation of $\wt{L}^\ddagger$ (cf.\ Def.~\ref{def:opT}).
\end{definition}

By slightly \emph{perturbing} the confidence intervals $B_p$ of the transition probabilities, we enforce that the ``\emph{attractive}'' state $\wb{s}$ is reached with non-zero probability from any state-action pair $(s,a)$ implying that the \emph{ergodic coefficient} of $\wt{\mathcal{M}}^{\ddagger}$
\[
\gamma = 1 - \min_{
        \substack{
s,u \in \mathcal{S},~ a, b \in \mathcal{A}\\
\wt{p},~ \wt{q} \in B_p^\ddagger
}
} 
\left\lbrace \sum_{j \in \mathcal{S}}\underbrace{\min \left\{ \wt{p}(j|s,a),\wt{q}(j|u,b) \right\}}_{\geq \eta ~\text{if}~ j=\overline{s}} \right\rbrace
\]

\vspace{0.1in}
is smaller than $1-\eta <1$, so that $\wt{L}^\ddagger$ is $\gamma$-contractive~\citep[Thm. 6.6.6]{puterman1994markov}, \ie Asm.~\ref{asm:L.contraction} holds. Moreover, for any policy $\pi \in \SR(\wt{\mathcal{M}}^{\ddagger})$, state $\overline{s}$ necessarily belongs to all \emph{recurrent classes} of $\pi$ implying that $\pi$ is unichain and so $\wt{\mathcal{M}}^{\ddagger}$ is unichain. As is shown in Thm.~\ref{th:augmented.perturbed.mdp}, the $\eta$-perturbation of $B_p$ introduces a small bias $\eta c$ in the final gain. 

By \emph{augmenting} (without perturbing) the confidence intervals $B_r$ of the rewards, we ensure two nice properties. First of all, for any vector $v\in \mathbb{R}^S$, $\wt{L}v = \wt{L}^\ddagger v$ and thus by definition $\wt{T}_cv = \wt{T}_c^\ddagger v$. Secondly, there exists a decision rule $\delta \in \MR(\wt{\mathcal{M}}^{\ddagger})$ such that $\forall s \in \calS$, $\wt{r}_\delta^\ddagger(s) = 0$ meaning that $sp\{\wt{L}_\delta^\ddagger v\} = sp\{\wt{P}_\delta^\ddagger v\}\leq \SP{v}$ \citep[Proposition 6.6.1]{puterman1994markov}. Thus if $\SP{v} \leq c$ then $sp\{\wt{L}_\delta^\ddagger v\} \leq c$ and so $\delta \in \wt{D}^\ddagger(c,v) \neq \emptyset$ which by Lem.~\ref{L:TRpolicyexistence} implies that $\wt{T}_c^\ddagger $ is globally feasible at $v$. Therefore, Asm.~\ref{asm:feasibility} holds in $\wt{\mathcal{M}}^{\ddagger}$. 

When combining both the perturbation of $B_p$ and the augmentation of $B_r$ we obtain Thm.~\ref{th:augmented.perturbed.mdp} (proof in App.~\ref{app:augmented.perturbed.mdp}).


\begin{theorem}\label{th:augmented.perturbed.mdp}
        Let $\wt{\mathcal{M}}$ be a bounded-parameter (extended) MDP and $\wt{\mathcal{M}}^{\ddagger}$ its ``modified'' counterpart (see Def.~\ref{def:augmented.perturbed.mdp}).
 Then
 \begin{enumerate}
         \item $\wt{L}^\ddagger$ is a $\gamma$-span contraction with $\gamma \leq 1 - \eta <1$ (\ie Asm.~\ref{asm:L.contraction} holds) and thus Lem.~\ref{lem:optimality.equation.T} applies to $\wt{T}_c^\ddagger$. Denote by $(g^+,h^+)$ a solution to equation~\eqref{E:optimality_eq_average_T} for $\wt{T}_c^\ddagger$.
         \item $\wt{T}_c^\ddagger$ is globally feasible at any $v \in \mathbb{R}^S$ \st $\SP{v} \leq c$ (\ie Asm.~\ref{asm:feasibility} holds) and $\wt{\mathcal{M}}^{\ddagger}$ is unichain implying that $\pi^+ = \opG{h^+}$ is unichain. Thus Thm.~\ref{thm:opT.optimality} applies to $\wt{T}_c^\ddagger$.
         \item 
        $
                 \forall \mu \in \PiC(\wt{\mathcal{M}}), \quad g^+ = g_c^*(\wt{\mathcal{M}}^\ddagger) \geq g^\mu(\wt{\mathcal{M}}) - \eta c.
        $
 \end{enumerate}
\end{theorem}

\scal (cf.\ Fig.~\ref{fig:ucrl.constrained}) is a variant of UCRL that applies \regopt (instead of \evi, see Eq.~\ref{eq:evi}) on the bounded parameter MDP~$\wt{\mathcal{M}}_k^{\ddagger}$ (instead of $\wt{\mathcal{M}}_k$, cf.\ step 2 in Fig.~\ref{fig:ucrl.constrained}) in each episode~$k$ to solve the optimization problem
\begin{align}\label{eq:optimistic.policy.scal}
        \max_{M\in\wt{\mathcal{M}}^{\ddagger}_k, \pi\in\PiC(M)} g^\pi_{M},
\end{align}
whose maximum is denoted by $g_c^*(\wt{\mathcal{M}}^\ddagger_k)$. The intervals $B_p^{\ddagger}$ of $\wt{\mathcal{M}}_k^{\ddagger}$ are constructed using parameter\footnote{Notice that given that $\beta_{p,k}^{sa} \geq \eta_k$ for all $(s,a) \in \calS \times \mathcal{A}$ (see definition in Sec.~\ref{sec:problem.statement}), the assumptions of Def.~\ref{def:augmented.perturbed.mdp} hold trivially.}  $\eta_k = \rmaxbound/(c\cdot t_k)$ and an arbitrary attractive state $\overline{s} \in \mathcal{S}$. \regopt is run at step~3 in Fig.~\ref{fig:ucrl.constrained} with an initial value function $v_0 = 0$, the same reference state $\wb s$ used for the construction of $B_p^{\ddagger}$, contraction factor $\gamma_k =1-\eta_k$, and accuracy $\varepsilon_k = \rmaxbound/\sqrt{t_k}$. \regopt finally returns an optimistic (nearly) optimal policy satisfying the span constraint. This policy is executed until the end of the episode. 

Thm.~\ref{th:augmented.perturbed.mdp} ensures that the specific instance of problem~\eqref{eq:optimistic.policy.relaxedregalc} for \scal (i.e., problem~\eqref{eq:optimistic.policy.scal}) is well defined and admits a maximizer $\pi_c^*(\wt{\mathcal{M}}_k^\ddagger)$ that can be efficiently computed using \regopt. Moreover, up to an accuracy $\eta_k \cdot c = \rmaxbound/t_k$, policy $\pi_c^*(\wt{\mathcal{M}}_k^\ddagger)$ is still optimistic \wrt all policies in the set of constrained policies $\Pi_c(\wt{\mathcal{M}}_k)$ for the \emph{initial} extended MDP. Since the true (unknown) MDP $M^*$ belongs to $\mathcal{M}_k$ with high probability, under the assumption that $\SP{h^{*}_{M^*}}\leq c$, $g_c^*(\wt{\mathcal{M}}_k^\ddagger) \geq g^*_{M^*}-\rmaxbound/t_k$. As briefly mentioned in Sec.~\ref{sec:planning}, in practice \regopt can only output an approximation $\wt{\mu}_k$ of $\pi_c^*(\wt{\mathcal{M}}_k^\ddagger)$ and we have no guarantees on $\SP{h^{\wt{\mu}_k}}$. However, the regret proof of \scal only uses the fact that $\SP{v_n}\leq c$ and this is always satisfied by definition of $\wt{T}_c^\ddagger$.
We are now ready to prove the following regret bound (see App.~\ref{app:regret.reg.ucrl}).

\begin{theorem}\label{thm:regret.scal}
 For any \emph{weakly communicating} MDP $M$ such that $\SP{h^{*}_M}\leq c$, with probability at least $1-\delta$ it holds that for any $T\geq 1$, the regret of \scal is bounded as
 \begin{align*}
         \Delta(\scal,T) = \mathcal{O} \left( \max \lbrace \rmaxbound, c \rbrace \sqrt{\nextstates S A T \ln \left( \frac{T}{\delta} \right)} \right),
 \end{align*}
 where $\nextstates= \max_{s\in\calS,a\in\A}\|p(\cdot|s,a)\|_0  \leq S$ is the maximal number of states that can be reached from any state.
\end{theorem}

The previous bound shows that when $c \leq \rmaxbound D$, \scal scales linearly with $c$, while \ucrl scales linearly with $\rmaxbound D$ (all other terms being equal).
Notice that the gap between $\SP{h^{*}}$ and $D$ can be arbitrarily large, and thus the improvement can be significant in many MDPs. As an extreme case, in weakly communicating MDPs the diameter can be infinite, leading \ucrl to suffer linear regret, while \scal is still able to achieve sub-linear regret. However when $c > \rmaxbound D$, given that the true MDP $M^*$ may not belong to $\mathcal{M}_k^\ddagger$, we cannot guarantee that the span of the value function $v_n$ returned by \regopt is bounded by $\rmaxbound D$. Nevertheless, we can slightly modify \scal to address this case: at the beginning of any episode $k$, we run both \regopt (with the same inputs) and \evi (as in \ucrl) in parallel and pick the policy associated to the value with smallest span. With this modification, \scal enjoys the best of both worlds, \ie the regret scales with $\min\lbrace \max\{\rmaxbound,c\}, \rmaxbound D \rbrace$ instead of $c$.
%
When $c$ is wrongly chosen ($c < \SP{h^*_{M^*}}$), \scal converges to a policy in $\Pi^*_c(M^*)$ which can be arbitrarily worse than the true optimal policy in $M^*$.
For this reason we cannot prove a regret bound in this scenario.
Finally, notice that the benefit of \scal over \ucrl comes at a negligible additional computational cost.

\section{Numerical Experiments}

\begin{figure*}[t]
        \centering
    \begin{minipage}{0.34\textwidth}
        \subfigure[]{
                \begin{tikzpicture}
                        \begin{scope}[scale=0.85]
	\tikzset{VertexStyle/.style = {draw, 
									shape          = circle,
	                                text           = black,
	                                inner sep      = 2pt,
	                                outer sep      = 0pt,
	                                minimum size   = 24 pt}}
	\tikzset{VertexStyle2/.style = {shape          = circle,
	                                text           = black,
	                                inner sep      = 2pt,
	                                outer sep      = 0pt,
	                                minimum size   = 14 pt}}
	\tikzset{Action/.style = {draw, 
                					shape          = circle,
	                                text           = black,
	                                fill           = black,
	                                inner sep      = 2pt,
	                                outer sep      = 0pt}}
	                                 
	\node[VertexStyle](s0) at (0,0) {$ s_{0} $};
	\node[Action](a0s0) at (.7,.7){};
	\node[VertexStyle](s1) at (2.5,0){$s_1$};
	\node[Action](a0s1) at (1.5,0.){};
	\node[VertexStyle](s2) at (5,0){$s_2$};
	\node[Action](a0s2) at (4.3,-0.7){};
	\node[Action](a1s2) at (5,1){};
    
	\draw[->, >=latex, double, color=red](s0) to node[midway, right]{{\small $a_0$}} (a0s0);
    \draw[->, >=latex](a0s0) to [out=30,looseness=0.8] node[midway, xshift=1em]{$\delta$} (s1);   
	\draw[->, >=latex](a0s0) to [out=30,in=120,looseness=0.8] node[above]{$1-\delta$} (s2);

	\draw[->, >=latex, double, color=red](s1) to node[above,xshift=0.2em]{{\small $a_0$}} (a0s1);
	\draw[->, >=latex](a0s1) to (s0);
    
	\draw[->, >=latex, double, color=red](s2) to node[midway, left]{{\small $a_0$}} (a0s2);
    \draw[->, >=latex](a0s2) to [out=210, in=330, looseness=0.8] node[midway,xshift=-1em]{$\delta$} (s1);   
	\draw[->, >=latex](a0s2) to [out=210, in=300, looseness=0.8] node[above]{$1-\delta$} (s0);
	\draw[->, >=latex, double, color=red](s2) to node[midway, right]{{\small $a_1$}} (a1s2);
	\draw[->, >=latex](a1s2) to [out=30, in=30, looseness=1.2] (s2);

    \node [above, yshift=0.2em] at (a0s0) {\scriptsize $r=0$};
    \node [below, xshift=-0.3em] at (a0s1) {\scriptsize $r \sim Be\left(\frac{1}{3}\right)$};
    \node [below, yshift=-0.2em, xshift=1em] at (a0s2) {\scriptsize $r \sim Be\left(\frac{2}{3}\right)$};
    \node [above] at (a1s2) {\scriptsize $r \sim Be\left(\frac{2}{3}\right)$};
    
    \end{scope}
    
\end{tikzpicture}
\label{F:toy3d}
        }\\
        
        \vspace{-0.8em}
        \centering
        \subfigure[]{
        \includegraphics[width=0.945\textwidth]{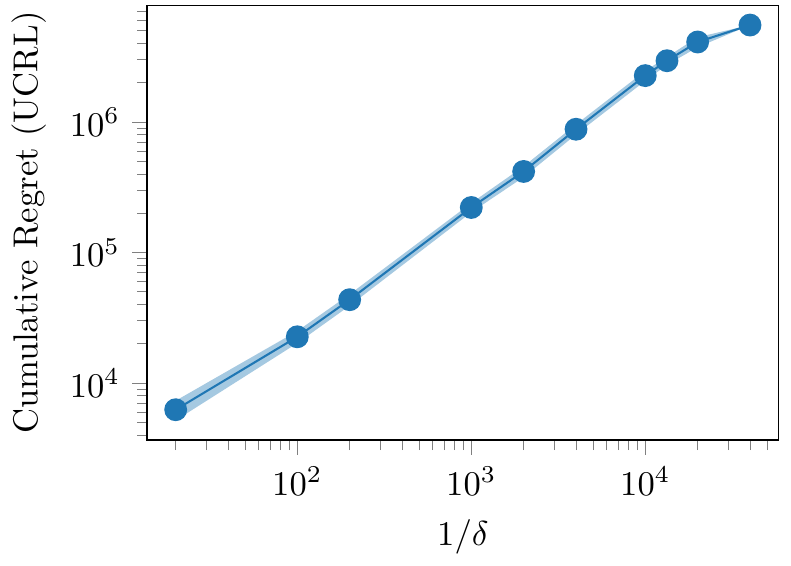}
        \label{F:delta_effect_UCRL}
        }
        \vspace{-0.15in}
        \caption{\textit{(upper)} Simple three-state domain. \textit{(lower)} Cumulative regret incurred by UCRL after $T=2.5\cdot 10^7$ steps as a function of the diameter $D \approx 1/\delta$ (averaged over $20$ runs). }
\end{minipage}
\hfill
        \begin{minipage}{0.65\textwidth}
        \centering
        \begin{tikzpicture}
                \begin{scope}[local bounding box=scope1]
                        \node at (0,0) {
                                \includegraphics[width=0.45\textwidth]{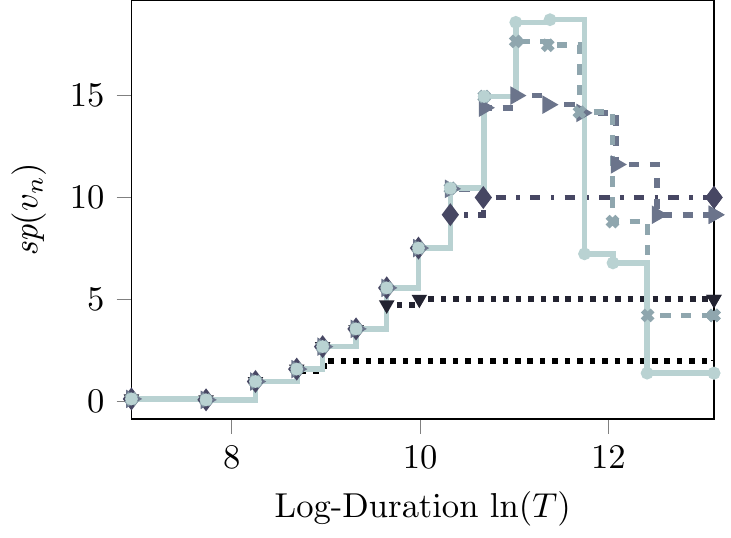}
                        };
                        \node at (5.6,0.15) {
                                \includegraphics[width=0.45\textwidth]{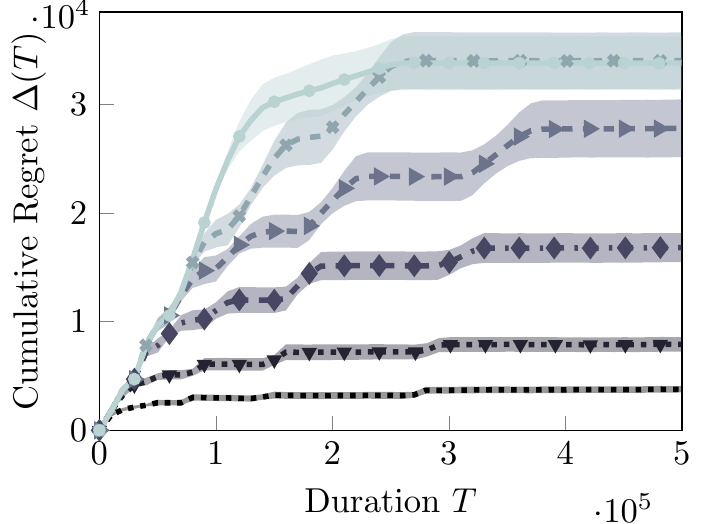}
                        };
                \end{scope}
                \begin{scope}[shift={($(scope1.south west)+(2,-2)$)}]
                        \node at (0.45,0) {
                                \includegraphics[width=0.477\textwidth]{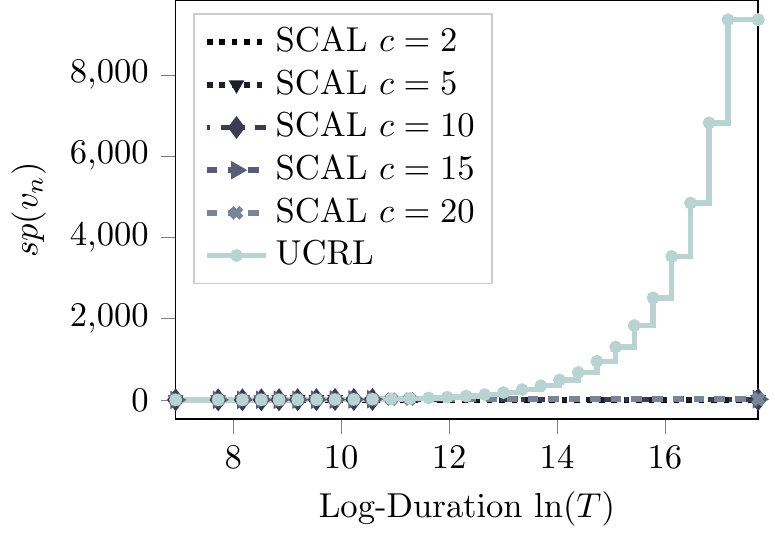}
                        };
                        \node at (6.18,0.15) {
                                \includegraphics[width=0.46\textwidth]{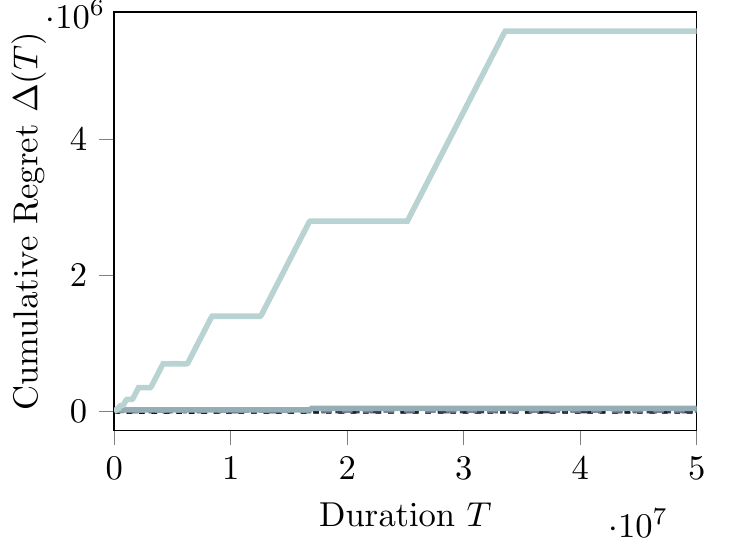}
                        };
                \end{scope}
                \draw (-3,2.1) rectangle (8.1,-1.8);
                \draw (8.1,-1.9) rectangle (-3,-5.8);
                \node at (-2.7, -1.6) {(a)};
                \node at (-2.7, -5.6) {(b)};
        \end{tikzpicture}
    \vspace{-0.3in}
    \caption{Results in the three-states domain with $\delta=0.005$ \textit{(top)} and $\delta = 0$ \textit{(bottom)}. We report the span of the optimistic bias \textit{(left)} and the cumulative regret \textit{(right)} as a function of $T$. 
            Results are averaged over $20$ runs and $95\%$ confidence intervals are shown.
    }
    \label{F:3dresults}
    \end{minipage}
    \vspace{-0.15in}
\end{figure*}

%

In this section, we numerically validate our theoretical findings. 
The code is available on~\href{https://github.com/RonanFR/UCRL}{GitHub}.
In particular, we show that the regret of \ucrl indeed scales linearly with the diameter, while \scal achieves much smaller regret that only depends on the span. This result is even more extreme in the case of non-communicating MDPs, where $D=\infty$.
Consider the simple but descriptive three-state domain shown in Fig.~\ref{F:toy3d} (results in a more complex domain are reported in App.~\ref{app:experiments}).
In this example, the learning agent only has to choose which action to play in state~$s_2$ (in all other states there is only one action to play).
The rewards are distributed as Bernoulli with parameters shown in Fig.~\ref{F:toy3d} and $\rmaxbound = 1$.
The optimal policy~$\pi^*$ is such that $\pi^*(s_2)=a_1$ with gain $g^* = \frac{2}{3}$ and bias $h^* = \Big[ \frac{-2-\delta}{3(1-\delta)}, \frac{-1}{1-\delta}, 0\Big]$.
If $\delta$ is small, $\SP{h^*}=\frac{1}{1-\delta} \approx 1$, while $D \approx \frac{1}{\delta}$.
Fig.~\ref{F:delta_effect_UCRL} shows that, as predicted by theory, the regret of \ucrl (for a fixed horizon $T$) grows linearly with $\frac{1}{\delta}\approx D$. The optimal bias span however is roughly equal to $1$. Therefore, we expect \scal to clearly outperform \ucrl on this example.
In all the experiments, we noticed that perturbing the extended MDP was not necessary to ensure convergence of \regopt and so we set $\eta_k=0$.
We also set $\gamma_k = 0$ to speed-up the execution of \regopt (see stopping condition in Fig.~\ref{fig:regopt}).

\textbf{Communicating MDPs.} 
We first set $\delta=0.005>0$, giving a communicating MDP.
With such a small $\delta$, visiting state $s_1$ is rather unlikely.
Nonetheless, since \ucrl is based on the OFU principle, it keeps trying to visit $s_1$ (\ie play $a_0$ in $s_2$) until it collects enough samples to understand that $s_1$ is actually a \emph{bad} state (before that, \ucrl \emph{``optimistically''} assumes that $s_1$ is a \emph{highly rewarding} state). Therefore, \ucrl plays $a_0$ in $s_2$ for a long time and suffers large regret. This problem is particularly challenging for any learning algorithm solely employing \emph{optimism} like \ucrl (cf.\ \citep{DBLP:journals/mima/Ortner08} for a more detailed discussion on the intrinsic limitations of optimism in RL).
In contrast, \scal is able to mitigate this issue when an appropriate constraint $c$ is used.
More precisely, whenever $s_1$ is believed to be the most rewarding state, the value function (bias) is maximal in $s_1$ and \regopt applies a ``truncation'' in that state and ``mixes'' deterministic actions. 
In other words, \scal leverages on the prior knowledge of the optimal bias span to understand that $s_1$ cannot be as good as predicted (from optimism).
The exploration of the MDP is greatly affected as \scal quickly discovers that action $a_0$ in $s_2$ is suboptimal.
Therefore, \scal is always performing better than UCRL (Fig.~\ref{F:3dresults}(a)) and the smaller $c$, the better the regret.
Surprisingly the \emph{actual} policy played by \scal in this particular MDP is always deterministic. \regopt mixes actions in $s_1$ where only one \emph{true} action is available but the mixing happens in the \emph{extended} MDP $\wt{\mathcal{M}}^{\ddagger}_k$ where the action set is compact. The policy that \regopt outputs is thus \emph{stochastic} in the \emph{extended} MDP but \emph{deterministic} in the \emph{true} MDP.
%

\textbf{Infinite Diameter.} By selecting $\delta =0$ the diameter becomes infinite ($D = +\infty$) but the MDP is still \emph{weakly} communicating (with transient state~$s_1$). 
UCRL is not able to handle this setting and suffers linear regret.
%
On the contrary, \scal is able to quickly recover the optimal policy (see Fig.~\ref{F:3dresults}(b) and App.~\ref{app:experiments}).
%

\vspace{-0.05in}
\section{Conclusion}
\vspace{-0.05in}

In this paper we introduced \scal, a \ucrl-like algorithm that is able to efficiently balance exploration and exploitation in any \emph{weakly communicating} MDP for which a finite bound $c$ on the optimal bias span $\SP{h^*}$ is known.
While \ucrl exclusively relies on \emph{optimism} and uses \evi to compute the exploratory policy, \scal leverages the knowledge of $c$ through the use of \regopt, a new planning algorithm specifically designed to handle constraints on the bias span.
We showed both theoretically and empirically that \scal achieves smaller regret than \ucrl. 
%
Although \scal was inspired by \regalc, it is the only \emph{implementable} approach so far. Therefore, this paper answers the long-standing open question of whether it is actually possible to design an \emph{algorithm} that does not scale with the diameter $D$ in the worst case.
Moreover, \scal paves the way for implementable algorithms able to learn in an MDP with \emph{continuous} state space. Indeed, existing algorithms achieving regret guarantees in this framework~\citep{DBLP:journals/corr/abs-1302-2550,pmlr-v37-lakshmanan15} all rely on \regalc.
%
We also believe that our approach can easily be extended to optimistic \psrl \citep{Agrawal2017posterior} to achieve an even better regret bound of $\wt{\O}\left(\min{\lbrace c, \rmaxbound D \rbrace \sqrt{SAT}}\right)$, \ie drop the dependency in $\nextstates$.
Finally, we leave it as an open question whether the assumption that $c$ is known can be relaxed. 

\newpage
\section*{Acknowledgements}
This research was supported in part by French Ministry of Higher Education and Research, Nord-Pas-de-Calais Regional Council and French National Research Agency (ANR) under project ExTra-Learn (n.ANR-14-CE24-0010-01).
Furthermore, this work was supported in part by the Austrian Science Fund (FWF): I 3437-N33 in the framework of the CHIST-ERA ERA-NET (DELTA project).

\bibliographystyle{plainnat}
\bibliography{span.bib}

\newpage
\onecolumn
\appendix

\section*{Index of the Appendix}
We start providing a brief recap of the content of the appendix:
\begin{itemize}
 \item App.~\ref{app:optim_problem_formulation}
 \begin{itemize}
 \item Proof of Lem.~\ref{prop:P.well.defined}: If $M$ is unichain then $\PiC^*(M) \neq \emptyset$. As a starting point the continuity of gain $g$ and span $h$ \wrt the policy is proved (see Lem.~\ref{L:continuity.unichain}).
 \item The policy associated to $\opT{v}$ can be interpreted as a solution of an LP problem (see App.~\ref{app:optim_problem_formulation2})
 \end{itemize}
 \item App.~\ref{app:counterexamples}
 \begin{itemize}
  \item Shows the limitations of \regopt ($\opT{}$): non-feasibility~\ref{sapp:counterexamples.nonfeasT} and non-convergence~\ref{sapp:nonconvergenctop}.
 \end{itemize}
 \item App.~\ref{sec:T.policy}
 \begin{itemize}
  \item We show how to compute the policy associated to the operator $\opT{v}$ when $\opT{}$ is feasible in $v$. We consider both MDPs and extended MDPs.
 \end{itemize}
 \item App.~\ref{app:operators}
 \begin{itemize}
  \item Proof of Lem.~\ref{L:TRpolicyexistence} \ie when operator $\opT{}$ is feasible (see App.~\ref{sapp:feasibility}).
  \item Proof of Lem.~\ref{lem:contraction} \ie that under Asm.~\ref{asm:L.contraction}, $\opT{}$ is a span contraction (see App.~\ref{sapp:contraction}).
  \item Proof of Lem.~\ref{lem:optimality.equation.T} \ie existence and uniqueness of the optimality equation for $\opT{}$, convergence of \regopt and gain dominance (see App.~\ref{sapp:optimality.equation.T})
  \item Proof of Thm.~\ref{thm:opT.optimality} \ie stopping condition for \regopt and approximation guarantees (see App.~\ref{sapp:opT.optimality}).
 \end{itemize}
 \item App.~\ref{app:augmented.perturbed.mdp}
 \begin{itemize}
  \item A formal definition of perturbed extended MDP (see Lem.~\ref{L:pperturbed_bmdp}) and span contraction property for $\wt{L}$ in the perturbed MDP.
  \item A formal definition of (reward) augmented extended MDP (see Lem.~\ref{L:r_augmented_bmdp}), equality of the operator in the original and augmented extend MDPs, and non-emptiness of $D(c,v)$ in the augmented MDP when $\SP{v} \leq c$.
  \item Proof of Thm.~\ref{th:augmented.perturbed.mdp}. We prove existence and uniqueness of the optimality equation for $\opT{}$, convergence of \regopt and gain dominance for the perturbed and augmented extended MDP (\ie $\mathcal{M}_K^\ddagger$) (see Thm.~\ref{T:improved_op}).
 \end{itemize}
 \item App~\ref{app:regret.reg.ucrl}
 \begin{itemize}
  \item Proof of Thm.~\ref{thm:regret.scal} \ie the regret of \regopt.
 \end{itemize}
 \item App~\ref{app:experiments}
         \begin{itemize}
                 \item We test \scal and \ucrl on a larger and more challenging domain (the knight quest)
         \end{itemize}
\end{itemize}


\section{Optimization with bias span constraint}\label{app:optim_problem_formulation}

\subsection{Existence of gain optimal policies under bias-span constraint: the unichain case (proof of Lem.~\ref{prop:P.well.defined})}

In this section we provide a formal proof of Lem.~\ref{prop:P.well.defined}.

In unichain MDPs, all policies $\pi\in\SR$ have a constant gain $g^\pi$ \citep[section 8.4]{puterman1994markov}, thus the search space reduces to $\PiC = \{\pi\in\SR: \SP{h^\pi}\leq c\}$. We assume that $\PiC \neq \emptyset$.
We first prove the following lemma.
\begin{lemma}\label{L:continuity.unichain}
 In a unichain MDP, $g:\pi \mapsto g^\pi$ and $h:\pi \mapsto h^\pi$ are continuous maps from $\SR$ to $\Re$ and $\SR$ to $\Re^{S}$ respectively.
\end{lemma}
\begin{proof}
 Let's consider two stationary policies $\pi = d^\infty \in \SR$ and $\wh{\pi} = \wh{d}^\infty \in \SR$. Denote by $P$ and $\wh{P}$ the transition matrices associated to $d$ and $\wh{d}$ respectively. Since the MDP is unichain by assumption, the Markov Chains characterized by $P$ and $\wh{P}$ each have a unique stationary distribution $\mu$ and $\wh{\mu}$ respectively. We express the gap $\mu - \wh{\mu}$ using the same decomposition as \citet{Seneta1993}
 \begin{align*}
   (\mu^\intercal - \wh{\mu}^\intercal)(I-\wh{P} + e \wh{\mu}^\intercal) = \mu^\intercal (P-\wh{P}) \implies (\mu^\intercal - \wh{\mu}^\intercal) = \mu^\intercal (P-\wh{P}) H_{\wh{P}}
 \end{align*}
 where $I$ is the identity matrix, $e = (1\dots,1)^\intercal$ is the vector of all 1's and $H_{\wh{P}} = (I-\wh{P} + e \wh{\mu}^\intercal)^{-1} - e \wh{\mu}^\intercal$ is the \emph{Drazin inverse} of $I-\wh{P}$ also known as the \emph{deviation matrix} of $\wh{P}$ (always well-defined, see Appendix A of \citep{puterman1994markov}). The above equality implies that
 \begin{align*}
  \|\mu^\intercal - \wh{\mu}^\intercal\|_1 \leq \underbrace{\| \mu^\intercal \|_1}_{=1} \|P-\wh{P}\|_{\infty,1} \|H_{\wh{P}}\|_{\infty,\infty} = \|P-\wh{P}\|_{\infty,1} \|H_{\wh{P}}\|_{\infty,\infty}
 \end{align*}
 where $\|A\|_{\infty,1} := \max_{i}\sum_j{|A_{ij}|}$ and $\|A\|_{\infty,\infty} := \max_{i,j}{|A_{ij}|}$. As a consequence of the above inequality, when $P \to \wh{P}$ we have $\mu \to \wh{\mu}$. Moreover, when $d \to \wh{d}$ we have $P \to \wh{P}$ by linearity and thus by composition:
 \begin{align*}
  \lim_{d \to \wh{d}} \mu = \wh{\mu}
 \end{align*}
 Denote by $r$ and $\wh{r}$ the reward functions associated to $d$ and $\wh{d}$ respectively.
 We have $g^\pi = \mu^\intercal r$ and $g^{\wh{\pi}} = \wh{\mu}^\intercal \wh{r}$ and since $r$ is linear (hence continuous) in $d$ we conclude that
 \begin{align*}
  \lim_{\pi \to \wh{\pi}} g^\pi = g^{\wh{\pi}}
 \end{align*}
 or in other words,  $g:\pi \mapsto g^\pi$ is continuous at $\wh{\pi}$ and since $\wh{\pi}$ was chosen arbitrarily, $g$ is continuous everywhere. Similarly, $h^\pi = H_P r$ and $P \mapsto H_P$ is continuous in $P$ (the computation of $H_P$ involves only continuous operations of $P$ and $\mu$ like addition, multiplication and inversion of matrices) and therefore $h^\pi$ is continuous too.
\end{proof}
Note that Lem.~\ref{L:continuity.unichain} does not hold in general when the MDP is not unichain (see Ex.~\ref{E:twostates} when $x \to 0$ and $y\to0$).

Since $\SP{\cdot}$ is a semi-norm, it is a continuous map from $\SR$ to $\Re$ and so the function $f: \pi \mapsto \SP{h^\pi}$ is continuous by composition. Since $f$ is continuous, $\SR$ is compact and $\Re$ is a Hausdorff space, we know from basic topology that $f$ is a proper map i.e, the preimage of every compact set in $\Re$ by $f$ is compact in $\SR$. Since we can express $\PiC$ as the preimage of the compact interval $[0,c]$ by $f$ \ie $\PiC = f^{-1}([0,c])$, it is clear that $\PiC$ is compact.
As a result, since $g^\pi$ is continuous in $\pi$ and $\PiC$ is compact, by Weierstrass extreme value theorem the maximum of $g^\pi$ is attained in $\PiC$ and so $\Pi_c^* \neq \emptyset$.

\subsection{Greedy policy under bias span constraint: LP formulation}\label{app:optim_problem_formulation2}

In this section, we show that the policy associated to $\opT{v}$ can be interpreted as the solution of a Linear Programming (LP) problem.

As mentioned in Sec.~\ref{ssec:operators}, a consequence of Lem.~\ref{L:TRpolicyexistence} (see proof in App.~\ref{app:operators}) is that whenever $D(c,v) \neq \emptyset$, there exists $\delta_v^+ \in D(c,v)$ such that $L_{\delta_v^+} v \geq L_d v$ for all $d \in D(c,v)$ \emph{component-wise}, and moreover $\opT{v} = L_{\delta_v^+} v$. As a result, we can express $\delta_v^+$ as a maximizer of the following optimization problem
\begin{align}\label{E:opT.as.LP}
 \max_{d \in  D(c,v)} \left\lbrace (L_d v)^\intercal e \right\rbrace
\end{align}
where $e=(1\dots1)^\intercal$ is the vector of all 1's. The maximum of \eqref{E:opT.as.LP} is then $(\opT{v})^\intercal e$. Since $d \mapsto r_d$ and $d \mapsto P_d$ are linear maps, the function we maximize $d \mapsto (L_d v)^\intercal e$ is also linear in $d$. Moreover, the set $D(c,v)$ can be expressed as a set of $S \times (S - 1)$ linear constraints in $L_d v$:
\begin{align*}
 L_dv(s) - L_dv(s') \leq c, ~~ \forall s \neq s'
\end{align*}
Therefore, optimization problem \eqref{E:opT.as.LP} can be formulated as an LP problem. But of course it is much easier to compute $\opT{v}$ using Def.~\ref{def:opT}. The policy $\delta_v^+$ associated to $\opT{v}$ can also be computed efficiently without solving \eqref{E:opT.as.LP} (see App.~\ref{sec:T.policy} for more details).

\begin{remark}
 Recall that computing the maximal gain of an MDP can be done by solving the following \emph{primal} LP problem \citep[Section 8.8]{puterman1994markov}
 \begin{align}\label{E:planning.LP}
 \begin{split}
   &\min_{g \in \mathbb{R},h\in \mathbb{R}^S} \lbrace g \rbrace\\
   \text{s.t.}~~ g + h(s) - \sum_{s' \in \calS}p(s'|s,a) &h(s') \geq r(s,a),~~ \forall s \in \calS, ~\forall a\in \A_s
   \end{split}
 \end{align}
One might wonder whether it is possible to reformulate optimization problem \eqref{P:opt_probl_well_posed} presented in Sec.~\ref{sec:optimization} by adapting the above primal formulation with the addition of $S \times (S - 1)$ linear constraints in $h$ (as we did above for $L_dv$):
\begin{align*}
 h(s) - h(s') \leq c, ~~ \forall s \neq s'
\end{align*}
Unfortunately it is not that simple. Indeed, the validity of LP problem \eqref{E:planning.LP} is a consequence of the following two properties \citep[Theorem 8.4.1]{puterman1994markov}:
\begin{align*}
 1)~~ge + h -Lh \geq 0 \implies g \geq g^*\\
 2)~~ge + h -Lh = 0 \implies g = g^*
\end{align*}
In general, these properties no longer hold for operator $\opT{}$ and optimal bias-span-constrained gain $g^*_c$. Therefore, the LP approach fails (one can easily try to solve the constrained LP on a simple MDP and observe the solution is incorrect). Using the \emph{dual} formulation is also tricky because the span constraint is not linear in the dual variables. Whether it is possible to formulate problem \eqref{P:opt_probl_well_posed} as an LP problem is left as an open question.
\end{remark}

\section{Limitations of \regopt}\label{app:counterexamples}
In this section, we illustrate the limitations of operator $\opT{}$ on some simple examples. For convenience, we introduce notation $\opN{}$ to denote the (value) operator associated to policy $\opG{v}$ as
 \begin{align}\label{eq:Nop}
 \opN{v} = L_{(\opG{v})}v.
 \end{align}
Trivially, if $\opT{v}$ is globally feasible, then $\opN{v} = \opT{v}$ and $\opG{v} = \delta^+_v$.

\begin{figure}[t]
        \begin{minipage}[b]{.32\textwidth}
  \centering
  \begin{tikzpicture}[->, >=stealth', shorten >=1pt, auto, node distance=2.8cm, semithick]
        \tikzstyle{every state}=[fill=red,draw=none,text=white]
        \node[state] (A) {$s_0$};
        \node[state] (B) [right of=A] {$s_1$};
        \path (A) edge[bend left] node {$a_0$, {\color{red}$r=0$}} (B);
        \path (A) edge[loop above] node {$a_1$, {\color{red}$r=0$}} (A);

        \path (B) edge[bend left] node {$a_0$, {\color{red} $r=1$}} (A);
        \path (B) edge[loop below] node {$a_1$, {\color{red} $r=1$}} (B);
  \end{tikzpicture}

\caption{\label{F:n_different_t}
        Example showing that $\opT{}$ might not always be feasible even when $\PiC^* \neq \emptyset$.
}

        \end{minipage}
        \hfill
        \begin{minipage}[b]{.32\textwidth}

  \centering
  \begin{tikzpicture}[->, >=stealth', shorten >=1pt, auto, node distance=2.4cm, semithick]
        \tikzstyle{every state}=[fill=red,draw=none,text=white]
        \node[state] (A) {$s_0$};
        \node[state] (B) [below of=A] {$s_1$};
        \node[state] (C) [right of=B] {$s_2$};
        \path (A) edge[bend right] node {$a_0$, {\color{red}$r=\alpha$}} (B);

        \path (B) edge[bend right] node {$a_1$, {\color{red} $r=\delta$}} (C);

        \path (B) edge[bend left] node {$a_0$, {\color{red}$r=\beta$}} (C);
        \path (C) edge[loop below] node {$a_0$, {\color{red} $r=0$}} (C);
  \end{tikzpicture}
\caption{\label{F:conjecture_wrong}
        Example showing that $\opT{}$ might not always be feasible at its fixed-point $h^+$.
}
        \end{minipage}
        \hfill
        \begin{minipage}[b]{.32\textwidth}
  \centering
  \begin{tikzpicture}[->, >=stealth', shorten >=1pt, auto, node distance=2.4cm, semithick]
        \tikzstyle{every state}=[fill=red,draw=none,text=white]
        \node[state] (A) {$s_0$};
        \node[state] (B) [right of=A] {$s_1$};
        \node[state] (C) [below of=B] {$s_2$};
        \node (D) at (1,0) {\color{red}$r=1$};
        \path (A) edge[bend left] node {$1-\delta$} (B);
        \path (A) edge[loop above] node {$\delta$} (A);

        \path (B) edge node {\color{red}$r=0$} (C);

        \path (C) edge[bend left] node {\color{red} $r=0$} (A);
  \end{tikzpicture}
\caption{\label{F:not_converging}
        Example showing that the sequence $(\opT{})^n v_0$ might not converge even when $\PiC \neq \emptyset$ and all policies are both unichain and aperiodic.
}
        \end{minipage}
\end{figure}

\subsection{Non-feasibility of $\opT$}\label{sapp:counterexamples.nonfeasT}
The following example shows that operator $\opT{}$ may generate vectors that do not correspond to a one-step policy evaluation, \ie there may not exists $\delta_v^+ \in D^{MR}$ such that $\opT{v} = L_{\delta_v^+} v$, even if $\SP{v} \leq c$ and/or if $\PiC \neq \emptyset$.
   \begin{example}\label{E:n_different_t}
        Consider the simple MDP provided in Fig.~\ref{F:n_different_t}. Let $v = [0,0]$ and $c=1/2$. In this case the policy $\pi$ playing 
        $a_0$ in both states matches the constraint $c$ (\ie $\SP{h^\pi}\leq c \implies \pi \in \PiC \implies \PiC \neq \emptyset$) and moreover $\SP{v} \leq c$, but clearly there exists no policy $\delta_v^+$ achieving $L_{\delta_v^+} v = T_c v$ and moreover
        \begin{align*}
                T_c v = [0, 1/2] \neq N_c v = [0, 1]
        \end{align*}
       
\end{example}

The previous example shows that there may not exists a policy associated to the \emph{one-step application} of operator $\opT{}$. Instead, Ex.~\ref{E:conjecture_wrong} shows that $\opT{}$ may also not be feasible \emph{at convergence}. In particular, we show that, surprising as it may seem, $\regopt$ can sometimes converge to a value $h^+$ that is not associated to any policy, even when $\Pi_c \neq \emptyset$ and even when $\Pi_c^* \neq \emptyset$.

\begin{example}\label{E:conjecture_wrong}
        Consider the simple MDP $M$ of Fig.~\ref{F:conjecture_wrong} where we assume that $\beta < \delta < \alpha$ and we set $c = \alpha + \beta$. The MDP is
        unichain and all gains are equal to $0$. The set of randomized decision rules can be parametrized by the probability $p$ of playing $a_1$ in $s_1$ and the associated set of bias functions is
        $$H = \left\lbrace \left[\begin{matrix} \alpha + (1-p)\cdot \beta + p \cdot \delta\\ (1-p)\cdot \beta + p \cdot \delta\\ 0 \end{matrix}\right] : p\in[0,1] \right\rbrace.$$
        Let's denote by $h(p)$ the bias associated to a policy parameterized by $p$, then $\SP{h(p)} = \alpha + (1-p)\cdot \beta + p \cdot \delta > c$ for all $p >0$. So there exists only one policy $\pi = d^\infty$ achieving the span constraint which plays $a_0$ in $s_1$
        (\ie $p=0$) implying that $\PiC(M) = \lbrace d^\infty \rbrace = \PiC^*(M) \neq \emptyset \implies \pi_c^* = d^\infty$. It is easy to verify that:
        \begin{align*}
        \text{Fixed point of }~ \opT{}:~ h^+&= ~ \left[\begin{matrix} \alpha + \beta \\ \delta\\ 0 \end{matrix}\right] \notin H\\
        \text{Fixed point of }~ \opN{}:~ h^\#&= ~ \left[\begin{matrix} \alpha + \delta \\ \delta\\ 0 \end{matrix}\right] \in H ~~ \text{but}~~ \SP{h^\#} > c,
        \end{align*}
Although $\opT{}$ admits a fixed point $h^+$, it is not globally feasible at $h^+$. On the other hand, while $N_c$ is globally feasible at its fixed point $h^\#$ by definition, $h^\#$ does not satisfy the bias constraint.
        One might be tempted to think that the problem in this example arises from the fact that $\PiC(M)$ is a singleton but it is actually more subtle than 
        that. Indeed, if we assume that $\beta >0$ and if we add an action $a_2$ in $s_1$ that goes to $s_2$ with probability $1$ and gives a reward $0$, we face the same problem but this time
        $\PiC(M)$ contains an infinite number of policies (since we include stochastic policies). The problem is actually coming from the fact that the action played in the only state achieving maximum bias (\ie $s_0$) is deterministic while the action played in state $s_1$ (which achieves a lower bias than $s_0$) is stochastic. $\opT{}$ is unable to converge to such policies: by definition, it can only converge to a policy that plays a stochastic action in the states with maximal bias (it can also converge to a bias that is not associated to any policy like in this example).
\end{example}
The issue presented in Ex.~\ref{E:conjecture_wrong} can be overcome by duplicating all actions and adjusting the rewards of the duplicated actions.
More formally, denote by $\wb{a}$ the action obtained by duplicating $a$. The probability of transition is not modified (\ie $p(\cdot|s,\wb{a}) = p(\cdot|s,a)$) but the reward is set to the minimal value (\ie $r(s,\wb{a}) = r_{\min} = 0$). Denote by $M^{\downarrow}$ this ``\emph{augmented}'' MDP. It is easy to verify that $h^+(M^{{\downarrow}}) = h^+(M) = (\alpha+\beta,\beta,0)^\intercal$ but unlike $M$, $M^{\downarrow}$ admits a policy associated to $h^+(M^{{\downarrow}})$ (using duplicated actions). As this example shows, augmenting the MDP never modifies the fixed point of $\opT{}$ but always makes $\opT{}$ globally feasible at any vector $v$ satisfying $\SP{v} \leq c$. Since by definition the fixed point $h^+$ of $\opT{}$ satisfies the span constraint, $\opT{}$ is globally feasible at $h^+$.
This example gives an intuition why \scal uses a modified MDP $\wt{\mathcal{M}}_k^\ddagger$ with augmented rewards (the confidence intervals $B_{r}^k$ are ``augmented'' by below).

\subsection{Non-convergence of $\opT{}^n$}\label{sapp:nonconvergenctop}
It is rather easy to design an MDP for which the stopping condition of \regopt (\ie $\SP{v_{n+1} - v_n}\leq \varepsilon$) is never met although all policies are unichain and aperiodic and $\PiC \neq \emptyset$. In contrast, for the optimal Bellman operator $L$, unichain and aperiodicity are sufficient conditions to ensure that the stopping condition $\SP{v_{n+1} - v_n}\leq \varepsilon$ is met after a finite number of iterations \citep[Theorem 8.5.7]{puterman1994markov}.
\begin{example}\label{E:not_converging}
        Consider the simple MDP $M$ provided in Fig.~\ref{F:not_converging} where we assume that $1>\delta >0$ and $1/2 \geq c>0$. 
        There is only one action available in every state and thus there is only one decision rule $d$.
        In that case $L=L_d$.
        The contraction condition of ~\citep[Theorem 8.5.3]{puterman1994markov} holds for $J=2$, \ie $L$ is a $2$-stage span contraction. More precisely we have:
        \begin{align*}
         P_d = \left[\begin{matrix} \delta & 1-\delta & 0 \\0 & 0 & 1\\ 1 & 0 & 0 \end{matrix}\right] \implies
         P_d^2 = \left[\begin{matrix} \delta^2 & \delta-\delta^2 & 1-\delta \\1 & 0 & 1\\ \delta & 1-\delta & 0 \end{matrix}\right]\\
         \implies \gamma_{d} \overset{def}{=} 1 - \min_{s,u \in \mathcal{S}} \left\lbrace \sum_{j \in \mathcal{S}} \min \lbrace P_d^2(j|s), P_d^2(j|u) \rbrace \right\rbrace 
         = \delta <1
        \end{align*}
        where $\gamma_{d}$ is the \emph{ergodic coefficient} associated to Markov Chain $P_d$.
	This implies that $L$ is a $2$-stage $\gamma_{d}$-span contraction: $\SP{L^{2n+1} v -L^{2n}v} \leq \gamma_{d}^n \SP{L v -v}$ for any vector $v$ and any integer $n \geq 0$. 
    On the other hand, the sequence $\opT{}^n v_0$ starting from $v_0 =0$ proceeds as follows
        \begin{align*}
                v_0 = \left[\begin{matrix} 0 \\ 0 \\0 \end{matrix}\right], ~
                v_1 = \left[\begin{matrix} c \\ 0 \\0 \end{matrix}\right], ~
                v_2 = \left[\begin{matrix} c \\ 0 \\c \end{matrix}\right], ~
                v_3 = \left[\begin{matrix} 2c \\ c \\c \end{matrix}\right] = v_1+ c e, ~
                \dots, ~
                v_{2n} = v_2 + n c e,~
                v_{2n+1} = v_1 + n c e
        \end{align*}
        where $e = (1,\dots, 1)^\intercal$ denotes the vector of all 1's. We see that unlike $L^nv_0$, $(\opT{})^n v_0$ is cycling with period 2 and the quantity $\SP{v_{2n+1} -v_{2n}} = \SP{v_2 - v_1}$ does not converge to $0$. Although Lem.~\ref{lem:contraction} shows that when $L$ is a $J$-stage span contraction with $J =1$ then $T_c$ is also a span contraction (proof in App.~\ref{app:operators}), surprisingly $(\opT{})^n v_0$ might not converge when $J >1$.
	Note that in this example $\PiC(M) = \emptyset$ and so one might wonder whether when $\PiC(M) \neq \emptyset$ the sequence $\opT{}^n v_0$ converges in span semi-norm.
        Unfortunately, it is not the case.
        Take the same MDP, duplicate the action in $s_0$ and assign a reward of $0$ to this new action (the new action loops on $s_0$ with probability $\delta$ and goes to 
        $s_1$ with probability $1-\delta$ as the original action, but the reward is $0$ instead of $1$). In that case, $\PiC(M) \neq \emptyset$ for all $c \geq 0$ but 
        we still have $L=L_d$ where $d$ plays the original action in $s_0$. Therefore, we face exactly the same problem as before although $\PiC(M) \neq \emptyset$.
\end{example}

\section{Policy $\delta_v^+$ associated to $\opT{v}$} \label{sec:T.policy}
In this section, we provide a detailed description on how to efficiently compute a policy $\delta_v^+$ associated to $\opT{v}$ when $\opT{}$ is feasible at $v$.
As mentioned in Sec.~\ref{sec:optimization}, we say that $\opT{}$ is \emph{feasible} at $v \in \mathbb{R}^S$ and $s\in\calS$ when there exists a distribution $\delta^+_v(s) \in \mathcal{P}(\mathcal{A})$ such that
\begin{equation}
\opT{v}(s) = \sum_{a\in \mathcal{A}_s} \delta_v^+(s,a) \big[ r(s,a) + p(\cdot,s,a)^\transp v \big].
\end{equation}
We distinguish between two types of states:
\begin{itemize}
        \item \textbf{Greedy states.} When $L{v}(s) \leq \min \{L{v}\} + c$~~\ie $s \in \overline{\calS}(c,v)$ (see Def.~\ref{def:opT}), 
                $\delta^+_v(s)$ plays a deterministic \emph{greedy} action $\wb{a} \in \argmax_{a \in \A_s} \{r(s,a) + p(\cdot|s,a)^\transp v\}.$
        \item \textbf{Truncated states.} When $L{v}(s) > \min\{L{v}\} + c$ \ie $s \notin \overline{\calS}(c,v)$ (see Def.~\ref{def:opT}), by definition of $L$ there exists at least one action $\wb{a}$ (\eg any greedy action) such that $r(s,\wb{a}) + p(\cdot|s,\wb{a})^\transp v > \min\{L{v}\} + c$. In addition, under the assumption that $\opT{}$ is feasible at $v$ and $s$, we know from condition~\eqref{E:Top_policy_cond} of Lem.~\ref{L:TRpolicyexistence} (proof in App.~\ref{app:operators}) that there exists an action $\underline{a}$ such that $r(s,\underline{a}) + p(\cdot|s,\underline{a})^\transp v \leq \min\{L{v}\} + c$.
                By the intermediate value theorem, we know that there exists a convex combination $\delta_v^+(s)$ of actions $\wb{a}$ and $\underline{a}$ achieving exactly $L_{\delta^+_v}v(s) = \min\{L v\} + c$.
                Note that there may exist multiple policies achieving this value (\eg when there are multiple actions achieving higher or smaller values than $\min\{L v\} + c$). However, to simplify the implementation, we can simply set $\delta_v^+(s)$ to play with non-zero probability only a greedy action $\wb{a} \in \argmax_{a \in \A_s} \{r(s,a) + p(\cdot|s,a)^\transp v\}$ and a minimal action $\underline{a} \in \argmin_{a \in \A_s} \{r(s,a) + p(\cdot|s,a)^\transp v\}$.
                Formally, let $\underline{v} = \min_{a \in \A_s} \{r(s,a) + p(\cdot|s,a)^\transp v\}$ and $\wb{v}= \max_{a \in \A_s} \{r(s,a) + p(\cdot|s,a)^\transp v\} = Lv(s)$. Then,
                \begin{equation}\label{E:T.policy.feasible}
                        \delta^+_v(s,a) = \begin{cases}
                                {(\wb{v} - \min\{L{v}\} -c)}/{(\wb{v} - \underline{v})} & \text{if } a = \underline{a}\\
                                {(\min\{L{v}\} + c - \underline{v})}/{(\wb{v} - \underline{v})} & \text{if } a = \wb{a}\\
                                0 & \text{otherwise}
                        \end{cases}
                \end{equation}
\end{itemize}

\paragraph{Bounded-Parameter MDP.} 
When we consider a bounded-parameter MDP $\wt{\mathcal{M}}$, the only change is in the computation of the minimal and maximal actions.
Define
\begin{align*}
        \wt{L}v(s) = \max_{a \in \A_s,\wt{r} \in B_r(s,a), \wt{p} \in B_p(s,a)} \big[ \wt{r} + \wt{p}^\transp v\big]\\
        \uwt{L}v(s) = \min_{a \in \A_s,\wt{r} \in B_r(s,a), \wt{p} \in B_p(s,a)} \big[ \wt{r} + \wt{p}^\transp v\big]
\end{align*}
Then, $\wb{a}$ and $\underline{a}$ are the actions associated to $\wt{L}v(s)$ and $\uwt{L}v(s)$, respectively. Given $\wb{a}$ and $\underline{a}$, the policy $\delta_v^+(s)$ associated to $\opT{v}(s)$ is computed as in~\eqref{E:T.policy.feasible}. The maximum and minimum of $\wt{r}$ for $\wt{r} \in B_r(s,a)$ are easy to compute (they correspond to the extreme values of the closed interval $B_r(s,a)$). The maximum of $\wt{p}^\transp v$ for $\wt{p}(s') \in B_p(s,a,s')$ can be computed in $\O(S)$ operations using the algorithm described in \citep[Appendix A]{DannBrunskill15}. To compute the minimum of $\wt{p}^\transp v$, the exact same algorithm can be used with input $-v$ instead of $v$ since $\min_{\wt{p}} \lbrace \wt{p}^\intercal v \rbrace = -\max_{\wt{p}} \lbrace \wt{p}^\intercal (-v)\rbrace$.


\section{Properties of operators $\opT{}$ and $\opG{}$}\label{app:operators}

\subsection{Feasibility of $\opT$ (proof of Lemma~\ref{L:TRpolicyexistence})}
\label{sapp:feasibility}

In this section, we prove Lem.~\ref{L:TRpolicyexistence}.

We start by proving that for any decision rule $d \in D(c,v)$, we have $\opT{v} \geq L_d v$ component-wise. By definition of $\opT{}$ and the optimal Bellman operator $L$ it holds that
$$
\forall s \in \overline{\calS}(c,v),\; \forall d \in D(c,v), \quad \opT{}v(s) = Lv(s) \geq L_d v(s).
$$
Moreover, let $m := \min_s \{Lv(s)\}$, then
\begin{align}
\forall s \in \mathcal{S} \setminus \overline{\calS}(c,v),\; \forall d \in D(c,v), \quad 
\opT{v}(s) &= m + c \\
                &\geq m + \SP{L_d v} \label{E:dcvstep} \\
                &= m + \max \{L_d v\} - \min \{L_d v\}\\
                &\geq m + L_d v (s) - m \label{E:maxminstep} = L_d v(s),
\end{align}
Inequality~\eqref{E:dcvstep} is a consequence of the fact that $d \in D(c,v) \Leftrightarrow \SP{L_d v} \leq c$.
Denote by $\hat{s} \in\argmax_s\{ Lv(s) \}$ any state achieving minimum value for $Lv(s)$. 
Inequality~\eqref{E:maxminstep} follows by noticing that $m =
Lv(\hat{s}) \geq L_d v(\hat{s}) \geq \min \{L_d v\}$.
In conclusion, for any $d \in D(c,v)$ and any $s \in \calS$, $\opT{v}(s) \geq L_d v(s)$. This immediately implies that whenever $\opT{}$ is globally feasible at $v$, $\opT{v} = \max_{\delta\in D(c,v)} L_{\delta} v$ and $\delta^+_v \in \argmax_{\delta\in D(c,v)} L_{\delta} v.$

Let's now prove the equivalence between the feasibility of $\opT$ at $v \in \mathbb{R}^S$ and $s \in \calS$ and condition~\eqref{E:Top_policy_cond} \ie
%
\begin{equation*}
\min_{a\in \mathcal{A}_s} \{ r(s,a) +  p(\cdot,s,a)^\transp v \} \leq \min_s \{L v(s) \}+ c.
\end{equation*}
If condition~\eqref{E:Top_policy_cond} holds, we can use the constructive procedure described in App.~\ref{sec:T.policy} to construct a stochastic action $\delta_v^+(s) \in \mathcal{P}(\mathcal{A})$ such that $L_{\delta_v^+}v(s) = \min_s \{L v(s) \}+ c =\opT{v}(s)$ and thus $\opT$ is feasible at $v$ and $s$. On the other hand, if condition~\eqref{E:Top_policy_cond} does not hold \ie $\min_{a\in \mathcal{A}_s} \{ r(s,a) +  p(\cdot|s,a)^\transp v \} > \min_s \{L v(s) \}+ c$ then it is clear that any $d(s) \in \mathcal{P}(\mathcal{A})$ will be such that $L_dv(s) >  \min_s \{L v(s) \}+ c$ and so $\opT$ is not feasible at $v$ and $s$. By contraposition, if $\opT$ is feasible at $v$ and $s$ then condition~\eqref{E:Top_policy_cond} holds thus proving the equivalence.

Finally, $\opT$ is globally feasible at $v$ if and only condition~\eqref{E:Top_policy_cond} holds in every state $s\in\calS$. Trivially, if $\opT$ is globally feasible then $\SP{\opT{v}} = \SP{L_{\delta_v^+} v} \leq c$ and so $\delta_v^+ \in D(c,v)$ implying that $D(c,v) \neq \emptyset$. If $D(c,v) \neq \emptyset$ then there exists $\delta\in \MR$ such that $\SP{L_\delta v} \leq c$. Assume that condition~\eqref{E:Top_policy_cond} does not hold in at least one state $\wb{s}\in \calS$ meaning that $\min_{a\in \mathcal{A}_{\wb{s}}} \{ r(\wb{s},a) +  p(\cdot|\wb{s},a)^\transp v \} > \min_s \{L v(s) \}+ c$ which implies that for any $d \in \MR$, $L_d v(\wb{s}) > \min_s \{L v(s) \}+ c$. This contradicts the fact that \[ \SP{L_\delta v } \leq c \implies L_\delta v (\wb{s}) \leq \max_s \{L_d v(s) \} \leq \min_s \{L_d v(s) \}+ c \leq \min_s \{L v(s) \}+ c\] where we used the definition of the span $\SP{u}:= \max \{u\} -\min \{u\}$ and the fact that $L_dv \leq L v$ component-wise by definition of $L$ implying that $\min_s \{L_d v(s) \} \leq \min_s \{L v(s) \}$. Therefore, condition~\eqref{E:Top_policy_cond} must hold in every state. In conclusion, $\opT$ is globally feasible at $v$ if and only $D(c,v) \neq \emptyset$.

\subsection{Contraction property of $\opT$ (proof of Lemma~\ref{lem:contraction})}
\label{sapp:contraction}

The purpose of this section is to prove Lem.~\ref{lem:contraction}. We first reinterpret operator $\opT$ as the composition of a \emph{projection} $\proj{}$ and the optimal Belmman operator $L$ ($\opT = \proj L$) and we prove interesting properties for $\proj{}$ and $\opT$.

Recall that $\opT{}$ can be seen as the \emph{truncation} of the optimal Bellman operator \ie $\opT{v}(s) = \min\{Lv(s), \min_{x}\{Lv(x)\} + c\}$. The following lemma shows that the truncation step is actually a projection in span semi-norm. Let $V_c = \{v : \SP{v} \leq c\}$ be the \emph{``semi-ball''} of span constrained value functions. For any vector $v \in \mathbb{R}^{S}$ and any $c \geq 0$, we define the truncation operator $\proj:~\mathbb{R}^{S}\rightarrow V_c$ as $\proj{v}(s) = \min \left\{v(s), \min_{x}\{v(x)\} +c \right\}$.
\begin{lemma}\label{prop:truncation}
For any vector $v \in \mathbb{R}^{S}$ and $c \geq 0$, $\proj{v}$ is the projection of $v$ on the semi-ball $V_c$ in span semi-norm \ie
$$
\proj{v} = \min_{z \in V_c} \SP{z - v}.
$$
\end{lemma}
\begin{proof}
Let $w = \proj{v}$. If $\SP{v} \leq c$, then by definition of $\proj{}$, $w= \proj{v} = v \in \argmin_{z \in V_c} \SP{z - v}$. Otherwise, using again the definition of $\proj{}$ we have that $w \leq v$ component-wise. As a result, we have
\begin{align*}
\max_s\{w(s) - v(s)\} &= 0 
\end{align*}
Moreover, the difference between $w$ and $v$ is maximal in the states $\wb{s} \in \argmax_s v(s)$ and thus
\begin{align*}
\min_s\{w(s) - v(s)\} &= -\max_s\{v(s)\} + \min_s\{v(s)\} + c,
\end{align*}
Therefore: $\SP{w-v} = \max\{w - v\} - \min \{w -v\} = \SP{v} - c$. Furthermore, by reverse triangle inequality\footnote{The triangle inequality for the span is proved in~\citep[Section 6.6.1]{puterman1994markov}.}, for any vector $z$ such that $\SP{z} \leq c$ we have that
\begin{align*}
 \SP{z - v} \geq \SP{v} - \SP{z} \geq \SP{v} - c = \SP{w-v},
\end{align*}
thus proving the lemma.
\end{proof}

We prove the following useful properties for $\proj{}$:

\begin{lemma}\label{L:proj_properties}
Let $v$ and $u$ be vectors in $\mathbb{R}^{S}$, then:
        \begin{enumerate}[label=(\alph*), labelindent=\parindent]
                \item Monotonicity
                        $$v \geq u \implies \proj{v} \geq \proj{u}.$$
                \item For any $s \in \mathcal{S}$ 
                        \begin{equation}\label{E:min_max_diff_proj}
                                \min\{v - u\} \leq \proj{v}(s) - \proj{u}(s) \leq \max\{v -u\}.
                        \end{equation}
                \item $\proj{}$ is non-expansive in span semi-norm
                $$\SP{\proj{v} - \proj{u}} \leq \SP{v-u}.$$
                \item $\proj{}$ is non-expansive in $\ell_\infty$-norm
                $$\norm[\infty]{\proj{v} - \proj{u}} \leq \norm[\infty]{v-u}.$$
                \item  Linearity
                $$\forall\lambda \in \mathbb{R},~~\proj{\left(v + \lambda e\right)} = \proj{v} + \lambda e.$$
        \end{enumerate}
\end{lemma}
\begin{proof}
For any state $s \in \mathcal{S}$, the difference $\proj{v}(s) - \proj{u}(s) $ can only take four different values depending on the configuration of $v(s)$ and $u(s)$ 
 \begin{subnumcases}{\proj{v}(s) - \proj{u}(s) =}
  v(s) - u(s) & \text{ if } $u(s)\leq \min\{u\}+c \text{ and }  v(s) \leq \min\{v\}+c$ \label{case:proj1}\\
  \min\{v\} + c - u(s) & \text{ if } $u(s)\leq \min\{u\}+c \text{ and }  v(s) > \min\{v\}+c$ \label{case:proj2}\\
  v(s) - \min\{u\} - c & \text{ if } $u(s)> \min\{u\}+c \text{ and }  v(s) \leq \min\{v\}+c$ \label{case:proj3}\\
  \min\{v\} - \min\{u\} & \text{ if } $u(s)> \min\{u\}+c \text{ and }  v(s) > \min\{v\}+c$\label{case:proj4}
 \end{subnumcases}
\begin{enumerate}[label=(\alph*), labelindent=\parindent]
 \item We need to show that for all $s\in \calS$, the difference $\proj{v}(s) - \proj{u}(s)$ is bigger or equal than zero in all four cases. Case~\eqref{case:proj1} follows directly from the assumption $v \geq  u$, while case~\eqref{case:proj4} is trivially proved since $v \geq u$ implies $\min\{v\} \geq \min\{u\}$. Case~\eqref{case:proj3} follows from $v(s) - \min\{u\} - c > v(s) - u(s) \geq 0$ (by assumption $u(s)> \min\{u\}+c$ in this case). Finally, case~\eqref{case:proj2} reduces to case~\eqref{case:proj4} since we assume that $u(s)\leq \min\{u\}+c$ implying that $\min\{v\} + c - u(s) \geq \min\{v\} - \min\{u\} \geq0$.
 \item We treat all four cases separately as we did to prove \emph{(a)}.
 \begin{itemize}
  \item Case~\eqref{case:proj1}: it is straightforward to see that $\min\{v - u\} \leq \proj{v}(s) - \proj{u}(s) = v(s) - u(s) \leq \max\{v -u\}$
  \item Case~\eqref{case:proj2}: we have that $\min\{v\} + c - u(s) < v(s) - u(s) \leq \max\{v-s\}$. To prove the other inequality we start by noticing that $\min\{v\} + c - u(s) > \min\{v\} + c - \min\{u\} - c = \min\{v\} - \min\{u\}$. Since moreover
\begin{equation}\label{E:min_diff}
        \min_s\{v(s)-u(s)\} \leq \min_{s}\left\{v(s) - \min_{s'} \{ u(s') \} \right\} = \min\{v\} - \min\{u\}
\end{equation}
the inequality holds.
  \item Case~\eqref{case:proj3}: by definition $v(s) - \min\{u\} - c \leq \min\{v\} + c - \min\{u\} - c = \min\{v\} - \min\{u\}$. Then:
\begin{align}\label{E:max_diff}
 \max_s\{v(s)-u(s)\} \geq \max_s \left\{ \min_{s'} \{v(s')\} -u(s) \right\} = \min\{v\} - \min\{u\}
\end{align}
 The other inequality trivially follows by noticing that $v(s) - \min\{u\} - c > v(s) - u(s) \geq \min\{v-u\}$.
  \item Case~\eqref{case:proj4}: inequalities~\eqref{E:min_max_diff_proj} is a consequence of inequalities~\eqref{E:min_diff} and~\eqref{E:max_diff}.
 \end{itemize}
 \item This is easy to prove exploiting inequality~\eqref{E:min_max_diff_proj} (property \emph{(b)}):
\begin{align*}
        \SP{\proj{v} - \proj{u}}
        &= \max\{\proj{v} - \proj{u}\} - \min\{\proj{v} - \proj{u}\}\\ 
        &\leq \max\{v-u\} - \min\{v-u\} = \SP{v-u}
\end{align*}
 \item We can again use inequality~\eqref{E:min_max_diff_proj}
\begin{align*}
        \norm[\infty]{\proj{v} - \proj{u}} 
        &= \max_s \left\{ |\proj{v}(s) - \proj{u}(s)| \right\} \\
        &= \max\left\{
                \max_s\{ \proj{v}(s) - \proj{u}(s)\}, \max_s \{\proj{u}(s) - \proj{v}(s)\}
           \right\}\\
           &\leq \max\left\{
                \max_s \{ v(s) - u(s) \}, \max_s \{ u(s) - v(s) \}
                     \right\} = \norm{v - u}
\end{align*}
 \item By definition of $\proj{}$, for any $s \in \mathcal{S}$: 
\begin{align*}
 \proj{\left(v + \lambda e\right)}(s) &= \min \left\{v(s) + \lambda, \min\{v + \lambda e \} +c \right\}\\
 &= \min \left\{v(s) + \lambda, \min\{v\} + \lambda  +c \right\}\\
 &= \min \left\{v(s) , \min\{v\}  +c \right\} + \lambda\\
 &= \proj{v }(s) +  \lambda.
\end{align*}
\end{enumerate}
\end{proof}

We are now ready to prove the following lemma:

\begin{lemma}\label{L:T_properties}
        Let $v$ and $u$ be vectors in $\mathbb{R}^{S}$. Operator $\opT{}$ enjoys the following properties:
        \begin{enumerate}[label=(\alph*), labelindent=\parindent]
                \item Monotonicity
                        $$ v \geq u \implies \opT{v} \geq \opT{u}.$$
                \item $\opT{v} \leq  L{v}$ and if in addition $\SP v \leq c$ and $v \leq Lv$ then $v \leq \opT{v}.$
                \item Linearity
                $$\forall \lambda \in \mathbb{R},~~\opT{\left(v + \lambda e\right)} = \opT{v} + \lambda e.$$
                \item $\opT{}$ is non-expansive both in span semi-norm and $\ell_\infty$-norm
                $$\SP{\opT{v} - \opT{u}} \leq \SP{v-u} ~~\text{and}~~\norm[\infty]{\opT{v} - \opT{u}} \leq \norm[\infty]{v-u}.$$
                Moreover, if $L$ is a $\gamma$-span contraction then $\opT$ is also a $\gamma$-span contraction (Lem.~\ref{lem:contraction}).
        \end{enumerate}
\end{lemma}
\begin{proof}
We rely on the properties proved in Lem.~\ref{L:proj_properties}.
  \begin{enumerate}[label=(\alph*), labelindent=\parindent]
   \item The monotonicity of $\opT$ is a direct consequence of the monotonicity of both $L$ \citep{puterman1994markov} and $\proj{}$ (property \emph{(a)} of Lem.~\ref{L:proj_properties}) and the fact that monotonicity is preserved by composition.
   \item $\opT{v} = \proj{Lv} \leq Lv$ by definition of $\proj{}$. If $v \leq Lv$ then using the fact that $\proj{}$ is monotone we have $\proj{v} \leq \proj{Lv} = \opT{v}$ and if we assume that $\SP v \leq c$ then $v = \proj{v}$ and so $v \leq \opT{v}$.
   \item The linearity of $\opT$ is a direct consequence of the linearity of both $L$ \citep{puterman1994markov} and $\proj{}$ (property \emph{(e)} of Lem.~\ref{L:proj_properties}) and the fact that linearity is preserved by composition.
   \item Using the fact that $L$ \citep{puterman1994markov} and $\proj{}$ (properties \emph{(c)} and \emph{(d)} of Lem.~\ref{L:proj_properties}) are non-expansive both in span semi-norm and $\ell_\infty$-norm we show the following: 
   \begin{align*}
    &\SP{\opT{v} - \opT{u}} = \SP{\proj{Lv} - \proj{Lu}} \leq \SP{Lv-Lu} \leq \SP{v-u} \\~~\text{and}~~ &\|\opT{v} - \opT{u}\|_\infty = \|\proj{Lv} - \proj{Lu}\|_\infty \leq \|Lv-Lu\|_\infty \leq \|v-u\|_\infty
   \end{align*}
   If $L$ is a $\gamma$-span contraction then:
   \begin{align*}
    \SP{\opT{v} - \opT{u}} = \SP{\proj{Lv} - \proj{Lu}} \leq \SP{Lv-Lu} \leq \gamma\SP{v-u}
   \end{align*}
   meaning that $\opT$ is also a $\gamma$-span contraction.
  \end{enumerate}
\end{proof}
Lem.~\ref{lem:contraction} immediately follows from property \emph{(d)} of Lem.~\ref{L:T_properties}.


\subsection{Convergence properties of $\opT$ (proof of Lemma~\ref{lem:optimality.equation.T})}
\label{sapp:optimality.equation.T}

In this section we provide a detailed proof of Lem.~\ref{lem:optimality.equation.T}.

We assume that Asm.~\ref{asm:L.contraction} holds which implies that $\opT$ is a $\gamma$-span contraction by Lem.~~\ref{lem:contraction}.
\begin{enumerate}
 \item \emph{Existence and uniqueness of the solution of optimality equation~\eqref{E:optimality_eq_average_T}:}\\
 Consider the quotient vector space $W = \mathbb{R}^{S}/\text{Span}(e)$ where $\text{Span}(e)$ is the linear span of vector $e$ i.e., the intersection of all vector spaces containing $e$: $\text{Span}(e) = \left\lbrace \lambda e:~ \lambda \in \mathbb{R} \right\rbrace$. The quotient space $W$ is a vector space with dimension $S -1$ (it is in bijection with $\mathbb{R}^{S-1} \times \lbrace 0 \rbrace$, where one coordinate is set to $0$ and the others are free real variables). Since $\text{Span}(e)$ is the null space of the \textit{semi-norm} $\SP{\cdot}$, then $\SP{\cdot}$ is indeed a \textit{norm} on $W$ and thus $\left( W, \SP{\cdot} \right)$ is a normed vector space. The operator $\opT{}$ is well-defined also on $\left( W, \SP{\cdot} \right)$ because of property \textit{(c)} of Lem.~\ref{L:T_properties} (linearity of $\opT$): $\forall h \in \mathbb{R}^{S}, ~\opT{(h + \lambda e)} = \opT{h} + \lambda e$ implying that for any given $w \in W$, the vector $\opT{w} \in W$ is uniquely defined (\ie there is no ambiguity in the definition of $\opT$). Moreover, if $h \in \mathbb{R}^S$ maps to $w \in W$ then $\opT h \in \mathbb{R}^S$ maps to $\opT w \in W$. Since $\opT{}$ is a span contraction, then $\opT{}$ has a unique fixed point $w^+$ in $W$ by Banach fixed-point theorem, which corresponds to the optimality equation $\opT{w^+} = w^+$ (in $W$). Let $h^+\in\mathbb{R}^{S}$ be an arbitrary (bounded) vector in the original space that maps to $w^+ \in W$. Since $\opT h^+ \in \mathbb{R}^S$ maps to $\opT w^+ \in W$ and $\opT{w^+} = w^+$ we have that $\opT{h^+}$ and $h^+$ differ only by a constant vector \ie $\SP{\opT{h^+} - h^+} =0$ or in other words, there exists a constant $g^+ \in \mathbb{R}$ such that $\opT{h^+} = h^+ +g^+e$ which proves the existence of the solution of optimality equation~\eqref{E:optimality_eq_average_T}. Any other solution $h'\in \mathbb{R}^S$ to this equation  will necessarily map to $w^+\in W$ by uniqueness of the solution in $W$ and so $\SP{h^+ - h'} =0$. As a result, the fixed point property of $\opT{}$ in $W$ translates into a fixed point \textit{up to a constant vector} in $\mathbb{R}^{S}$, which leads to the optimality equation $\opT{h^+} = h^+ +g^+e$ as in~\eqref{E:optimality_eq_average_T} where $h^+$ is defined up to a constant. Furthermore, let $(g^+_1,h^+_1)$ and $(g^+_2,h^+_2)$ be two solutions of~\eqref{E:optimality_eq_average_T}. Since there exists a $\lambda \in \mathbb{R}$ such that $h^+_2 = h^+_1 + \lambda e$ we have
\begin{align*}
g_2^+e = \opT{h_2^+} -h^+_2 = \opT{(h^+_1 + \lambda  e)} - (h^+_1 + \lambda  e) = \opT{h^+_1} + \lambda  e -  h^+_1 - \lambda  e = \opT{ h^+_1 }-  h^+_1 = g^+_1 e,
\end{align*}
where we used property \textit{(c)} of Lemma~\ref{L:T_properties}, which leads to $g_1^+ = g_2^+$ and thus the uniqueness of $g^+$ in~\eqref{E:optimality_eq_average_T}.
\item \emph{Convergence of (relative) value iteration:}\\
Fix an arbitrary state $\overline{s} \in \mathcal{S}$ and any initial vector  $v_0 = v \in \mathbb{R}^{S}$, the relative value iteration algorithm implemented by \regopt proceeds through iterations as
\begin{align}\label{eq:rel.v.iteration}
v_n = \opT{} v_{n-1} - (\opT{} v_{n-1})(\overline{s}) e = \opT{}^n v - \left(\opT{}^n v\right) (\overline{s}) e.
\end{align}
The last equality in~\eqref{eq:rel.v.iteration} can be proved by induction on $n\geq 1$: it is trivially true for $n=1$ and assuming that for a given $n\geq1$ it holds that $ v_n= \opT{}^n v - \left(\opT{}^n v\right) (\overline{s}) e$, then 
\begin{align*}
 v_{n+1} = \opT{} v_{n} - (\opT{} v_{n})(\overline{s}) e &= \opT{( \opT{}^n v - \left(\opT{}^n v\right) (\overline{s}) e)} -\left( \opT\left(\opT{}^n v - \left(\opT{}^n v\right) (\overline{s}) e \right)\right)(\overline{s}) e\\
 &= \opT{}^{n+1} v - \left(\opT{}^n v\right) (\overline{s}) e  - \left(\opT{}^{n+1} v\right) (\overline{s}) e + \left(\opT{}^n v\right) (\overline{s}) e\\
 &= \opT{}^{n+1} v -  \left(\opT{}^{n+1} v\right) (\overline{s}) e
\end{align*}
where we used the linearity of $\opT$ (property \textit{(c)} of Lem.~\ref{L:T_properties}).

Denote by $q_n = v_{n+1} - v_n$. Since $v_{n+1}(\overline{s}) = v_n(\overline{s}) = 0$, then $q_n(\overline{s}) =0$ and the absolute value of any component $q_n(s)$ can be upper-bounded by its span.~\footnote{Since $q_n$ takes the value $0$ in $\wb s$, there are only three possible scenarios: \textit{1)} $q_n$ is non-negative (with 0 included) and then $\max\{q_n\} = \SP{q_n}$, \textit{2)} $q_n$ is non-positive (with 0 included) and then $\max\{|q_n|\} = 0 - \min\{q_n\} = \SP{q_n}$, \textit{3)} $q_n$ has both positive and negative values and then both its maximal and minimal value are smaller than the span.} As a result, we have
\begin{align*}
|q_n(s)| \leq \|q_n\|_\infty \leq \SP{q_n}.
\end{align*}
Using the span contraction property of $\opT$ we have that
\begin{align*}
        |v_{n+1}(s) - v_n(s)|
        &\leq \SP{v_{n+1} - v_n} = \SP{\opT{v_n} - (\opT{v_n})e - \opT{v_{n-1}} + (\opT{v_{n-1}})e}\\
        &\stackrel{(a)}{=} \SP{\opT{v_n} - \opT{v_{n-1}}} \leq \gamma \SP{v_n - v_{n-1}} = \gamma \SP{q_{n-1}}\\
        &\stackrel{(b)}{\leq} \gamma^n \SP{\opT{v} - v},
\end{align*}
where \textit{(a)} follows from the fact that $\SP{f} = \SP{f + \lambda e}$ for any $\lambda\in\Re$ and \textit{(b)} is obtained by iterating the first inequality. Since $\SP{\opT{v} - v}$ is bounded and $\gamma < 1$, we can conclude that $\{v_{n+1} - v_n\}_n$ is a convergent sequence. Now we show that $\{v_n\}_n$ is a Cauchy sequence. Let $m > n$, then the following inequalities hold
\begin{align}\label{eq:cauchy}
        |v_{m}(s) - v_n(s)|
        &\leq \SP{v_{m} - v_n} = \SP{v_{m} - v_{m-1} + v_{m-1} - v_{m-2} + \ldots + v_{n+1} - v_n} \\
        &\leq \SP{v_{m} - v_{m-1}} + \SP{v_{m-1} - v_{m-2}} + \ldots + \SP{v_{n+1} - v_n} \nonumber\\
        &\stackrel{(a)}{\leq} \gamma^{m-1} \SP{\opT{v} - v} + \gamma^{m-2} \SP{\opT{v} - v} + \ldots + \gamma^{n} \SP{\opT{v} - v} \nonumber\\
        &= \gamma^n \SP{\opT{v} - v} \sum_{k=0}^{m-n-1}\gamma^{k} \leq \gamma^n \SP{\opT{v} - v} \sum_{k=0}^{\infty}\gamma^{k} = \frac{\gamma^n}{1-\gamma} \SP{\opT{v} - v},\nonumber
\end{align}
where \textit{(a)} is the application of the previous inequality $\SP{v_{n+1} - v_n} \leq \gamma^n \SP{\opT{v} - v}$. Since $\gamma <1$, for any arbitrary $\varepsilon>0$, there exists a $N_{\varepsilon}$, so that for any $m>n>N_\varepsilon$, $\|v_m - v_n\|_\infty \leq \varepsilon$. As a result $\{v_n\}_n$ is a Cauchy sequence and since $(\Re^S,\|\cdot\|_\infty)$ is a Banach space, $v_n$ converges to a vector that we denote by $h(v,\overline{s})$. We now show that $h(v,\overline{s})$ satisfies the optimality equation. Using property~\emph{(c)} in Lem.~\ref{L:T_properties} and Eq.~\ref{eq:rel.v.iteration}, we can write
\begin{align}\label{eq:rel.vi.eq}
\opT{v_n} = \opT{\big(\opT^nv - (\opT^nv)(\wb s)e\big)}
&=  \opT{^{n+1}} v - (\opT{}^n v) (\overline{s}) e = v_{n+1} + \left( \opT{^{n+1}v} - \opT{^nv} \right)(\overline{s})e.
\end{align}
Then,
\begin{align}\label{eq:rel.vi.eq2}
\SP{\opT{v_n} - v_{n+1}} = \SP{\big( \opT{^{n+1}v} - \opT{^nv}) \big)(\overline{s})e} = 0.
\end{align}
By continuity of the semi-norm $\SP{\cdot}$ and uniqueness of the limit this implies that
\begin{equation*}
        \label{E:limtozero_T}
\lim_{n\to+\infty} \SP{\opT{v_n} - v_{n+1}} = \SP{\opT{h(v,\overline{s})} - h(v,\overline{s})} = 0,
\end{equation*}
where we used the fact that the sequences $v_{n}$ and $v_{n+1}$ converge to $h(v,\overline{s})$. Since $\opT{h(v,\overline{s})} - h(v,\overline{s})$ has zero span, we conclude that there exists a constant value $g \in\mathbb{R}$ such that $\opT{h(v,\overline{s})} = h(v,\overline{s}) + ge$, which is indeed optimality equation~\eqref{E:optimality_eq_average_T} and by uniqueness of the solution, $g = g^+$. This proves that relative value iteration using $\opT$ does converge to a solution of the optimality equation. 

Alternatively, we can prove that standard value iteration converges in the sense that $\lim_{n \rightarrow + \infty} \opT{}^{n+1} v - \opT{}^n v = g^+e$. Using the continuity of $\opT$ and the fact that $\lim_{n\to +\infty} v_n = \lim_{n\to +\infty} v_{n+1} = h(v,\overline{s})$, we can write
\begin{equation}\label{E:limtozero_T2}
\lim_{n \rightarrow + \infty} \opT{v_n} - v_{n+1} = \opT{h(v,\overline{s})} - h(v,\overline{s}) = g^+ e.
\end{equation}
Using the definition of relative value iteration ($v_n = \opT{}^n v - \opT{}^n v(\overline{s})e$) and the linearity of $\opT$ we have
\begin{align*}
 \opT{}^{n+1}v-\opT{}^n v &= \opT{}(v_n + (T_c^n v)(\overline{s}) e)-(v_n + (T_c^n v)(\overline{s}) e)\\ &= \opT{}v_n - v_n  \\&=  \underbrace{v_{n+1} - v_n}_{\underset{n \rightarrow + \infty}{\longrightarrow} 0} + \underbrace{\opT{}v_n - v_{n+1}}_{\underset{n \rightarrow + \infty}{\longrightarrow} g^+ e} \underset{n \rightarrow + \infty}{\longrightarrow} g^+ e,
\end{align*}
where the limits rely on the convergence of $v_n$ and Eq.~\ref{E:limtozero_T2}.
\item \emph{Dominance} $g^+ \geq g_c^*$\emph{:}\\
Let $\pi = d^{\infty}\in \Pi_c$ be a policy with constant gain and bounded bias span. The evaluation Bellman equation gives $L_d h^{\pi} = h^{\pi} + g^{\pi} e$ and $\SP{L_d h^{\pi}} = \SP{h^{\pi}} \leq c$. This implies that $d \in D(c,h^{\pi}) \neq \emptyset$ and so from Lemma~\ref{L:TRpolicyexistence} we have
\begin{align*}
  \opT{h^{\pi}} \geq h^{\pi} + g^{\pi} e.
\end{align*}
By monotonicity and linearity of $\opT{}$ (properties \emph{(a)} and \emph{(c)}) of Lemma~\ref{L:T_properties}) we have
\begin{align*}
\opT{h^{\pi}} \geq h^{\pi} + g^{\pi} e \implies \opT{}^2 h^{\pi} \geq \opT{}(h^{\pi} + g^{\pi} e) = \opT{}h^{\pi} + g^{\pi} e \geq h^{\pi} + 2 g^{\pi} e,
\end{align*}
As a result, we can iterate the inequality and obtain for all $n \in \mathbb{N}$:
\begin{align*}
\opT{}^{n+1} h^{\pi} \geq \opT{}^{n}h^{\pi} + g^{\pi} e \implies \underbrace{\opT{}^{n+1} h^{\pi} - \opT{}^{n}h^{\pi}}_{\underset{n \rightarrow + \infty}{\longrightarrow} g^+ e} \geq g^{\pi} e \implies g^+ \geq g^{\pi},
 \end{align*}
where we used property \emph{2.} of Lem.~\ref{lem:optimality.equation.T} proved above. Since the inequality holds for any $\pi \in \Pi_c$, it also holds for the supremum $\sup_{\pi \in \PiC} g^\pi = g_c^*$ (solution to problem~\eqref{P:opt_probl_well_posed}) \ie $g^+ \geq g^*_c$.
\end{enumerate}


\subsection{Approximation guarantees of \regopt (proof of Theorem~\ref{thm:opT.optimality})}
\label{sapp:opT.optimality}

In this section we prove a slightly more general statement than Thm.~\ref{thm:opT.optimality}.

We use operators $\opG{}$ and $\opN{}$ defined in Def.~\ref{def:Gop} and App.~\ref{app:counterexamples} respectively. We recall that when $\opT{}$ is globally feasible at $v$, then $\opN{v} = \opT{v}$ and $\opG{v} = \delta^+_v$. We first slightly relax Asm.~\ref{asm:feasibility} (Asm.~\ref{asm:feasibility2} below) and then prove a generalisation of Thm.~\ref{thm:opT.optimality} (Thm.~\ref{thm:opT.optimality2} below).

\begin{assumption}\label{asm:feasibility2}
Operator $\opT{}$ is globally feasible at $h^+$, i.e., the decision rule $d^+ = \opG{h^+}$ is such that $\opT{h^+} = L_{d^+}{h^+} = \opN{h^+} = h^+ + g^+ e$.
\end{assumption}

\begin{theorem}\label{thm:opT.optimality2}
Assume Asm.~\ref{asm:L.contraction} and~\ref{asm:feasibility2} hold and let
 \begin{align*}
  M_n^+ &= \frac{2 \gamma^n}{1-\gamma} \SP{v_1 - v_0} + \max_s \left\lbrace \opT{v_n}(s) -v_n(s) \right\rbrace\\
  m_n^+ &= \frac{- 2 \gamma^n}{1-\gamma} \SP{v_1 - v_0} + \min_s \left\lbrace \opT{v_n}(s) -v_n(s) \right\rbrace \\
  M_n &= \max_s \left\lbrace \opN{v_n}(s) -v_n(s) \right\rbrace \\
  m_n &= \min_s \left\lbrace \opN{v_n}(s) -v_n(s) \right\rbrace,
 \end{align*}
and $d_n = \opG{v_n}$ be the decision rule computed after $n$ iterations and $\pi_n = (d_n)^\infty$ the corresponding policy. Then we have%
 \begin{align}
  &1) ~ \Big\| g^{\pi_n} - \frac{1}{2} \left( M_n + m_n \right) e \Big\|_{\infty}\leq \frac{1}{2} \left( M_n - m_n \right) = \frac{1}{2} \SP{\opN{v_n} -v_n}\\
  &2) ~ \Big| g^+ - \frac{1}{2} \left( M_n^+ + m_n^+ \right) \Big| \leq \frac{1}{2} \left( M_n^+ - m_n^+ \right) = \frac{2 \gamma^n}{1-\gamma} \SP{v_1 - v_0} + \frac{1}{2}\SP{\opT{v_n} -v_n}\\
  &3) ~ \Big\| g^+e - g^{\pi_n} \Big\|_\infty \leq \max\left\{ M_n^+ - m_n, \; M_n-m_n^+ \right\}
 \end{align}
Moreover, if in addition $\pi^+ = (d^+)^\infty$ is unichain then $g^+$ is the solution to optimization problem~\eqref{P:opt_probl_well_posed}, i.e., $g^+ = g^*_c$ and $\pi^+ \in \PiC^*$.
\end{theorem}
\begin{proof}
We first analyse the convergence error. Let $\pi_n$ be the policy associated to the decision rule $d_n = \opG{v_n}$.
We prove the three convergence statements of the theorem.
\begin{enumerate}
 \item We recall that operator $\opN{}$ is such that $L_{d_n} v_n = L_{\opG{v_n}} v_n = \opN{v_n}$. By definition of the gain $g^{\pi_n}$, there exists a \emph{stationary transition matrix} $P^*_{d_n}:=C\text{-}\lim_{k \to +\infty} (P_{d_n})^k$ \citep[Appendix A]{puterman1994markov} such that $g^{\pi_n} = P^*_{d_n} r_{d_n}$. Furthermore, since $P^*_{d_n} P_{d_n} = P^*_{d_n}$, for any vector $v_n$, $g^{\pi_n} = P^*_{d_n} (r_{d_n} + P_{d_n} v_n - v_n ) = P^*_{d_n} (L_{d_n} v_n - v_n )$. Since $\min \left\lbrace L_{d_n} v_n -v_n \right\rbrace e \leq L_{d_n} v_n - v_n \leq \max \left\lbrace L_{d_n} v_n -v_n \right\rbrace e$ and by multiplying these two inequalities by $P^*_{d_n}$ (which is a stochastic matrix) we obtain $m_n e \leq g^{\pi_n} \leq M_n e $ and the result holds.
 \item Using Eq.~\ref{eq:cauchy} and letting $m\rightarrow \infty$, we have $ \| h^+ - v_n \|_{\infty} \leq \frac{\gamma^n}{1-\gamma}\SP{v_1 - v_0} $. Therefore by monotonicity and linearity of $\opT{}$ (property \emph{(a)} and \emph{(c)} of Lem.~\ref{L:T_properties}) and using optimality equation $\opT{h^+} = h^+ + g^+e$:
 \begin{align*}
  &- \frac{\gamma^n}{1-\gamma}\SP{v_1 - v_0}e + h^+ \leq v_n \leq h^+ + \frac{\gamma^n}{1-\gamma}\SP{v_1 - v_0}e\\
  \implies &- \frac{\gamma^n}{1-\gamma}\SP{v_1 - v_0}e + h^+ + g^+ e \leq \opT{v_n} \leq h^+ + g^+e + \frac{\gamma^n}{1-\gamma}\SP{v_1 - v_0}e \\
  \implies &- \frac{2 \gamma^n}{1-\gamma}\SP{v_1 - v_0}e + g^+ e  \leq \opT{v_n} - v_n \leq g^+ e  + \frac{2 \gamma^n}{1-\gamma}\SP{v_1 - v_0}e\\
  \implies &- \frac{2 \gamma^n}{1-\gamma}\SP{v_1 - v_0}e + \opT{v_n} - v_n \leq g^+ e \leq \opT{v_n} - v_n + \frac{2 \gamma^n}{1-\gamma}\SP{v_1 - v_0}e\\
  \implies  &~m_n^+ \leq g^+ \leq M_n^+
 \end{align*}
  and the result holds.
  \item The last inequality is a direct consequence of the two inequalities previously proved:
 \begin{align*}
  m_n e \leq g^{\pi_n} \leq M_n e ~~
  \text{and}~~m_n^+ \leq g^+ \leq M_n^+.
 \end{align*}
\end{enumerate}

We now prove optimality. Under the global feasibility assumption at $h^+$ (Asm.~\ref{asm:feasibility2}), we have that there exists a decision rule $d^+$ and an associated policy $\pi^+ = (d^+)^\infty$ such that
\begin{align}\label{eq:optim.eq.feasible}
\opT{h^+} = L_{d^+}h^+ = h^+ + g^+e,
\end{align}
%
 Since $(g^+,h^+)$ is a solution of the Bellman evaluation equations \eqref{eq:eval.equations} associated to $\pi^+ = (d^+)^\infty$ and since by assumption $\pi^+$ is unichain, Corollary 8.2.7. of \citet{puterman1994markov} holds and so $g^+ = g^{\pi^+}$ and there exists $\lambda \in \Re$ such that $h^{\pi^+} = h^+ + \lambda e$ implying that
%
\begin{align*}
\SP{h^{\pi^+}} = \SP{h^{+} + \lambda e} = \SP{h^{+}} = \SP{h^{+} + g^{+}e} = \SP{\opT{h^{+}}}\leq c,
\end{align*}
where we used the invariance of the span by translation, Eq.~\ref{eq:optim.eq.feasible} and the definition of $\opT$. As a result, $\pi^+ \in \PiC$ and by property \emph{3.} of Lem.~\ref{lem:optimality.equation.T} we can conclude that 
\begin{align*}
g^{\pi^+} = g^+ \geq g^\pi, \quad \forall \pi\in\PiC,
\end{align*}
which implies that $\pi^+ \in \PiC^*$ and $g^{\pi^+} = g^*_c$. Note that if $\pi^+$ is not unichain then we might have $sp\{h^{\pi^+}\} > \SP{h^{+}}$ in which case it is possible that $\pi^+ \notin \PiC$ and so the result does not hold.
\end{proof}

We now relate Asm.~\ref{asm:feasibility2} and Thm.~\ref{thm:opT.optimality2} to (respectively) Asm.~\ref{asm:feasibility} and Thm.~\ref{thm:opT.optimality}. As we just showed in the proof of Thm.~\ref{thm:opT.optimality2}, we always have $\SP{h^{+}} \leq c$ and therefore, whenever Asm.~\ref{asm:feasibility} holds, Asm.~\ref{asm:feasibility2} holds too. As a result, if Asm.~\ref{asm:L.contraction} also holds then the first part of Thm.~\ref{thm:opT.optimality2} holds too. Moreover, if $\SP{v_0} \leq c$ it is straightforward to see that $\SP{v_n} \leq c$ for any $n\geq 1$ and so due to Asm.~\ref{asm:feasibility}, $\opT{v_n} = \opN{v_n}$ for all $n \geq 1$. This implies that \[\max\left\{ M_n^+ - m_n, \; M_n-m_n^+ \right\} = \SP{\opT{v_n} -v_n} + \frac{2 \gamma^n}{1-\gamma} \SP{v_1 - v_0} = \SP{v_{n+1} -v_n} + \frac{2 \gamma^n}{1-\gamma} \SP{v_1 - v_0} \]
and as a consequence Thm.~\ref{thm:opT.optimality} holds. 


\section{Modified bounded-parameter (extended) MDPs (Proof of Thm.~\ref{th:augmented.perturbed.mdp})}\label{app:augmented.perturbed.mdp}

In this section we prove Thm.~\ref{th:augmented.perturbed.mdp}.

We analyse separately the two modifications introduced in Def.~\ref{def:augmented.perturbed.mdp} on the rewards and transition probabilities (transition ``\emph{kernel}''). For a given bounded-parameter (extended) MDP $\wt{\mathcal{M}}$, we denote by $\wt{\mathcal{M}}_\eta$ the \emph{perturbed} bounded-parameter MDP whose transition kernel is an $\eta$-perturbation of the original one (see formal definition in Lem.~\ref{L:pperturbed_bmdp} below), and by $\wt{\mathcal{M}}^{\downarrow}$ the \emph{augmented} bounded-parameter MDP whose reward intervals are extended from below compared to the original ones (the maximum is not changed while the lower bound of the interval is set to zero, see formal definition in Lem.~\ref{L:r_augmented_bmdp} below). We first prove interesting properties for operators $\wt{T}_c^{\eta}$ associated to $\wt{\mathcal{M}}_\eta$ (Lem.~\ref{L:pperturbed_bmdp}) and $\wt{T}_c^{\downarrow}$ associated to $\wt{\mathcal{M}}^{\downarrow}$ (Lem.~\ref{L:r_augmented_bmdp}). We then consider the MDP $\wt{\mathcal{M}}_\eta^{\downarrow}$ that is both \emph{augmented} and \emph{perturbed} and we present the properties of the corresponding operator $\wt{T}_c^{\eta,\downarrow}$ in Thm.~\ref{T:improved_op}. Note that in Sec.~\ref{sec:regret} we called the augmented and perturbed MDP ``\emph{modified MDP}'' for simplicity, and we used the notation $\wt{\mathcal{M}}^\ddagger$ instead of $\wt{\mathcal{M}}_\eta^{\downarrow}$ for clarity. Thm.~\ref{th:augmented.perturbed.mdp} in Sec.~\ref{sec:regret} is thus equivalent to Thm.~\ref{T:improved_op} stated below.

In the following, for any closed interval $[a,b] \subset \Re$ we use the notations $\min\{ [a,b]\} := a$ and $\max\{ [a,b]\} := b$.

\begin{lemma}\label{L:pperturbed_bmdp}
 Let $\wt{\mathcal{M}}$ be a bounded-parameter MDP defined for all $s, s' \in \mathcal{S}$ and all $a \in \mathcal{A}_s$ by
 \begin{align*}
  r(s,a) \in B_r(s,a) ~~ \text{and} ~~  p(s'|s,a) \in B_p(s,a,s')
 \end{align*}
 where $\mathcal{S}$ and $\mathcal{A}_s$ are finite, $B_r(s,a)$ and $B_p(s,a,s')$ are closed intervals of $[0,\rmaxbound]$ and $[0,1]$ respectively.
 Let $1 \geq \eta >0$ and $\overline{s} \in \mathcal{S}$ and consider the ``\emph{perturbed}'' bounded-parameter MDP $\wt{\mathcal{M}}_{\eta}$ defined $\forall s,s' \in \mathcal{S}$ and $\forall a \in \mathcal{A}_s$ by:
 \begin{align*}
  B_r^\eta(s,a) = B_r(s,a) ~~ \text{and} ~~ B_p^\eta(s,a,s') = \begin{cases}B_p(s,a,s')~~ \text{if}~~s'\neq \overline{s} \\ \text{and}~~B_p(s,a,\overline{s}) \cap [\eta, 1]~~ \text{if}~~s'= \overline{s} \end{cases}
 \end{align*}
 where we assume that $\eta$ is small enough so that $\forall s \in \mathcal{S}$ and $\forall a \in \mathcal{A}_s$,
 \begin{enumerate}[nosep]
  \item $B_p(s,a,\overline{s}) \cap [\eta, 1] \neq \emptyset$ 
  \item $\sum_{s' \in \mathcal{S}} \min \{B_p^\eta(s,a,s')\} \leq 1$ and $\sum_{s' \in \mathcal{S}} \max \{B_p(s,a,s')\} \geq 1$ 
 \end{enumerate}
 If $\wt{L}$ denotes the optimal Bellman operator of $\wt{\mathcal{M}}$ and $\wt{L}_{\eta}$ the optimal Bellman operator of $\wt{\mathcal{M}}_{\eta}$, then $\forall v\in \mathbb{R}^{S}$:
 \begin{align}
  \big\| \wt{L} v - \wt{L}_\eta v \big\|_\infty \leq \SP{v}\eta 
 \end{align}
 Moreover, $\wt{L}_\eta$ is a $\gamma$-span contraction with $\gamma \leq 1- \eta <1$ and $\wt{\mathcal{M}}_{\eta}$ is unichain.
\end{lemma}
\begin{proof}
 For all states $s \in \mathcal{S}$ and actions $a \in \mathcal{A}_s$ we use the following notations
 \begin{align}
  \widetilde{r}(s,a) {:=} \max\{B_r(s,a)\} ~~ \text{and} ~~ \widetilde{p}(\cdot|s,a) {:=} \argmax_{p(s') \in B_p(s,a,s')} \sum_{s' \in \mathcal{S}} p(s')v(s')
 \end{align}
 and we define $\widetilde{r}_\eta(s,a)$ and $\widetilde{p}_\eta(\cdot|s,a)$ similarly with $B_r(s,a)$ and $B_p(s,a,s')$ replaced by $B_r^\eta(s,a)$ and $B_p^\eta(s,a,s')$.
 \begin{align*}
  \forall s \in \mathcal{S}, ~ |Lv(s) - \wt{L}_\eta v(s)| &= \left|\max_{a \in \mathcal{A}_s} \left\lbrace \widetilde{r}(s,a) +  \sum_{s' \in \mathcal{S}} \widetilde{p}(s'|s,a) v(s') \right\rbrace - \max_{a \in \mathcal{A}_s} \left\lbrace \widetilde{r}_\eta(s,a) +  \sum_{s' \in \mathcal{S}} \widetilde{p}_\eta(s'|s,a) v(s') \right\rbrace \right|\\
  &\leq \max_{a \in \mathcal{A}_s} \left| \underbrace{\widetilde{r}(s,a) - \widetilde{r}_\eta(s,a)}_{=0} + \sum_{s' \in \mathcal{S}} (\widetilde{p}(s'|s,a) - \widetilde{p}_\eta(s'|s,a))v(s') \right|
 \end{align*}
  where we used the fact that $|\max_x f(x) - \max_x g(x)| \leq \max_x |f(x) - g(x)| $ and $B_r^\eta(s,a) =B_r(s,a)$ by definition. Since $\widetilde{p}(\cdot| s,a)$ and $\widetilde{p}_\eta(\cdot|s,a)$ are probability distributions (i.e., sum to 1), for any real $\lambda$: $$\sum_{s' \in \mathcal{S}} (\widetilde{p}(s'|s,a) - \widetilde{p}_\eta(s'|s,a))v(s')= \sum_{s' \in \mathcal{S}} (\widetilde{p}(s'|s,a) - \widetilde{p}_\eta(s'|s,a))(v(s') + \lambda)$$
  Taking $\lambda = -(\max_s v(s) + \min_s v(s))/2$ we obtain:
  \begin{align*}
  \forall s \in \mathcal{S}, ~ |Lv(s) - \wt{L}_\eta v(s)| &\leq \max_{a \in \mathcal{A}_s} \sum_{s' \in \mathcal{S}} |\widetilde{p}(s'|s,a) - \widetilde{p}_\eta(s'|s,a)| \cdot \max_s \{v(s) +\lambda\}\\
  &=\max_{a \in \mathcal{A}_s} \| \widetilde{p}(\cdot|s,a) - \widetilde{p}_\eta(\cdot|s,a)\|_1 \cdot (\max_s \{v(s)\} +\lambda)\\
  &= \max_{a \in \mathcal{A}_s} \| \widetilde{p}(\cdot|s,a) - \widetilde{p}_\eta(\cdot|s,a)\|_1 \cdot \frac{\SP{v}}{2}
 \end{align*}
  We now need to upper-bound $\| \widetilde{p}(\cdot|s,a) - \widetilde{p}_\eta(\cdot|s,a)\|_1$.
  $\widetilde{p}(\cdot|s,a)$ and $\widetilde{p}_\eta(\cdot|s,a)$ can be computed using the following procedure \citep[Appendix A]{DannBrunskill15}:
  \begin{enumerate}
   \item Assume without loss of generality that the coordinates of $v$ are sorted in decreasing order: $v(s_1) \geq v(s_2) \geq ... \geq v(s_n)$
   \item Initialise $\widetilde{p}^0(s'|s,a) = \min \{B_p(s,a,s')\}$, $\Delta^{0} = 1 - \sum_{s'\in\mathcal{S}}\widetilde{p}^0(s'|s,a)$ for all $s' \in \mathcal{S}$, and $i=1$ 
   \item While $\Delta^{i-1} > 0$ do
   \begin{itemize}
   \item $ \delta^i \leftarrow \min \big\lbrace \Delta^{i-1}; \max \{B_p(s,a,s_i)\} - \widetilde{p}^{i-1}(s_i|s,a)\big\rbrace$
    \item $\widetilde{p}^{i}(s_i|s,a) \leftarrow \widetilde{p}^{i-1}(s_i|s,a)+ \delta^i $
    \item $\Delta^{i} \leftarrow \Delta^{i-1} - \delta^i$
    \item $i \leftarrow i+1$
   \end{itemize}
   \item Return $\widetilde{p}(\cdot|s,a) = \widetilde{p}^{i-1}(\cdot|s,a)$
  \end{enumerate}
  Let's now show that at any iteration of the above procedure $\widetilde{p}(\cdot|s,a)$ and $\widetilde{p}_\eta(\cdot|s,a)$ are at most $2\eta$-far in $\ell_1$-norm.
  Notice that at the end of iteration $i$, the vector $\widetilde{p}^{i}(\cdot|s,a)$ differs from $\widetilde{p}^{i-1}(\cdot|s,a)$ only in state $s_i$. In the following we use index $\eta$ to denote the quantities obtained when the above procedure is applied with $B_p^\eta(s,a,s')$ instead of $B_p(s,a,s')$ (the output is then $\widetilde{p}_\eta(\cdot|s,a)$). The conditions $B_p(s,a,\overline{s}) \cap [\eta, 1] \neq \emptyset$, $\sum_{s' \in \mathcal{S}} \min \{B_p^\eta(s,a,s')\} \leq 1$ and $\sum_{s' \in \mathcal{S}} \max \{B_p(s,a,s')\} \geq 1$ ensure that the procedure doesn't stop prematurely when $B_p^\eta(s,a,s')$ replaces $B_p(s,a,s')$. Indeed, when they hold there exists a vector $p$ satisfying $\sum_{s' \in \calS}p(s') = 1$ and $\forall s' \in \calS, ~ p(s') \in B_p^\eta(s,a,s')$.
  \begin{itemize}
          \item \textbf{$\bm{i=0}$ (initialization):} By definition, for any $s' \neq \overline{s}$, $B_p^\eta(s,a,s') = B_p(s,a,s')$ implying that $\widetilde{p}^0(s'|s,a) = \widetilde{p}_\eta^0(s'|s,a)$ and moreover $\eta \leq\min \{B_p^\eta(s,a,\overline{s})\} \leq \eta + \min \{B_p(s,a,\overline{s})\}$ implying that $\eta \leq \widetilde{p}^0_\eta(\overline{s}|s,a) \leq \eta + \widetilde{p}^0(\overline{s}|s,a)$ and thus
   $\Delta^0_\eta +\eta \geq \Delta^0 \geq \Delta^0_\eta$ and 
   \[ \| \widetilde{p}^0(\cdot|s,a) - \widetilde{p}_\eta^0(\cdot|s,a) \|_1 = |\widetilde{p}^0(\overline{s}|s,a) - \widetilde{p}_\eta^0(\overline{s}|s,a)| \leq \eta \]
   \end{itemize}
   When the procedure to compute $\widetilde{p}_\eta(\cdot|s,a)$ stops, there are only two possibilities: state $\overline{s}$ is updated either before or after $\Delta_\eta = 0$. In the following we analyze separately the two cases.
   \begin{itemize}
           \item \textbf{$\bm{\Delta_\eta = 0}$ occurs first:} \ie $\Delta^{i}_\eta = 0$ and for $k=1,...,i$, $s_k \neq \overline{s}$ and $\Delta^{k-1}_\eta > 0$.\\
                   As a consequence, we have that $\widetilde{p}_\eta(\cdot|s,a) = \widetilde{p}_\eta^{i}(\cdot|s,a)$ and by triangle inequality:
   \begin{align}\label{eqn:temp.eq}
    \| \widetilde{p}(\cdot|s,a) - \widetilde{p}_\eta(\cdot|s,a)\|_1 \leq \| \widetilde{p}^{i}(\cdot|s,a) - \widetilde{p}_\eta^{i}(\cdot|s,a)\|_1 + \| \widetilde{p}^{i}(\cdot|s,a) - \widetilde{p}(\cdot|s,a)\|_1
   \end{align}
   By assumption for $k=1,...,i-1$, $\Delta^{k-1}_\eta > 0$, we have that $\widetilde{p}^k_\eta(s_k|s,a) = \max \{B_p^\eta(s,a,s_k)\} = \max \{B_p(s,a,s_k)\}$. Moreover, for all $k=1,...,i-1$, ~$s_k \neq \overline{s}$ and thus by trivial induction we have that $\Delta^{k} - \Delta^{k}_\eta = \Delta^0 - \Delta^0_\eta$, $\Delta^{k-1} \geq \Delta^{k-1}_\eta > 0$ and  ~$\widetilde{p}^k(s_k|s,a) =\max \{B_p(s,a,s_k)\} = \widetilde{p}^k_\eta(s_k|s,a)$. Then:  
   \[ \forall k=1,...,i-1, \quad \| \widetilde{p}^k(\cdot|s,a) - \widetilde{p}^k_\eta(\cdot|s,a)\|_1 = \Delta^0 - \Delta^0_\eta \leq \eta\]
   and thus $\| \widetilde{p}^{i}(\cdot|s,a) - \widetilde{p}^{i}_\eta(\cdot|s,a)\|_1 = \Delta^0 - \Delta^0_\eta + \delta^{i} - \delta^{i}_\eta$.
   Since $\| \widetilde{p}^{i}(\cdot|s,a) - \widetilde{p}(\cdot|s,a)\|_1 = \Delta^{i}$, after incorporating everything into~\eqref{eqn:temp.eq} we obtain: 
   \begin{align*}
    \| \widetilde{p}(\cdot|s,a) - \widetilde{p}_\eta(\cdot|s,a)\|_1 &\leq \Delta^0 - \Delta^0_\eta - \delta^{i}_\eta + \underbrace{\delta^{i} + \Delta^{i}}_{= \Delta^{i-1}}\\
    &= \Delta^0 - \Delta^0_\eta + \underbrace{\Delta^{i-1}}_{= \Delta^{i-1}_ \eta+\Delta^0 - \Delta^0_\eta} - \delta^{i}_\eta = 2(\Delta^0 - \Delta^0_\eta) + \underbrace{\Delta^{i-1}_ \eta - \delta^{i}_\eta }_{= \Delta^{i}_\eta = 0}\\
    &= 2(\Delta^0 - \Delta^0_\eta) \leq 2\eta
   \end{align*}
   \item \textbf{$\bm{\overline{s}}$ is updated first:} \ie $s_i = \overline{s}$ and for $k=1,...,i-1$, $s_k \neq \overline{s}$ and $\Delta^{k}_\eta > 0$. By trivial induction we have that for all $k < i$, ~$\Delta^{k} - \Delta^{k}_\eta = \Delta^0 - \Delta^0_\eta$~ and ~$\widetilde{p}^k(s_k|s,a) =\max \{B_p(s,a,s_k)\} = \max \{B_p^\eta(s,a,s_k)\} = \widetilde{p}^k_\eta(s_k|s,a)$. Since $\delta^i_\eta$ is defined as the minimum between two values, there are only two possible cases:
   \begin{itemize}
    \item If $\delta^i_\eta = \max \{B_p^\eta(s,a,s_i)\} - \widetilde{p}_\eta^{i-1}(s_i|s,a) \leq \Delta^{i-1}_\eta$ we have that 
    \begin{align*}
     \max \{B_p(s,a,s_i)\} - \widetilde{p}^{i-1}(s_i|s,a)&= \max \{B_p^\eta(s,a,s_i)\} - \widetilde{p}_\eta^{i-1}(s_i|s,a) +  \Delta^0 - \Delta^0_\eta\\
     &\leq \Delta^{i-1}_\eta +  \Delta^0 - \Delta^0_\eta= \Delta^{i-1}
    \end{align*}
    which implies $\delta^i = \max \{B_p(s,a,s_i)\} - \widetilde{p}^{i-1}(s_i|s,a)$. As a consequence,
    \begin{itemize}
            \item $\widetilde{p}^{i}(s_i|s,a) = \max \{B_p(s,a,s_i)\} = \max \{B_p^\eta(s,a,s_i)\} =\widetilde{p}_\eta^{i}(s_i|s,a)$
            \item $\Delta^{i} = \Delta^{i}_\eta$
    \end{itemize}
     Thus for $k > i$, $\widetilde{p}^{k}(s_k|s,a) = \widetilde{p}_\eta^{k}(s_k|s,a)$ and so $\widetilde{p}(\cdot|s,a) = \widetilde{p}_\eta(\cdot|s,a)$.
     \item If $\delta^i_\eta = \Delta^{i-1}_\eta \leq \max \{B_p^\eta(s,a,s_i)\} - \widetilde{p}_\eta^{i-1}(s_i|s,a) $ we have that
     \begin{align*}
     \widetilde{p}^{i-1}_\eta(s_i|s,a) = \widetilde{p}^{0}_\eta(s_i|s,a) = \widetilde{p}^{0}_\eta(\overline{s}|s,a) = \min\{B_p^\eta(s,a,\overline{s})\} &= \min\{B_p(s,a,\overline{s})\} +\Delta^0 - \Delta^0_\eta\\ &= \widetilde{p}^{0}(\overline{s}|s,a) +\Delta^0 - \Delta^0_\eta\\ &= \widetilde{p}^{0}(s_i|s,a) + \Delta^0 - \Delta^0_\eta\\ &= \widetilde{p}^{i-1}(s_i|s,a) + \Delta^0 - \Delta^0_\eta
     \end{align*}
     implying that
     \begin{align*}
     \Delta^{i-1} = \Delta^{i-1}_\eta + \Delta^0 - \Delta^0_\eta &\leq \max \{B_p^\eta(s,a,s_i)\} - \widetilde{p}^{i-1}_\eta(s_i|s,a) + \Delta^0 - \Delta^0_\eta\\
     &= \max \{B_p(s,a,s_i)\} - \widetilde{p}^{i-1}(s_i|s,a)
    \end{align*}
    which implies $\delta^i = \Delta^{i-1}$. As a consequence,
    \begin{itemize}
            \item $\widetilde{p}^{i}(s_i|s,a) = \widetilde{p}^{i-1}(s_i|s,a) + \delta^i = \underbrace{\widetilde{p}^{i-1}(s_i|s,a) + \Delta^0 - \Delta^0_\eta}_{\widetilde{p}^{i-1}_\eta(\bar{s}|s,a)} + \Delta^{i-1}_\eta = \widetilde{p}_\eta^{i}(s_i|s,a)$
            \item $\Delta^{i} = \Delta^{i}_\eta = 0$
    \end{itemize}
     Thus for $k > i$, $\widetilde{p}^{k}(s_k|s,a) = \widetilde{p}_\eta^{k}(s_k|s,a)$ and so $\widetilde{p}(\cdot|s,a) = \widetilde{p}_\eta(\cdot|s,a)$.
   \end{itemize}
   \end{itemize}
   In conclusion: $\max_{a \in \mathcal{A}_s} \| \widetilde{p}(\cdot|s,a) - \widetilde{p}_\eta(\cdot|s,a)\|_1 \leq 2 \eta$ implying  $\|Lv - \wt{L}_{\eta}v\|_\infty \leq \SP{v} \eta $. 
   
   From~\citep[Theorem 6.6.6]{puterman1994markov}, we know that $\wt{L}_\eta$ is Lipschitz continuous in span semi-norm with Lipschitz constant:
   \begin{align*}
           \gamma &= 1 - \min_{s\in \mathcal{S},u \in \mathcal{S}, a \in \mathcal{A}_s, b \in \mathcal{A}_u} \min_{\widetilde{p}_\eta, \overline{p}_\eta} \left\lbrace \sum_{j \in \mathcal{S}}\min \left\{ \widetilde{p}_\eta(j|s,a),\overline{p}_\eta(j|u,b) \right\} \right\rbrace\\
    &= 1 - \min_{s \in \mathcal{S},u \in \mathcal{S}, a \in \mathcal{A}_s, b \in \mathcal{A}_u}\min_{\widetilde{p}_\eta, \overline{p}_\eta} \left\lbrace \underbrace{\sum_{j \neq \overline{s}}\min \left \lbrace \widetilde{p}_\eta(j|s,a),\overline{p}_\eta(j|u,b)  \right \rbrace}_{\geq 0} + \underbrace{\min \left \lbrace \widetilde{p}_\eta(\overline{s}|s,a),\overline{p}_\eta(\overline{s}|u,b) \right \rbrace}_{\geq \eta} \right\rbrace \leq 1- \eta < 1
   \end{align*}
   Thus $\wt{L}_\eta$ is a $\gamma$-span-contraction with $\gamma \leq 1-\eta <1$. The term $\gamma$ is often referred to as ``\emph{ergodic coefficient}'' in the literature.
   
   Finally, by definition of $\wt{\mathcal{M}}_{\eta}$, for any decision rule $d \in \MR$ and for any state $s \in \calS$:~ $p(\overline{s}|s,d(s))>0$. Assume that the policy $\pi = d^\infty$ associated to $d$ has more than one \emph{recurrent class} and pick $s_1, s_2 \in \calS$ belonging to two different recurrent classes. Since, $p(\overline{s}|s_1,d(s_1))>0$ and $p(\overline{s}|s_2,d(s_2))>0$, necessarily $\overline{s}$ must belong to both recurrent classes which is impossible as two distinct recurrent classes have disjoint state spaces by definition. Therefore $\pi$ is unichain. Since $\pi$ was chosen arbitrarily, $\wt{\mathcal{M}}_{\eta}$ is unichain which concludes the proof.
\end{proof}

We now consider the perturbation of the reward intervals.

\begin{lemma}\label{L:r_augmented_bmdp}
 Let $\wt{\mathcal{M}}$ be a bounded-parameter MDP defined $\forall s, s' \in \mathcal{S}$ and $\forall a \in \mathcal{A}_s$ by:
 \begin{align*}
  r(s,a) \in B_r(s,a) ~~ \text{and} ~~ p(s'|s,a) \in B_p(s,a,s')
 \end{align*}
 where $\mathcal{S}$ and $\mathcal{A}_s$ are finite, $B_r(s,a)$ and $B_p(s,a,s')$ are closed intervals of $[0,\rmaxbound]$ and $[0,1]$ respectively.
 Consider the ``augmented'' bounded-parameter MDP $\wt{\mathcal{M}}^{\downarrow}$ defined $\forall s, s' \in \mathcal{S}$ and $\forall a \in \mathcal{A}_s$ by:
 \begin{align*}
         B^{\downarrow}_r(s,a) = [r_{\min}, \max\{B_r(s,a)\}] ~~ \text{and} ~~ B^{\downarrow}_p(s,a,s') = B_p(s,a,s')
 \end{align*}

 Let $\wt{L}_d$ ($\wt{L}$) and $\wt{L}^{\downarrow}_d$ ($\wt{L}^{\downarrow}$) be the (optimal) Bellman operators of $\wt{\mathcal{M}}$ and $\wt{\mathcal{M}}^{\downarrow}$, respectively. Then 
 \begin{enumerate}
         \item $\forall v \in  \mathbb{R}^{S}$, $\wt{L}v = \wt{L}^{\downarrow}v$  and $\wt{\opT{}}v = \wt{T}^{\downarrow}_cv$.
         \item $\forall v \in \mathbb{R}^{S}$ s.t. $\SP{v} \leq c$,~~
         $
                 \wt{D}^{\downarrow}(c,v) = \{d \in \MR(\wt{\mathcal{M}}^{\downarrow}) ~|~ sp\{\wt{L}^{\downarrow}_d v\} \leq c\} \neq \emptyset
         $
 \end{enumerate}
\end{lemma}
\begin{proof}
For all states $s \in \mathcal{S}$ and actions $a \in \mathcal{A}_s$ we use the following notations
 \begin{align}
  \widetilde{r}(s,a) {:=} \max\{B_r(s,a)\} ~~ \text{and} ~~ \widetilde{p}(\cdot|s,a) {:=} \argmax_{p(s') \in B_p(s,a,s')} \sum_{s' \in \mathcal{S}} p(s')v(s')\\
  \uwidetilde{r}(s,a) {:=} \min\{B_r(s,a)\} ~~ \text{and} ~~ \uwidetilde{p}(\cdot|s,a) {:=} \argmin_{p(s') \in B_p(s,a,s')} \sum_{s' \in \mathcal{S}} p(s')v(s')
 \end{align}
and we define $\widetilde{r}_{\downarrow}(s,a)$, ~$\uwidetilde{r}_{\downarrow}(s,a)$ and $\widetilde{p}_{\downarrow}(\cdot|s,a)$, ~$\uwidetilde{p}_{\downarrow}(s,a)$ similarly with $B_r(s,a)$ and $B_p(s,a,s')$ replaced by $B_r^{\downarrow}(s,a)$ and $B_p^{\downarrow}(s,a,s')$.

Notice that the bounded-parameter MDP $\wt{\mathcal{M}}^{\downarrow}$ is just augmented from below, i.e., the maximum value of the reward is not altered: $\forall (s,a) \in \mathcal{S}\times\mathcal{A}$, $\widetilde{r}(s,a)=\widetilde{r}_{\downarrow}(s,a)$. Moreover, by definition: $\forall s,s' \in \mathcal{S},~ \forall a \in\mathcal{A}$,~ $\widetilde{p}(s'|s,a)=\widetilde{p}_{\downarrow}(s'|s,a)$.
As a consequence, $\forall v \in \mathbb{R}^{S}$ and $\forall s \in \mathcal{S}$, $\wt{L}v(s) = \max_{a\in \mathcal{A}_s} \left\{ \widetilde{r}(s,a) + \sum_{s' \in \mathcal{S}} \widetilde{p}(s'|s,a) v(s') \right\} = \wt{L}^{\downarrow}v(s)$.
Since $\wt{\opT{}}v = \Gamma_c \wt{L}v$ and $\wt{T}^{\downarrow}_c v = \Gamma_c \wt{L}^{\downarrow}v$ by definition, it follows that $\wt{\opT{}}v = \wt{T}^{\downarrow}_cv$

To prove the second statement, for any $v \in \Re^S$ we define $\delta_v\in \MR(\wt{\mathcal{M}}^{\downarrow})$ achieving the component-wise minimal value of $\wt{L}^{\downarrow}_{d} v$ for $d\in \MR(\wt{\mathcal{M}}^{\downarrow})$. Formally:
\[
        \delta_v := \argmin_{d\in \MR(\wt{\mathcal{M}}^{\downarrow})} \left\{ \wt{L}^{\downarrow}_d v
\right\} \implies \forall s \in \calS,~~ L_{\delta_v} v(s) = \min_{a\in\mathcal{A}_s} \left\{
                \underbrace{\uwidetilde{r}_{\downarrow}(s,a)}_{=r_{\min}} + \sum_{s' \in \mathcal{S}} \uwidetilde{p}_{\downarrow}(s'|s,a) v(s')
\right\}.
\]
As a consequence, if $v$ satisfies $\SP{v} \leq c$ then:
\begin{align*}
        \SP{\wt{L}^{\downarrow}_{\delta_v} v} &= \SP{r_{\min}\bm{e} + \uwidetilde{P}^{\downarrow}_{\delta_v} v } = \SP{\uwidetilde{P}^{\downarrow}_{\delta_v} v } \leq \SP{v} \leq c
\end{align*}
where we used \citep[Poposition 6.6.1]{puterman1994markov} applied to the stochastic matrix $\uwidetilde{P}^{\downarrow}_{\delta_v}$. Then, ${\delta_v}\in \wt{D}^{\downarrow}(c,v) \neq \emptyset$.
\end{proof}

We can finally ``merge'' Lem.~\ref{L:pperturbed_bmdp} and~\ref{L:r_augmented_bmdp} and provide properties for operator $\opT{}^{\eta,\downarrow}$ associated to the augmented and perturbed bounded-parameter MDP $\wt{\mathcal{M}}_\eta^{\downarrow}$.

\begin{theorem}[Equivalent to Thm.~\ref{th:augmented.perturbed.mdp}]\label{T:improved_op}
        Let $\wt{\mathcal{M}}$ be a bounded-parameter MDP and $\wt{\mathcal{M}}_{\eta}^{\downarrow}$ be the associated both ``augmented'' and ``perturbed'' bounded-parameter MDP (see Lem.~\ref{L:pperturbed_bmdp} and~\ref{L:r_augmented_bmdp}).
        Then,
        \begin{enumerate}
                \item Lem.~\ref{lem:optimality.equation.T} applies to $\wt{T}_c^{\eta,\downarrow}$. Denote by $(g^+,h^+)$ a solution to equation~\eqref{E:optimality_eq_average_T} for $\wt{T}_c^{\eta,\downarrow}$.
                \item $\wt{\mathcal{M}}_{\eta}^{\downarrow}$ is unichain and Thm.~\ref{thm:opT.optimality} applies to  $\wt{T}_c^{\eta,\downarrow}$. Denote by $\pi^+ \in \SR(\wt{\mathcal{M}}_{\eta}^{\downarrow})$ any policy achieving: \[\wt{L}^{\eta,\downarrow}_{\pi^+} h^+ = \wt{T}_c^{\eta,\downarrow} h^+ = \wt{N}_{c}^{\eta,\downarrow} h^+ = h^+ + g^+ e.\]
                \item $\forall \mu \in \PiC(\wt{\mathcal{M}}), \quad g^+ = g^{\pi^+}_{\wt{\mathcal{M}}_{\eta}^{\downarrow}} = g_c^*(\wt{\mathcal{M}}_{\eta}^{\downarrow}) \geq g^{\mu^\infty}_{\wt{\mathcal{M}}} - \eta c.$
        \end{enumerate}
\end{theorem}
\begin{proof}
        Lem.~\ref{L:pperturbed_bmdp} states that $\wt{L}_{\eta}$ is a $\gamma$-span contraction ($\gamma < 1$). 
        Since for any $v \in \Re^S$,  $\wt{L}_{\eta}^{\downarrow}v = \wt{L}_{\eta}v$ (property \emph{1.} of Lem.~\ref{L:r_augmented_bmdp}), $\wt{L}_{\eta}^{\downarrow}$ is also a $\gamma$-span contraction.
        As a consequence, Asm.~\ref{asm:L.contraction} holds and so Lem.~\ref{lem:optimality.equation.T} applies to $\wt{T}_c^{\eta,\downarrow}$ thus proving the first statement.

	To prove the second statement, notice that $\wt{L}_{\eta}^{\downarrow}$ satisfies Asm.~\ref{asm:feasibility} due to property \emph{2.} of Lem.~\ref{L:r_augmented_bmdp} and Lem.~\ref{L:TRpolicyexistence}. Moreover, $\wt{\mathcal{M}}_{\eta}^{\downarrow}$ is unichain by Lem.~\ref{L:pperturbed_bmdp} meaning that all policies (and in particular $\pi^+$) are unichain. Therefore, Thm.~\ref{thm:opT.optimality} applies to  $\wt{T}_c^{\eta,\downarrow}$.
        
        Finally, we prove the third statement. We first prove by induction that the two sequences $(v_n)_{n\in \mathbb{N}}$ and $(\widehat{v}_n)_{n\in \mathbb{N}}$ defined by $v_0 = \widehat{v}_0$ such that $\SP{v_0}\leq c$ and for all $n \in \mathbb{N}$, $v_{n+1} = \wt{T}_c^{\downarrow} v_n$ and $\widehat{v}_{n+1} = \wt{T}_c^{\eta,\downarrow} \widehat{v}_n$, satisfy $\| v_n - \widehat{v}_n \|_{\infty} \leq n \eta c $:
        \begin{enumerate}
         \item The result trivially holds for $n = 0$.
         \item Assume the result holds for $n \in \mathbb{N}$. Let's show that it is also true for $n+1$:
         \begin{align*}
          \|v_{n+1} - \widehat{v}_{n+1} \|_\infty &= \big\|\wt{T}_c^{\downarrow}v_{n} - \wt{T}_c^{\eta,\downarrow}\widehat{v}_{n} \big\|_\infty = \big\|\proj{}\wt{L}^{\downarrow}v_{n} - \proj{}\wt{L}^{\downarrow}_\eta \widehat{v}_{n} \big\|_\infty \\
          &\leq \big\|\wt{L}^{\downarrow}v_{n} - \wt{L}_\eta^{\downarrow} \widehat{v}_{n} \big\|_\infty \\
          &\leq \underbrace{\big\|\wt{L}^{\downarrow}v_{n} - \wt{L}_\eta^{\downarrow} v_{n} \big\|_\infty}_{\leq \SP{v_n} \eta \leq \eta c} + \underbrace{\big\|\wt{L}_\eta^{\downarrow} v_{n} - \wt{L}_\eta^{\downarrow} \widehat{v}_{n} \big\|_\infty}_{\leq \| v_{n} -  \widehat{v}_{n} \|_\infty \leq n \eta c }\\
          &\leq (n+1) \eta c 
         \end{align*}
        \end{enumerate}
        The first inequality comes from the fact that $\proj{}$ is non-expansive (property \emph{(d)} of Lem.~\ref{L:proj_properties}). The second inequality is just the triangle inequality. The last inequality follows from Lem.~\ref{L:pperturbed_bmdp}, the fact that $\SP{v_n} \leq c$ by definition, the fact that $\wt{L}_\eta$ is non-expansive and the induction assumption. Let $\mu\in \PiC(\wt{\mathcal{M}})$ and for simplicity denote by $h^\mu$ (respectively $g^\mu$) the bias $h^{\mu^\infty}_{\wt{\mathcal{M}}}$ associated to policy $\mu^\infty$ in $\wt{\mathcal{M}}$ (respectively the gain $g^{\mu^\infty}_{\wt{\mathcal{M}}}$). Since $\SP{h^\mu }\leq c$ by definition of $\PiC(\wt{\mathcal{M}})$, we can apply the result we just proved with $v_0 = \wh{v}_0 = h^\mu$:
        \begin{align*}
         \forall n \in \mathbb{N},~~(\wt{T}_c^{\eta,\downarrow})^n h^\mu \geq (\wt{T}_c^{\downarrow})^n h^\mu - n  \eta c e
        \end{align*}
        where $e = (1, \dots,1)^\intercal$ is the vector of all 1's and $(\wt{T}_c^{\eta,\downarrow})^n$ denotes $n$ consecutive applications of operator $\wt{T}_c^{\eta,\downarrow}$. By property \emph{1} of Lem.~\ref{L:r_augmented_bmdp} we have that $(\wt{T}_c^{\downarrow})^n h^\mu = (\wt{T}_c)^n h^\mu$ implying that:
        \begin{align}\label{eq:temp3.eq}
         \forall n \in \mathbb{N},~~(\wt{T}_c^{\eta,\downarrow})^n h^\mu \geq (\wt{T}_c)^n h^\mu - n  \eta c e
        \end{align}
        Now using the Bellman evaluation equation of $\mu$ we have:
        \begin{align*}
         \wt{L}_\mu h^\mu = h^\mu +g^\mu e \implies \SP{\wt{L}_\mu h^\mu} = \SP{h^\mu} \leq c \implies \mu \in \wt{D}(c,h^\mu)
        \end{align*}
        Therefore, by Lem.~\ref{L:TRpolicyexistence} we have that $\wt{T}_c h^\mu \geq \wt{L}_\mu h^\mu = h^\mu +g^\mu e$ and using the monotonicity of $\wt{T}_c$ (property \emph{(a)} of Lem.~\ref{L:T_properties}) we obtain by induction that:
        \begin{align}\label{eq:temp4.eq}
         \forall n \in \mathbb{N},~~(\wt{T}_c)^n h^\mu \geq h^\mu  +n g^\mu e
        \end{align}
        Combining~\eqref{eq:temp3.eq} and~\eqref{eq:temp4.eq} we have that
        \begin{align}\label{eq:temp5.eq}
          \forall n \geq 1,~~ \frac{1}{n}\sum_{k=1}^{n}\left[(\wt{T}_c^{\eta,\downarrow})^k h^\mu - (\wt{T}_c^{\eta,\downarrow})^{k-1}h^\mu\right] = \frac{1}{n}\left[(\wt{T}_c^{\eta,\downarrow})^n h^\mu - h^\mu\right] \geq (g^\mu - \eta c) e
        \end{align}
        The term on the left-hand side of~\eqref{eq:temp4.eq} is the \emph{Cesaro mean} of the sequence $\left((\wt{T}_c^{\eta,\downarrow})^n h^\mu - (\wt{T}_c^{\eta,\downarrow})^{n-1}h^\mu\right)_{n \in \mathbb{N}}$. By property \emph{2.} of Lem.~\ref{lem:optimality.equation.T} we know that this sequence converges to $g^+$ and thus by Cesaro theorem we know that the Cesaro mean has the same limit. Therefore, taking the limit on both sides of the inequality in~\eqref{eq:temp4.eq} yields: 
        \begin{align*}
          g^+ \geq g^\mu - \eta c
        \end{align*}
        which concludes the proof.
\end{proof}

\section{Regret Analysis of \scal (Proof of Thm.~\ref{thm:regret.scal})}\label{app:regret.reg.ucrl}

We follow the proof structure in~\citep{Jaksch10} and use similar notations. The main differences with \citet{Jaksch10}'s regret proof are the following:
\begin{enumerate}[noitemsep]
 \item We use empirical Bernstein confidence bounds for both the rewards and the transition probabilities and not Hoeffding bounds.
 \item The actual confidence bounds used by extended value iteration needs to be adapted in order to insure both convergence of the algorithm and feasibility of the policy (the MDP is ``modfied'', see Def~\ref{def:augmented.perturbed.mdp}).
 \item The policy returned by extended value iteration may be stochastic.
\end{enumerate}

\subsection{Splitting into episodes}

The regret after $T$ time steps is defined as:
\begin{align*}
 \Delta(\scal,T) = Tg^* - \sum_{t=1}^T r_t(s_t,a_t)
\end{align*}
Define the filtration $\mathcal{F}_t = \sigma(s_1, a_1,r_1,\dots,s_{t+1})$ and the stochastic process $X_t = r_t(s_t,a_t) - \sum_{a \in \mathcal{A}_{s_t}} r(s_t,a)\widetilde{\pi}_{k_t}(s_t,a)$ where $k_t$ is the episode at time $t$ and $\widetilde{\pi}_{k_t}$ is the stochastic policy being executed at time $t$. Note that $\widetilde{\pi}_{k_t}$ is a random variable that is $\mathcal{F}_{t-1}$-measurable. Moreover, $(X_t,\mathcal{F}_t)_{t\geq 0}$ is a Martingale Difference Sequence (MDS) since $|X_t| \leq \rmaxbound$ and $\mathbb{E}[X_t | \mathcal{F}_{t-1}] = 0$. Using Azuma's inequality (see for example \citet[Lemma 10]{Jaksch10}):
\begin{align*}
 \mathbb{P}\left(\sum_{t=1}^T r_t(s_t,a_t) \leq \sum_{t=1}^T \sum_{a \in\mathcal{A}_{s_t}} r(s_t,a)\widetilde{\pi}_{k_t}(s_t,a) - \rmaxbound\sqrt{\frac{5}{2}T\ln \left(\frac{11T}{\delta}\right)}\right) \leq \left(\frac{\delta}{11T}\right)^{5/4} < \frac{\delta}{20T^{5/4}}
\end{align*}
For any episode $k$, we denote by $t_k$ the starting time of that episode. Let's also denote by $\nu_k(s)$ (resp. $\nu_k(s,a)$) the total number of visits in state $s$ (resp. state-action pair $(s,a)$) during episode $k$ (\ie before time $t_{k+1}$, $t_{k+1}$ \emph{not} included, and after time $t_k$, $t_k$ included):
\begin{align*}
        \nu_k(s,a) &:= \big| \left\lbrace t_{k} \leq \tau < t_{k+1}: (s_\tau, a_\tau) =(s,a)\right\rbrace \big|\\
        \nu_k(s) &:= \big| \left\lbrace t_{k} \leq \tau < t_{k+1}: s_\tau =s \right\rbrace \big| = \sum_{a \in \mathcal{A}_s} \nu_k(s,a)
\end{align*}
Defining $\Delta_k=\sum_{s \in \mathcal{S}} \nu_k(s) \left(g^* - \sum_{a \in \mathcal{A}_{s_t}} r(s,a)\widetilde{\pi}_{k}(s,a)\right)$, it holds with probability at least $1-\frac{\delta}{20T^{5/4}}$ that:
\begin{align}\label{eqn:splitting}
\begin{split}
 \Delta(\scal,T) &\leq Tg^* - \sum_{t=1}^T \sum_{a \in\mathcal{A}_{s_t}} r(s_t,a)\widetilde{\pi}_{k_t}(s_t,a) + \rmaxbound\sqrt{\frac{5}{2}T\ln \left(\frac{11T}{\delta}\right)}\\
 &= \sum_{k=1}^m \sum_{s \in \calS} \nu_k(s)g^* - \sum_{k=1}^m \sum_{s \in \calS} \nu_k(s) \sum_{a \in\mathcal{A}_{s}} r(s,a)\widetilde{\pi}_{k}(s,a)+ \rmaxbound\sqrt{\frac{5}{2}T\ln \left(\frac{11T}{\delta}\right)}\\
 &= \sum_{k=1}^m \Delta_k + \rmaxbound\sqrt{\frac{5}{2}T\ln \left(\frac{11T}{\delta}\right)}
\end{split}
\end{align}

\subsection{Dealing with failing confidence regions}

We start by bounding the term $\sum_{k=1}^m \Delta_k \mathbbm{1}_{M \not\in \wt{\mathcal{M}}_k}$ corresponding to the regret suffered in episodes where the true MDP $M$ is not contained in the original set of plausible MDPs ${\mathcal{M}}_k$ (and not the modified set $\mathcal{M}_k^\ddagger$). We use exactly the same proof as in~\citep{Jaksch10}.
\begin{align*}
 \sum_{k=1}^m \Delta_k \mathbbm{1}_{M \not\in {\mathcal{M}}_k} 
 &\leq \rmaxbound\sum_{k=1}^m \sum_{s}\nu_k(s)\mathbbm{1}_{M \not\in {\mathcal{M}}_k}  
 = \rmaxbound\sum_{k=1}^m \sum_{s,a} \nu_k(s,a) \mathbbm{1}_{M \not\in {\mathcal{M}}_k}\\
 &\leq {\rmaxbound}\sqrt{T} + {\rmaxbound} \sum_{t=\lfloor T^{1/4}\rfloor + 1}^{T}t \mathbbm{1}_{\lbrace \exists k\geq 1:t=t_k ~ \text{and} ~ M \not\in {\mathcal{M}}_k\rbrace}
\end{align*}
Provided~$\mathbb{P}(M \not\in {\mathcal{M}}_k) \leq \frac{\delta}{15t_k^6}$ for all $k\geq 1$ (see Thm.~\ref{thm:mdp_confidence} below), we conclude as in~\citet{Jaksch10} that with probability at least $1-\frac{\delta}{20T^{5/4}}$: 
\begin{align}\label{eqn:failed_episodes}
 \sum_{k=1}^m \Delta_k \mathbbm{1}_{M \not\in {\mathcal{M}}_k} \leq {\rmaxbound} \sqrt{T}
\end{align}

We recall the upper confidence bounds used for the reward function and the transition kernel in the algorithm:
\begin{align}
        \lvert \wt{r}(s,a) - \wh{r}_k(s,a) \rvert \leq  \overbracket[0.5pt]{ \sqrt{\frac{14 \alpha_r \wh{\sigma}_{r,k}^2(s,a) \ln(2 SA t_k/\delta) }{\max \lbrace1,N_k(s,a)\rbrace}} + \frac{49\alpha_r \rmaxbound \ln(2 SA t_k/\delta)}{3\max \lbrace1,N_k(s,a)-1\rbrace}}^{\beta_{r,k}^{sa}}\label{eqn:cb_reward}\\
        \lvert \wt{p}(s'|s,a) - \wh{p}_k(s'|s,a) \rvert \leq \underbracket[0.5pt]{\sqrt{\frac{14 \alpha_p \wh{\sigma}_{p,k}^2(s'|s,a) \ln(2 SA t_k/\delta)}{\max \lbrace1,N_k(s,a)\rbrace}} +\frac{49\alpha_p \ln(2 SA t_k/\delta)}{3 \max \lbrace1,N_k(s,a)-1\rbrace}}_{\beta_{p,k}^{sas'}}\label{eqn:cb_proba}
\end{align}
where for the theoretical analysis we set\footnote{$\alpha_r$ and $\alpha_p$ are coefficients used to shrink the confidence intervals in the implementation in order to speed up the learning in practice. However, to insure that $M \in {\mathcal{M}}_k$ holds with high probability they should both be set equal to $1$.} $\alpha_r = \alpha_p =1$ and the \emph{unbiased} estimates of the variances are
\begin{align}\label{eqn:def_empirical_variance}
        &\wh{\sigma}_{r,k}^2(s,a) = \frac{\sum_{t=1}^{t_k -1}\Big(r_t(s_t,a_t) - \wh{r}_k(s,a) \Big)^2\mathbbm{1}_{\left\lbrace(s_t,a_t) = (s,a)\right\rbrace}}{N_k(s,a) - 1}\\
        \text{and} ~~ &\wh{\sigma}_{p,k}^2(s'|s,a) = \wh{p}_k(s'|s,a)\left( 1 - \wh{p}_k(s'|s,a)\right)
\end{align}
Note that although the definition of the sample variance of the reward $r$ involves a sum, it can be computed dynamically using the following well-known recurrence relation:
\begin{align*}
 &\wh{\sigma}_{n+1}^2 = \frac{(n-1)}{n}\wh{\sigma}_{n}^2 + \frac{1}{n+1}(r_{n+1} - \wh{r}_n)^2 
\end{align*}
where $n$ denotes the number of samples ($N_k(s,a)$ in our case), $r_{n+1}$ is the $(n+1)$-th sample observed, and $\wh{r}_n = 1/n \sum_{i=1}^n r_i$ and $\wh{\sigma}_{n}^2  = 1/(n-1)\sum_{i=1}^{n}(r_i - \wh{r}_{n})^2$ are respectively the empirical average and sample variance obtained with the first $n$ samples.
The previous formula is subject to numerical unstability because the second term becomes negligible compared to the first term as $n$ grows. A better approach for computing the variance is to exploit the following iterative scheme known as Welford's method~\citep[][p. 232]{knuth1997artvol2}:
\begin{align*}
        \wh{r}_{n} &= \wh{r}_{n-1} + \frac{r_n - \wh{r}_{n-1}}{n}\\
        S_n &= S_{n-1} + \left(r_n - \wh{r}_{n-1} \right) \left(r_n - \wh{r}_n \right)\\
        \wh{\sigma}_n^2 &= \frac{S_n}{n-1} 
\end{align*}
for $n\geq 2$, with $\wh{r}_1 = r_1$ and $S_1 = 0$. This approach is less prone to numerical instability and its accuracy is comparable to the one of two-pass methods.
\begin{theorem}\label{thm:mdp_confidence}
For any $k \geq 1$, the probability that the true MDP $M$ is not contained in the set of plausible MDPs ${\mathcal{M}}_k$ at time $t_k$ (as given by the confidence intervals in \eqref{eqn:cb_reward} and \eqref{eqn:cb_proba}) is at most $\frac{\delta}{15t_k^6}$, that is $\mathbb{P}(M \not\in {\mathcal{M}}_k) \leq \frac{\delta}{15t_k^6}$.
\end{theorem}
\begin{proof} First note the following equality
\begin{align*}
 \mathbb{P}\left( M \not\in {\mathcal{M}}_k \right) = \mathbb{P}\left( \bigcup_{s,a, s'} \big\lbrace \wt{r}_k(s,a) \not\in B_r^k(s,a) \big\rbrace \cup  \big\lbrace \wt{p}_k(s'|s,a) \not\in B_p^k(s,a,s') \big\rbrace \right)
\end{align*}
Using Theorem 4 in~\citep{Maurer2009empirical}\footnote{Also known as ``Empirical Bernstein'' concentration inequality.} we have that given an episode $k \geq 1$, a state action pair $(s,a) \in \mathcal{S} \times \mathcal{A}$, a number of visits $N_k(s,a)$ in $(s,a)$ before time $t_k$ and (similarly to~\citep{Jaksch10}):
 \begin{align*}
         &\epsilon_{r,k} = \wh{\sigma}_{r,k}  \sqrt{\frac{2 \ln(120 SA t_k^7/\delta) }{\max \lbrace1,N_k(s,a)\rbrace}} + \frac{7 \rmaxbound \ln(120 SA t_k^7/\delta)}{3\max \lbrace1,N_k(s,a)-1\rbrace} \leq \beta_{r,k}^{sa}\\ 
         \intertext{then}
         &\mathbb{P}\left( | r(s,a) - \wh{r}_k(s,a) | \geq \beta_{r,k}^{sa} \right) \leq \mathbb{P}\left( |r(s,a) - \wh{r}_k(s,a) | \geq \epsilon_{r,k} \right) \leq \frac{\delta}{60t_k^7SA}
 \end{align*}
 Similarly, $\mathbb{P}\left( | p(s'|s,a) - \wh{p}_k(s'|s,a) | \geq \beta_{p,k}^{sas'} \right) \leq \frac{\delta}{20t_k^7SA}$.\\[.2cm]
 Note that when $N_k(s,a) = 0$ (\ie there hasn't been any observation), the bound holds trivially with probability $1$ both for rewards and transition probabilities. Recall the definitions of $B_r^k(s,a)$ and $B_p^k(s,a,s')$
\begin{align*}
 &B_r^k(s,a) = [\wh{r}_k(s,a) - \beta_{r,k}^{sa}, \wh{r}_k(s,a) + \beta_{r,k}^{sa}] \cap [0, \rmaxbound]\\
 \text{and}~~ &B_p^k(s,a,s') = [\wh{p}_k(s'|s,a) - \beta_{p,k}^{sas'},\wh{p}_k(s'|s,a) + \beta_{p,k}^{sas'}] \cap [0,1]                                                                                                                                                                                              
\end{align*}
Therefore it is clear that
\begin{align*}
 &\mathbb{P}\Big( \wt{r}_k(s,a) \notin B_r^k(s,a) \Big) \leq \frac{\delta}{60t_k^7SA} \\
 \text{and}~~ &\mathbb{P}\Big( \wt{p}_k(s'|s,a) \notin B_p^k(s,a,s') \Big) \leq \frac{\delta}{20t_k^7S^2 A}                                                                                                                                                                                          
\end{align*}
By taking a union bound over all state-action pairs $(s,a)$ and all possible values for $N_k(s,a) = 0, \dots , t_k-1$ we obtain
\begin{align*}
 \mathbb{P}(M \not\in {\mathcal{M}}_k) \leq \sum_{s,a} \sum_{N_k(s,a) = 1}^{t_k -1} \left( \frac{\delta}{60 t_k^7SA} + \sum_{s'}\frac{\delta}{20 t_k^7S^2 A} \right) \leq \frac{\delta}{15 t_k^6}
\end{align*}
\end{proof}

\subsection{Episodes whith $M \in {\mathcal{M}}_k$}

Now we assume that $M\in {\mathcal{M}}_k$ and we first bound $\Delta_k$. Note that we do \emph{not} assume that $M$ belongs to the modified set of MDPs ${\mathcal{M}}_k^\ddagger$. Denote by $\wt{g}_k:= 1/2(\max \{v_{n+1} - v_n\} + \min \{v_{n+1} - v_n\})$ where $v_n$ is the value function returned by \regopt{($0,\overline{s},\gamma_k,\varepsilon_k$)} (see Sec.~\ref{sec:regret}). \regopt recursively applies operator $\wt{T}^{\ddagger}_c := \wt{T}_c^{\eta_k,\downarrow}$ (see App.~\ref{app:augmented.perturbed.mdp}) with a perturbation $\eta_k = 1/c\sqrt{t_k}$ until the stopping condition is reached. Moreover, the stopping condition is such that when the algorithm stops, the accuracy of the gain $\wt{g}_k$ with respect to $g_c^*(\wt{\mathcal{M}}_k^\ddagger)$ is $\varepsilon_k = 1/\sqrt{t_k}$ (see convergence guarantee \emph{2)} of Thm.~\ref{thm:opT.optimality2}). Therefore, due to the fact that $M\in {\mathcal{M}}_k$ and since by assumption there exists an optimal policy $\pi^*$ such that $\SP{h^{\pi^*}(M)} \leq c$ we can apply Thm.~\ref{th:augmented.perturbed.mdp} which implies that
\begin{align*}
\wt{g}_k \overbrace{\geq}^{\text{Thm.~\ref{thm:opT.optimality2}}} g_c^*(\wt{\mathcal{M}}_k^\ddagger) - \underbrace{\varepsilon_k}_{=\rmaxbound /\sqrt{t_k}} \overbrace{\geq}^{\text{Thm.~\ref{th:augmented.perturbed.mdp}}} g_c^*(\wt{\mathcal{M}}_k) -\underbrace{\frac{c\cdot\rmaxbound }{c \cdot t_k}}_{=\rmaxbound /t_k} - \frac{\rmaxbound }{\sqrt{t_k}} \overbrace{\geq}^{M\in {\mathcal{M}}_k} g^* - \frac{\rmaxbound }{t_k} -\frac{\rmaxbound }{\sqrt{t_k}} \geq g^* - \frac{2\rmaxbound }{\sqrt{t_k}}
\end{align*}
implying:
\begin{align*}
\Delta_k \leq \sum_{s \in \mathcal{S}} \nu_k(s) \left(\wt{g}_k - \sum_{a \in\mathcal{A}_{s}} r(s,a)\widetilde{\pi}_{k}(s,a) \right) + 2\rmaxbound \sum_{s \in \mathcal{S}} \frac{ \nu_k(s)}{\sqrt{t_k}}
\end{align*}

\subsection{Extended Value Iteration}

A direct consequence of Thm.~\ref{thm:opT.optimality2} is that when the convergence criterion holds at iteration $n$ then:
\begin{align*}
 \forall s \in \mathcal{S}, ~~|v_{n+1}(s) - v_n(s) - \wt{g}_k | \leq \frac{\rmaxbound }{\sqrt{t_k}}
\end{align*}
As is shown in the proof of Thm.~\ref{T:improved_op}, operator $\wt{T}_c^{\eta_k,\downarrow}$ is feasible at $v_n$ for every $n \in \mathbb{N}$ (Asm.~\ref{asm:feasibility} holds) and we can expand ${v}_{n+1}$ as
\begin{align*}
 \forall s \in \mathcal{S}, ~~ v_{n+1}(s) = \sum_{a \in\mathcal{A}_{s}}\wt{r}_k(s,a)\widetilde{\pi}_{k}(s,a) + \sum_{s'\in \mathcal{S}}\sum_{a \in\mathcal{A}_{s}}\wt{p}_k(s'|s,a)\widetilde{\pi}_{k}(s,a) v_n(s')
\end{align*}
implying:
\begin{align}\label{E:gerror.app}
 \forall s \in \mathcal{S}, ~~ \left| \left(\wt{g}_k - \sum_{a \in\mathcal{A}_{s}}\wt{r}_k(s,a)\widetilde{\pi}_{k}(s,a) \right) - \left( \sum_{s'\in \mathcal{S}}\sum_{a \in\mathcal{A}_{s}}\wt{p}_k(s'|s,a)\widetilde{\pi}_{k}(s,a) v_n(s') - v_n(s) \right)\right| \leq \frac{\rmaxbound }{\sqrt{t_k}}
\end{align}
Setting $\bm{\nu}_k:= \left(\nu_k(s)\right)_{s\in \mathcal{S}}$ the row vector of visit counts for each state and $\bm{\wt{P}}_k:= \left(\sum_{a \in\mathcal{A}_{s}}\wt{p}_k(s'|s,a)\widetilde{\pi}_{k}(s,a)\right)_{s,s'\in \mathcal{S}}$ the ``optimistic'' transition matrix of $\wt{\pi}_k$ we obtain (using~\eqref{E:gerror.app}):
\begin{align*}
 \Delta_k &\leq \sum_{s \in \mathcal{S}} \nu_k(s) \left(\wt{g}_k - \sum_{a \in\mathcal{A}_{s}} r(s,a)\widetilde{\pi}_{k}(s,a) \right) + 2\rmaxbound \sum_{s \in \mathcal{S}} \frac{ \nu_k(s)}{\sqrt{t_k}}\\
 &= \sum_{s \in \mathcal{S}} \nu_k(s) \left(\wt{g}_k - \sum_{a \in\mathcal{A}_{s}}\wt{r}_k(s,a)\widetilde{\pi}_{k}(s,a) \right) + \sum_{s,a} \nu_k(s)\widetilde{\pi}_{k}(s,a) \Big(\wt{r}_k(s,a) - r(s,a) \Big) + 2\rmaxbound \sum_{s \in \mathcal{S}} \frac{ \nu_k(s)}{\sqrt{t_k}}\\
 &\leq \bm{\nu}_k(\bm{\wt{P}}_k-I)\bm{v}_{n} + \sum_{s,a} \nu_k(s)\widetilde{\pi}_{k}(s,a) \Big(\wt{r}_k(s,a) - r(s,a) \Big) + 3\rmaxbound \sum_{s \in \mathcal{S}} \frac{ \nu_k(s)}{\sqrt{t_k}}
\end{align*}
Since the rows of $\bm{\wt{P}}_k$ sum to 1 (\ie $\bm{\wt{P}}_k\bm{e} =\bm{e}$), we can replace $\bm{v}_n$ by $\bm{w}_k$ where we set 
\begin{align*}
 \bm{w}_k := \bm{v}_n- \frac{\max_s v_n(s) + \min_s v_n(s)}{2}\bm{e}
\end{align*}
In conclusion,
\begin{align}\label{eqn:delta_k}
 \Delta_k \leq \bm{\nu}_k(\bm{\wt{P}}_k-I)\bm{w}_k + \sum_{s,a} \nu_k(s)\widetilde{\pi}_{k}(s,a) \Big(\wt{r}_k(s,a) - r(s,a) \Big) + 3\rmaxbound \sum_{s \in \mathcal{S}} \frac{ \nu_k(s)}{\sqrt{t_k}}
\end{align}

By definition of operator $\wt{T}_ c^{\eta_k,\downarrow}$, we have that $ \SP{w_k} = \SP{v_n} = \SP{\wt{T}_ c^{\eta_k,\downarrow} v_{n-1}}\leq c$ and since $w_k$ is obtained by ``recentering'' $v_n$ around $0$ we have that $\|w_k\|_\infty = \SP{w_k}/2 \leq c/2$.



\subsection{Bounding the reward}

To guarantee the \emph{feasibility} of operator $\wt{\opT{}}$ we had to \emph{augment} the MDP (see Lem.~\ref{L:r_augmented_bmdp}), \ie allow the rewards $\wt{r}_k$ to be as small as $0$ even when $\wh{r}_k - \beta_{r,k} >0$. Nevertheless, the upper-bound of the reward was not modified (only the lower-bound) and so $\wt{r}_k \leq \min \left\lbrace \rmaxbound, \wh{r}_k + \beta_{r,k}\right\rbrace \leq \wh{r}_k + \min \left\lbrace \rmaxbound, \beta_{r,k}\right\rbrace $. Therefore:
\begin{align*}
 \sum_{s,a} \nu_k(s)\widetilde{\pi}_{k}(s,a) \Big(\wt{r}_k(s,a) - r(s,a) \Big) \leq \sum_{s,a} \nu_k(s)\widetilde{\pi}_{k}(s,a) \min \left\lbrace \rmaxbound, \beta_{r,k}^{sa} \right\rbrace + \sum_{s,a} \nu_k(s)\widetilde{\pi}_{k}(s,a)  \Big(\wh{r}_k(s,a) - r(s,a) \Big)
\end{align*}
Moreover, since we assumed that $M\in {\mathcal{M}}_k$ the bound $\wh{r}_k \leq \min \left\lbrace \rmaxbound, r+\beta_{r,k} \right\rbrace \leq r + \min \left\lbrace \rmaxbound, \beta_{r,k}\right\rbrace$ holds and thus
\begin{align*}
 \sum_{s,a} \nu_k(s)\widetilde{\pi}_{k}(s,a) \Big(\wt{r}_k(s,a) - r(s,a) \Big) \leq 2 \sum_{s,a} \nu_k(s)\widetilde{\pi}_{k}(s,a)\min \left\lbrace \rmaxbound, \beta_{r,k}^{sa} \right\rbrace
\end{align*}
Note that when summing over all episodes $k\geq 1$, we can rewrite
\begin{align*}
 \sum_{k=1}^m \sum_{s,a} \nu_k(s)\widetilde{\pi}_{k}(s,a)\min \left\lbrace \rmaxbound, \beta_{r,k}^{sa} \right\rbrace = \sum_{t=1}^T \sum_{a \in \mathcal{A}_{s_t}} \widetilde{\pi}_{k_t}(s_t,a) \min \left\lbrace \rmaxbound, \beta_{r,k_t}^{s_ta} \right\rbrace 
\end{align*}
Define the filtration $\mathcal{F}_t = \sigma(s_1, a_1,r_1,\dots,s_{t+1})$ and the stochastic process
\begin{align*}
 X_t = \sum_{a \in \mathcal{A}_{s_t}}\widetilde{\pi}_{k_t}(s_t,a) \min \left\lbrace \rmaxbound, \beta_{r,k_t}^{s_ta} \right\rbrace  - \min \left\lbrace \rmaxbound, \beta_{r,k_t}^{s_ta_t} \right\rbrace 
\end{align*}
Note that $\widetilde{\pi}_{k_t}$ is a random variable that is $\mathcal{F}_{t-1}$-measurable. Moreover, $(X_t,\mathcal{F}_t)_{t\geq 0}$ is an MDS since $|X_t| \leq \rmaxbound$ and $\mathbb{E}[X_t | \mathcal{F}_{t-1}] = 0$. Using Azuma's inequality:
\begin{align*}
        \mathbb{P}\left(\sum_{t=1}^{T} \min \left\lbrace \rmaxbound, \beta_{r,k_t}^{s_ta_t} \right\rbrace  \leq \sum_{t=1}^{T} \sum_{a\in \mathcal{A}_{s_t}} \wt{\pi}_{k_t}(s_t,a) \min \left\lbrace \rmaxbound, \beta_{r,k_t}^{s_ta} \right\rbrace - \rmaxbound\sqrt{\frac{5}{2}T\ln \left(\frac{11T}{\delta}\right)}\right) \leq \left(\frac{\delta}{11T}\right)^{5/4} < \frac{\delta}{20T^{5/4}}
\end{align*}
or in other words, with probability at least  $1-\frac{\delta}{20T^{5/4}}$:
\begin{align*}
 \sum_{k=1}^{m}\sum_{s,a} \nu_k(s)\wt{\pi}_k(s,a) \min \left\lbrace \rmaxbound, \beta_{r,k}^{sa} \right\rbrace \leq \sum_{k=1}^{m}\sum_{s,a} \nu_k(s,a) \underbrace{\min \left\lbrace \rmaxbound, \beta_{r,k}^{sa} \right\rbrace}_{\leq \beta_{r,k}^{sa}} + \rmaxbound\sqrt{\frac{5}{2}T\ln \left(\frac{11T}{\delta}\right)}
\end{align*}
In conclusion, with probability at least  $1-\frac{\delta}{20T^{5/4}}$:
\begin{align}\label{eqn:bound_reward}
 \sum_{s,a} \nu_k(s)\widetilde{\pi}_{k}(s,a) \Big(\wt{r}_k(s,a) - r(s,a) \Big) \leq 2\sum_{k=1}^{m}\sum_{s,a} \nu_k(s,a) \beta_{r,k}^{sa} + 2\rmaxbound\sqrt{\frac{5}{2}T\ln \left(\frac{11T}{\delta}\right)}
\end{align}

\subsection{Bounding the transition matrix}

We denote by $\bm{P}_k:=\left(\sum_{a}p(s'|s,a)\widetilde{\pi}_{k}(s,a)\right)_{s,s'\in \mathcal{S}}$ the true transition matrix and $\bm{\wh{P}}_k:=\left(\sum_{a}\wh{p}_k(s'|s,a)\widetilde{\pi}_{k}(s,a)\right)_{s, s'\in \mathcal{S}}$ the estimated transition matrix. We do the following decomposition
\begin{align*}
        \bm{\nu}_k(\bm{\wt{P}}_k-I)\bm{w}_k = \underbrace{
                \bm{\nu}_k(\bm{\wt{P}}_k-\bm{\wh{P}}_k)\bm{w}_k + \bm{\nu}_k(\bm{\wh{P}}_k-\bm{P}_k)\bm{w}_k
        }_{\bm{\nu}_k(\bm{\wt{P}}_k-\bm{P}_k)\bm{w}_k}
        + \bm{\nu}_k(\bm{P}_k-I)\bm{w}_k
\end{align*}
Since we assumed that $M\in {\mathcal{M}}_k$ the difference $\bm{\wh{P}}_k-\bm{P}_k$ concentrates. Moreover, the \emph{perturbation} $\eta_k>0$ applied by operator $\wt{T}_ c^{\eta_k,\downarrow}$ to guarantee \emph{convergence} (see Lem.~\ref{L:pperturbed_bmdp}) is only \emph{shrinking} (and not expanding) the confidence intervals $B_p^k(s,a,s')$ and therefore by construction $\wt{p}_k(s'|s,a) \in B_p^k(s,a,s')$ implying that the difference $\bm{\wt{P}}_k-\bm{\wh{P}}_k$ also concentrates. More formally, we have the following bounds
\begin{align*}
 \bm{\nu}_k(\bm{\wt{P}}_k-\bm{\wh{P}}_k)\bm{w}_k &\leq \sum_{s,a}\nu_k(s)\widetilde{\pi}_{k}(s,a) \cdot \| \wt{p}_k(\cdot|s,a) - \wh{p}_k(\cdot|s,a)  \|_1 \cdot \| w_k\|_\infty\\
 &\leq \sum_{s,a}\nu_k(s)\widetilde{\pi}_{k}(s,a) \cdot \min \left\lbrace 2,\beta_{p,k}^{sa}\right\rbrace \cdot \frac{c}{2}
\end{align*}
and
\begin{align*}
 \bm{\nu}_k(\bm{\wh{P}}_k-\bm{P}_k)\bm{w}_k &\leq \sum_{s,a}\nu_k(s)\widetilde{\pi}_{k}(s,a) \cdot \| \wh{p}_k(\cdot|s,a)-p(\cdot|s,a)  \|_1 \cdot \| w_k\|_\infty\\
 &\leq \sum_{s,a}\nu_k(s)\widetilde{\pi}_{k}(s,a) \cdot \min \left\lbrace 2,\beta_{p,k}^{sa}\right\rbrace \cdot \frac{c}{2}
\end{align*}
where $\beta_{p,k}^{sa} = \sum_{s' \in \mathcal{S}} \beta_{p,k}^{sas'} $. The term $ \min \left\lbrace 2,\beta_{p,k}^{sa}\right\rbrace $ appears because $ \wt{p}_k(\cdot|s,a)$, $\wh{p}_k(\cdot|s,a)$ and $p(\cdot|s,a)$ are probability distributions and any two probability distributions cannot be more than $2$-far in $\ell_1$ norm.

Similarly to what we did for the reward, when summing over all episodes $k\geq 1$, we can rewrite
\begin{align*}
 \sum_{k=1}^m \sum_{s,a} \nu_k(s)\widetilde{\pi}_{k}(s,a)\min \left\lbrace 2,\beta_{p,k}^{sa}\right\rbrace = \sum_{t=1}^T \sum_{a \in \mathcal{A}_{s_t}} \widetilde{\pi}_{k_t}(s_t,a) \min \left\lbrace 2, \beta_{p,k_t}^{s_ta} \right\rbrace 
\end{align*}
Define the filtration $\mathcal{F}_t = \sigma(s_1, a_1,r_1,\dots,s_{t+1})$ and the stochastic process
\begin{align*}
 X_t = \sum_{a \in \mathcal{A}_{s_t}}\widetilde{\pi}_{k_t}(s_t,a) \min \left\lbrace 2, \beta_{p,k_t}^{s_ta} \right\rbrace  - \min \left\lbrace 2, \beta_{p,k_t}^{s_ta_t} \right\rbrace 
\end{align*}
Note that $\widetilde{\pi}_{k_t}$ is a random variable that is $\mathcal{F}_{t-1}$-measurable. Moreover, $(X_t,\mathcal{F}_t)_{t\geq 0}$ is an MDS since $|X_t| \leq 2$ and $\mathbb{E}[X_t | \mathcal{F}_{t-1}] = 0$. Using Azuma's inequality:
\begin{align*}
        \mathbb{P}\left(\sum_{t=1}^{T} \min \left\lbrace 2, \beta_{p,k_t}^{s_ta_t} \right\rbrace  \leq \sum_{t=1}^{T} \sum_{a\in \mathcal{A}_{s_t}} \wt{\pi}_{k_t}(s_t,a) \min \left\lbrace 2, \beta_{p,k_t}^{s_ta} \right\rbrace - 2\sqrt{\frac{5}{2}T\ln \left(\frac{11T}{\delta}\right)}\right) \leq \left(\frac{\delta}{11T}\right)^{5/4} < \frac{\delta}{20T^{5/4}}
\end{align*}
or in other words, with probability at least  $1-\frac{\delta}{20T^{5/4}}$:
\begin{align*}
 \sum_{k=1}^{m}\sum_{s,a} \nu_k(s)\wt{\pi}_k(s,a) \min \left\lbrace 2, \beta_{p,k}^{sa} \right\rbrace \leq \sum_{k=1}^{m}\sum_{s,a} \nu_k(s,a) \underbrace{\min \left\lbrace 2, \beta_{p,k}^{sa} \right\rbrace}_{\leq \beta_{p,k}^{sa}} + 2\sqrt{\frac{5}{2}T\ln \left(\frac{11T}{\delta}\right)}
\end{align*}
In conclusion, with probability at least  $1-\frac{\delta}{20T^{5/4}}$:
\begin{align}\label{eqn:bound_proba1}
 \sum_{k=1}^m\bm{\nu}_k(\bm{\wt{P}}_k-\bm{P}_k)\bm{w}_k \leq c\sum_{k=1}^{m}\sum_{s,a} \nu_k(s,a) \beta_{p,k}^{sa} + 2c\sqrt{\frac{5}{2}T\ln \left(\frac{11T}{\delta}\right)}
\end{align}

We now show that the remaining term $\bm{\nu}_k(\bm{P}_k-I)\bm{w}_k$ is an MDS. Let's denote by $\bm{e}_{i}$ the unit row vector with i-th coordinate 1 and all other coordinates 0.
\begin{align*}
\bm{\nu}_k(\bm{P}_k-I)\bm{w}_k &= \sum_{t=t_k}^{t_{k+1}-1} \left( \sum_{a \in \mathcal{A}_{s_t}}p(\cdot|s_t,a)\wt{\pi}_k(s_t,a) - \bm{e}_{s_t} \right)\bm{w}_k\\
&= \sum_{t=t_k}^{t_{k+1}-1} \underbrace{\left( \sum_{a \in \mathcal{A}_{s_t}}p(\cdot|s_t,a)\wt{\pi}_k(s_t,a) - \bm{e}_{s_{t+1}} \right)\bm{w}_k}_{:=X_t} + \sum_{t=t_k}^{t_{k+1}-1} \left(  \bm{e}_{s_{t+1}} - \bm{e}_{s_{t}} \right)\bm{w}_k\\
&= \sum_{t=t_k}^{t_{k+1}-1} X_t+ \underbrace{w_k(s_{t_{k+1}}) - w_k(s_{t_{k}})}_{\leq \SP{w_k} \leq c}
\end{align*}
Since $\|\bm{w}_k\|_\infty \leq \frac{c}{2}$ we have $|X_t| \leq \left(\|\sum_{a \in \mathcal{A}_{s_t}}p(\cdot|s_t,a)\wt{\pi}_k(s_t,a)\|_1 + \|\bm{e}_{s_{t+1}} \|_1 \right)\cdot \frac{c}{2} \leq c$. If we define the filtration $\mathcal{F}_t = \sigma\left(s_1,a_1,r_1,\dots, s_{t+1} \right)$ then $\mathbb{E}[X_t|\mathcal{F}_{t-1}]=0$ since $\wt{\pi}_{k_t}$ is $\mathcal{F}_{t-1}$-measurable. Using Azuma's inequality:
\begin{align*}
 \mathbb{P}\left(\sum_{t=1}^{T} \left( \sum_{a \in \mathcal{A}_{s_t}}p(\cdot|s_t,a)\wt{\pi}_k(s_t,a) - \bm{e}_{s_{t+1}} \right)\bm{w}_k \geq c \sqrt{\frac{5}{2}T\ln \left(\frac{11T}{\delta}\right)}\right) \leq \left(\frac{\delta}{11T}\right)^{5/4} < \frac{\delta}{20T^{5/4}}
\end{align*}
In conclusion, with probability at least $1-\frac{\delta}{20T^{5/4}}$:
\begin{align}\label{eqn:bound_proba2}
 \sum_{k=1}^m \bm{\nu}_k(\bm{P}_k-I)\bm{w}_k \mathbbm{1}_{M \in {\mathcal{M}}_k} \leq c \sqrt{\frac{5}{2}T\ln \left(\frac{11T}{\delta}\right)} +   m c
\end{align}
In Appendix C.2 of~\citep{Jaksch10} it is proved that given the stopping condition used for episodes, when $T \geq SA$ we can bound $m$ as $m \leq SA\log_2\left(\frac{8T}{SA}\right)$.

\subsection{Summing over episodes with $M \in {\mathcal{M}}_k$}
 
We now gather inequalities~\eqref{eqn:bound_reward},~\eqref{eqn:bound_proba1} and~\eqref{eqn:bound_proba2} into inequality~\eqref{eqn:delta_k} summed over all episodes $k$ for which $M \in {\mathcal{M}}_k$ which yields (after taking a union bound) that with probability at least $1-\frac{3\delta}{20T^{5/4}}$ (for $T \geq SA$)
\begin{align}\label{eqn:bound_contain_true_mdp}
\begin{split}
 \sum_{k=1}^m \Delta_k \mathbbm{1}_{M\in {\mathcal{M}}_k}\leq ~& c\sum_{k=1}^{m}\sum_{s,a}\nu_k(s,a)\beta_{p,k}^{sa} +  c \sqrt{\frac{5}{2}T\ln \left(\frac{11T}{\delta}\right)} +  c SA\log_2\left(\frac{8T}{SA}\right) \\ &+ 2 \sum_{k=1}^{m}\sum_{s,a} \nu_k(s,a)\beta_{r,k}^{sa} + 2\rmaxbound \sqrt{\frac{5}{2}T\ln \left(\frac{11T}{\delta}\right)} + 3 \rmaxbound \sum_{k=1}^{m}\sum_{s  \in \mathcal{S}} \frac{ \nu_k(s)}{\sqrt{t_k}} 
 \end{split}
\end{align}
The first and fourth terms appearing in the bound of Eq.~\ref{eqn:bound_contain_true_mdp} can be expanded as follows
\begin{align*}
 \sum_{k=1}^{m}\sum_{s,a}\nu_k(s,a)\beta_{r,k}^{sa} 
 &= \sum_{k=1}^{m}\sum_{s,a}\Bigg[\underbrace{\sqrt{14 \wh{\sigma}_{r,k}^2 \ln\left(\frac{2SAt_k}{\delta} \right)}}_{\leq \rmaxbound\sqrt{14 \ln\left(\frac{2SAT}{\delta} \right)}}\frac{\nu_k(s,a)}{\sqrt{\max\left\lbrace 1, N_k(s,a) \right\rbrace}} \\ 
 &\qquad\qquad\qquad{}+ \frac{49}{3}\rmaxbound \underbrace{\ln\left(\frac{2SAt_k}{\delta} \right)}_{\leq \ln\left(\frac{2SAT}{\delta} \right)} \frac{\nu_k(s,a)}{\max\left\lbrace 1, N_k(s,a)-1 \right\rbrace} \Bigg]\\
 &\leq \rmaxbound\sqrt{14 \ln\left(\frac{2SAT}{\delta} \right)} \sum_{k=1}^{m}\sum_{s,a}\frac{\nu_k(s,a)}{\sqrt{\max\left\lbrace 1, N_k(s,a) \right\rbrace}} \\
 &\qquad\qquad\qquad{}+ \frac{49}{3}\rmaxbound \ln\left(\frac{2SAT}{\delta} \right) \sum_{k=1}^{m}\sum_{s,a}\frac{\nu_k(s,a)}{\max\left\lbrace 1, N_k(s,a)-1 \right\rbrace}
\end{align*}
and similarly using the fact that $\beta_{p,k}^{sa} = \sum_{s' \in \mathcal{S}} \beta_{p,k}^{sas'} $
\begin{align*}
 \sum_{k=1}^{m}\sum_{s,a}\nu_k(s,a)\beta_{p,k}^{sa} \leq \sqrt{14 \ln\left(\frac{2SAT}{\delta} \right)} \sum_{k=1}^{m}\sum_{s,a}\frac{\nu_k(s,a)}{\sqrt{\max\left\lbrace 1, N_k(s,a) \right\rbrace}} \sum_{s' \in \mathcal{S}} \sqrt{\wh{p}_k(s'|s,a)(1-\wh{p}_k(s'|s,a))} \\ + \frac{49}{3} S \ln\left(\frac{2SAT}{\delta} \right) \sum_{k=1}^{m}\sum_{s,a}\frac{\nu_k(s,a)}{\max\left\lbrace 1, N_k(s,a)-1 \right\rbrace}
\end{align*}
By Cauchy-Schwartz inequality\footnote{The inequality obtained is somehow tight since when $\wh{p}_k(\cdot|s,a)$ is uniform on its support, it becomes an equality.}
\begin{align*}
 \sum_{s' \in \mathcal{S}} \sqrt{\wh{p}_k(s'|s,a)(1-\wh{p}_k(s'|s,a))} &= \sum_{s' \in \mathcal{S}:~ \wh{p}_k(s'|s,a)>0} \sqrt{\wh{p}_k(s'|s,a)(1-\wh{p}_k(s'|s,a))}\\
 &\leq \sqrt{\left(\sum_{s' \in \mathcal{S}:~ \wh{p}_k(s'|s,a)>0} \wh{p_k}(s'|s,a)\right)\cdot \left( \sum_{s' \in \mathcal{S}:~ \wh{p}_k(s'|s,a)>0} 1- \wh{p_k}(s'|s,a) \right)}\\
 &\leq \sqrt{\big|\supp \lbrace \wh{p_k}(\cdot|s,a) \rbrace\big| - 1}
\end{align*}
where $\supp \lbrace \wh{p}_k(\cdot|s,a) \rbrace =  \lbrace s' \in \mathcal{S}: ~ \wh{p}_k(s'|s,a)>0 \rbrace$ is the support of $\wh{p}_k(\cdot|s,a)$ and $|\cdot|$ denotes the cardinal of a set. Note that by definition of $\wh{p}_k$, $\supp \lbrace \wh{p}_k(\cdot|s,a) \rbrace \subseteq \supp \lbrace p(\cdot|s,a) \rbrace$ and so $\big|\supp \lbrace \wh{p}_k(\cdot|s,a) \rbrace\big| \leq \big|\supp \lbrace {p}(\cdot|s,a) \rbrace\big|$. Let's denote by $\nextstates$ the maximal support over all state-action pairs $(s,a)$:
\begin{align*}
 \nextstates = \max_{s,a \in \mathcal{S} \times \mathcal{A}}  \big|\supp \lbrace {p}(\cdot|s,a) \rbrace\big|
\end{align*}
Therefore
\begin{align*}
 \sum_{k=1}^{m}\sum_{s,a}\nu_k(s,a)\beta_{p,k}^{sa} \leq \sqrt{14 (\nextstates - 1) \ln\left(\frac{2SAT}{\delta} \right)} \sum_{k=1}^{m}\sum_{s,a}\frac{\nu_k(s,a)}{\sqrt{\max\left\lbrace 1, N_k(s,a) \right\rbrace}} \\ + \frac{49}{3} S\ln\left(\frac{2SAT}{\delta} \right) \sum_{k=1}^{m}\sum_{s,a}\frac{\nu_k(s,a)}{\max\left\lbrace 1, N_k(s,a)-1 \right\rbrace}
\end{align*}
As proved by \citet[Sections 4.3.1 and 4.3.3]{Jaksch10}, the stopping condition used for episodes implies that \[\sum_{k=1}^{m}\sum_{s  \in \mathcal{S}} \frac{ \nu_k(s)}{\sqrt{t_k}} = \sum_{k=1}^{m}\sum_{s,a} \frac{\nu_k(s,a)}{\sqrt{t_k}} \leq \sum_{k=1}^{m}\sum_{s,a} \frac{\nu_k(s,a)}{\sqrt{\max\left\lbrace 1, N_k(s,a) \right\rbrace}} \leq \left(\sqrt{2} +1 \right)\sqrt{SAT}\]
Finally we bound the term
\begin{align*}
 \sum_{k=1}^{m}\sum_{s,a} \frac{\nu_k(s,a)}{\max\left\lbrace 1, N_k(s,a)-1 \right\rbrace} = \sum_{s,a}\sum_{t=1}^{T} \frac{\mathbbm{1}_{\lbrace (s_t, a_t) =(s,a)\rbrace}}{\max\left\lbrace 1, N_{k_t}(s,a)-1 \right\rbrace}
\end{align*}
The stopping condition of episodes ensures that for all $t\geq 1$, $N_t(s,a)\leq 2N_{k_t}(s,a)$ where $N_t(s,a)$ is the total number of visits in $(s,a)$ before time $t$, $t$ \emph{not} included:
\begin{align*}
 N_t(s,a) = \# \left\lbrace 1 \leq \tau < t: (s_\tau, a_\tau) =(s,a) \right\rbrace
\end{align*}
Therefore, similarly to what is done in~\citep[Proof of Lemma 5]{Ouyang2017learning}
\begin{align}
\notag  \sum_{k=1}^{m}\sum_{s,a} \frac{\nu_k(s,a)}{\max\left\lbrace 1, N_k(s,a)-1 \right\rbrace} 
\notag &\leq 2\sum_{s,a}\sum_{t=1}^{T} \frac{\mathbbm{1}_{\lbrace (s_t, a_t) =(s,a)\rbrace}}{\max\left\lbrace 1, N_{t}(s,a)-1 \right\rbrace}\\
&= 2 \sum_{s,a} \Bigg[ \mathbbm{1}_{\lbrace N_{T+1}(s,a) \geq 1 \rbrace} + \mathbbm{1}_{\lbrace N_{T+1}(s,a) \geq 2 \rbrace} + \underbrace{\sum_{j=2}^{N_{T+1}(s,a)-1}\frac{1}{j-1}}_{\leq 1 + \ln\left( N_{T+1}(s,a) \right) \mathbbm{1}_{\lbrace N_{T+1}(s,a) \geq 1 \rbrace}} \Bigg]\\
 \label{subeq:harmonic.series}&\leq 6SA + 2\sum_{s,a}\ln\left( N_{T+1}(s,a) \right) \mathbbm{1}_{\lbrace N_{T+1}(s,a) \geq 1 \rbrace}\\
 \label{subeq:jensen}&\leq 6SA +2 SA \ln\left( \frac{\sum_{s,a}N_{T+1}(s,a) \mathbbm{1}_{\lbrace N_{T+1}(s,a) \geq 1 \rbrace} }{\sum_{s,a} \mathbbm{1}_{\lbrace N_{T+1}(s,a) \geq 1 \rbrace} } \right) \leq 6SA + 2SA\ln(T)
\end{align}
where~\eqref{subeq:harmonic.series} follows from the rate of divergence of an harmonic series and~\eqref{subeq:jensen} is derived by applying Jensen's inequality to the concave function $\ln(\cdot)$ in the penultimate inequality (with a normalization factor $\sum_{sa} \mathbbm{1}_{\lbrace N_{T+1}(s,a) \geq 1 \rbrace} \leq SA$).

In conclusion, for $T\geq SA$, with probability at least $1-\frac{3\delta}{20T^{5/4}}$
\begin{align}\label{eqn:successful_episodes} 
 \begin{split}
 \sum_{k=1}^m \Delta_k \mathbbm{1}_{M\in {\mathcal{M}}_k}\leq &\left(\sqrt{28}+\sqrt{14} \right)  c\sqrt{(\nextstates - 1) SAT\ln\left(\frac{2SAT}{\delta} \right)} + \frac{98}{3} c S^2A\ln\left(\frac{2SAT}{\delta} \right)(3+\ln(T))\\
 &+2\left(\sqrt{28}+\sqrt{14} \right) \rmaxbound \sqrt{ SAT\ln\left(\frac{2SAT}{\delta} \right)} + \frac{196}{3}\rmaxbound SA\ln\left(\frac{2SAT}{\delta} \right)(3+\ln(T))\\
 &+  (3c + 2\rmaxbound) \sqrt{\frac{5}{2}T\ln \left(\frac{11T}{\delta}\right)} + c SA\log_2\left(\frac{_T}{SA}\right)  + 3(\sqrt{2}+1)\rmaxbound \sqrt{SAT}
 \end{split}
\end{align}

\subsection{Completing the regret bound}

From~\eqref{eqn:splitting} we have that with probability at least $1-\frac{\delta}{20T^{5/4}}$
\begin{align}\label{eqn:splitting2}
 \Delta(\scal,T) \leq  \sum_{k=1}^m \Delta_k \mathbbm{1}_{M \in {\mathcal{M}}_k} + \sum_{k=1}^m \Delta_k \mathbbm{1}_{M \not\in {\mathcal{M}}_k} + \rmaxbound \sqrt{\frac{5}{2}T\ln \left(\frac{11T}{\delta}\right)}
\end{align}
By gathering~\eqref{eqn:failed_episodes} and~\eqref{eqn:successful_episodes} into~\eqref{eqn:splitting2} (using a union bound) we have that with probability at least $1 - \frac{\delta}{20T^{5/4}} - \frac{3\delta}{20 T^{5/4}} - \frac{\delta}{20T^{4/5}} = 1-\frac{\delta}{4T^{5/4}}$ (for $T\geq SA$)
\begin{align}\label{eqn:final_bound_expanded} 
 \begin{split}
 \Delta(\scal,T)\leq &\left(\sqrt{28}+\sqrt{14} \right) c\sqrt{(\nextstates - 1) SAT\ln\left(\frac{2SAT}{\delta} \right)} + \frac{98}{3} c S^2A\ln\left(\frac{2SAT}{\delta} \right)(3+\ln(T))\\
 &+2\left(\sqrt{28}+\sqrt{14} \right) \rmaxbound \sqrt{ SAT\ln\left(\frac{2SAT}{\delta} \right)} + \frac{196}{3}\rmaxbound SA\ln\left(\frac{2SAT}{\delta} \right)(3+\ln(T))\\
 &+ 3 (c + \rmaxbound) \sqrt{\frac{5}{2}T\ln \left(\frac{11T}{\delta}\right)} + c SA\log_2\left(\frac{8T}{SA}\right)  + 3(\sqrt{2}+1)\rmaxbound \sqrt{SAT} + {\rmaxbound} \sqrt{T}
 \end{split}
\end{align}
For $T\leq SA$ the regret can be bounded with probability $1$ as
\begin{align*}
 \Delta(\scal,T)\leq \rmaxbound T = \rmaxbound \sqrt{T} \cdot \sqrt{T} \leq \rmaxbound \sqrt{SAT}
\end{align*}
Finally, we take a union bound over all possible values of $T$ and use the fact that $\sum_{T= 2}^{+\infty}\frac{\delta}{4T^{5/4}}<\delta$.

In conclusion, there exists a numerical constant $\alpha$ such that for any MDP $M$, with probability at least $1-\delta$ our algorithm \scal has a regret bounded by
\begin{align}\label{eqn:final_bound_simplified} 
 \Delta(\scal,T) \leq \alpha\cdot\left( \max{\left\lbrace \rmaxbound,c \right\rbrace} \sqrt{\nextstates S A T \ln\left(\frac{T}{\delta}\right)} + \max{\left\lbrace \rmaxbound,c \right\rbrace} S^2A\ln^2\left(\frac{T}{\delta} \right) \right)
\end{align}
The second term of the upper-bound in Eq.~\ref{eqn:final_bound_simplified} is negligible when $T$ is big enough and so
\begin{align*}
        \mathbb{P}\left( \Delta(\scal,T) = \mathcal{O} \left( \max{\left\lbrace \rmaxbound,c \right\rbrace} \sqrt{\nextstates S A T \ln\left(\frac{T}{\delta}\right)}\right) \right) \geq 1-\delta
\end{align*}

\section{Additional Experiments}\label{app:experiments}
In this section we provide clearer figures for the three-states domain and we present a more challenging domain: Knight Quest.

\subsection{Three-States MDP}
We simply restate the results presented in the main paper on bigger figures (see Fig.~\ref{f:3d_delta_big} and~\ref{f:3d_delta_zero_big}).
\begin{figure}[t]
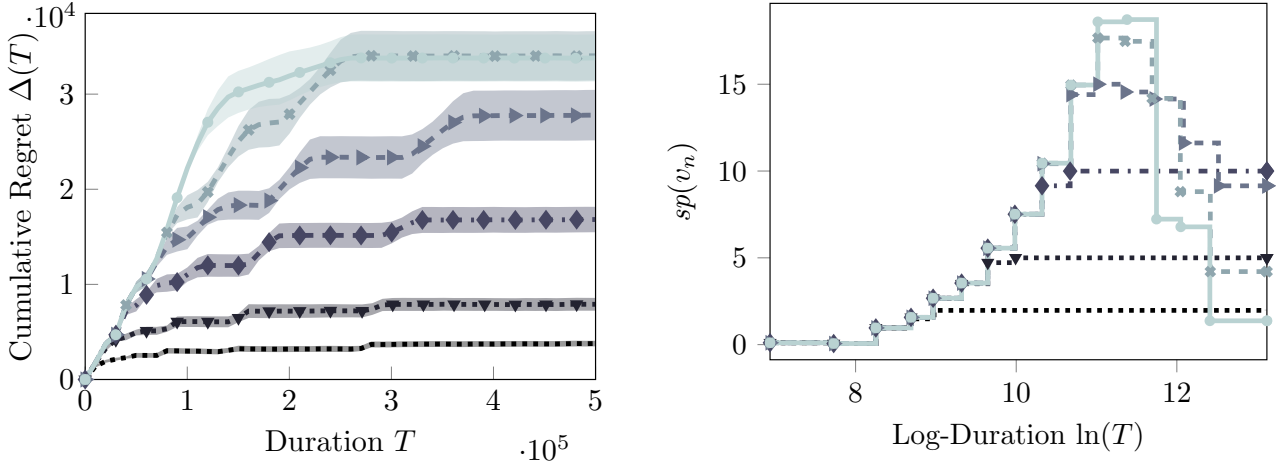

        \centering
        \includegraphics[height=6.1cm]{REGRET_Toy3D_BERN_pure.pdf}\hfill
        \includegraphics[height=6.1cm]{SPAN_Toy3D_BERN_pure.pdf}
        \caption{\ucrl and \scal behaviour with $\delta=0.005$ in the three-states MDP.}
        \label{f:3d_delta_big}
\end{figure}
\begin{figure}[t]
        \centering
        \includegraphics[height=6.1cm]{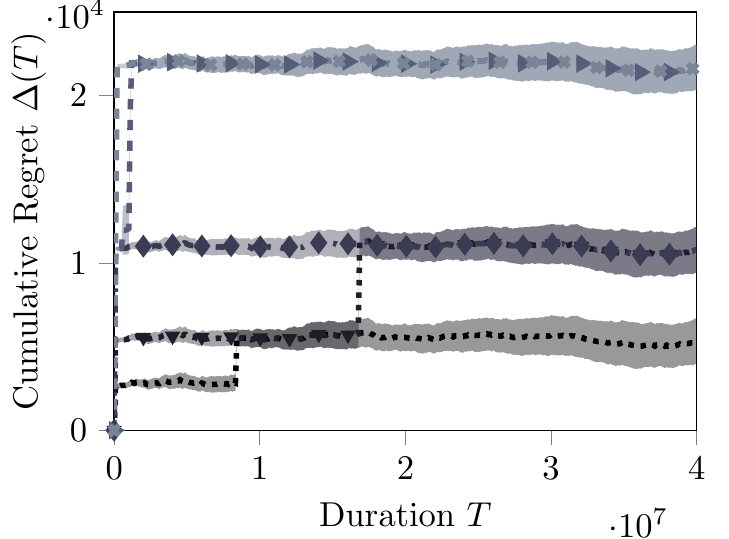}\hfill
        \includegraphics[height=6.1cm]{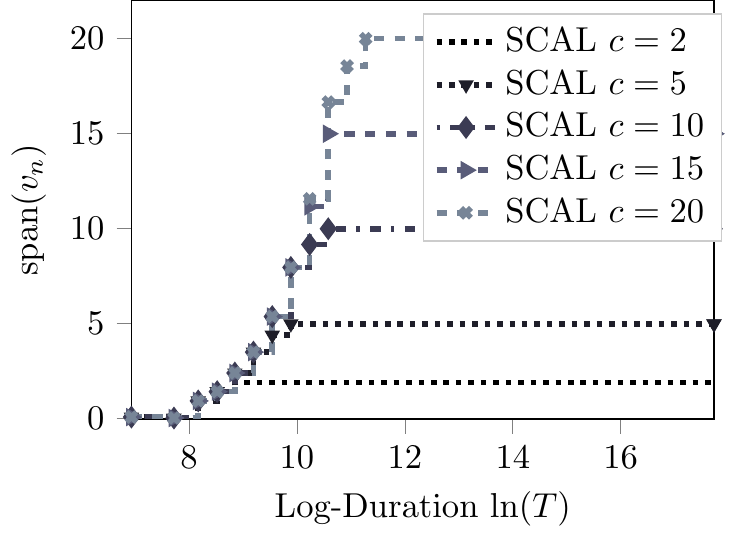}
        \caption{\scal behaviour with $\delta=0$ in the three-states MDP. In this setting, the MDP is weakly communicating and \ucrl is not able to learn. We have omitted \ucrl since it is out-of-scale (see Fig.~\ref{F:3dresults}).}
        \label{f:3d_delta_zero_big}
\end{figure}

\subsection{Knight Quest}
The second environment takes inspiration from classical arcade games. The goal is to rescue a princess in the shortest time without being killed by the dragon.
To achieve this task, the knight needs to collect gold, buy the magic key and reach the princess location.
A representation of the environment is provided in Fig.~\ref{F:knightenv}.

The elements of the game are: I) the knight; II) the princess; III) a dragon patrolling the princess; IV) a gold mine and V) a town.

\textbf{Town, Princess and Gold Mine.} These elements are special states of the environment. The town (T) is the place where the knight can buy objects and where it is reset when he rescues the princess or he is killed by the dragon. The princess (P) is the terminal state, while the gold mine (G) is the place where the knight can collect gold.

\textbf{Dragon.} The dragon (D) is the enemy and it is randomly moving around the princess's location. Let's denote with $d \in \{0,1,2\}$ the position of the dragon such that: $d= 0 := (0,1)$, $d = 1 := (1,0)$ and $d = 2 := (1,1)$. The transition probabilities of the dragon are:
\begin{align*}
        p_{d}(\cdot|0) = [0.4, 0, 0.6]; \quad
        p_{d}(\cdot|1) = [0, 0.4, 0.6]; \quad
        p_{d}(\cdot|2) = [0.4, 0.2, 0.4].
\end{align*}
When the dragon can kill the knight when they are at the same position and the knight does not have the armour.

\begin{wrapfigure}{r}{0.5\textwidth}
        \centering
        \begin{tikzpicture}
                \draw[step=1.5cm,gray,very thin] (0,0) grid (6,-6.0001);
                \filldraw[fill=Lavender,draw=gray, very thin, fill opacity=0.1] (0,0) rectangle (1.5,-1.5);                
                \node[Lavender] at (0.75,-0.75) {\huge P};
                \filldraw[fill=Gray,draw=gray, very thin, fill opacity=0.1] (4.5,0) rectangle (6,-1.5);                
                \node[Gray] at (5.25,-0.75) {\huge T};
                \filldraw[fill=Goldenrod,draw=gray, very thin, fill opacity=0.1] (1.5,-4.5) rectangle (3,-6);                
                \node[Orange] at (2.25,-5.25) {\huge G};
                \filldraw[fill=ForestGreen,draw=gray, very thin, fill opacity=0.1] (1.5,0) rectangle (3,-3);                
                \filldraw[fill=ForestGreen,draw=gray, very thin, fill opacity=0.1] (0,-1.5) rectangle (1.5,-3);                
                \node at (2.25,-2.25) {\color{ForestGreen}\huge D};
                \draw[<->,dashed, ForestGreen] (0.5, -2.25) -- (1.75,-2.25);
                \draw[<->,dashed, ForestGreen] (2.25, -0.5) -- (2.25,-1.75);
                \foreach \x in {0,1,2,3}
                    \node at ($(1.5*\x cm, 0.2)+(0.75,0)$) {$\x$};
                \foreach \x in {0,1,2,3}
                    \node at ($(-0.2, -1.5*\x cm)+(0, -0.75)$) {$\x$};
        \end{tikzpicture}
        \caption{Representation of the Knight Quest $4\times 4$ map. The green shadowed cells are the locations where the dragon can move.}
        \label{F:knightenv}
\end{wrapfigure}
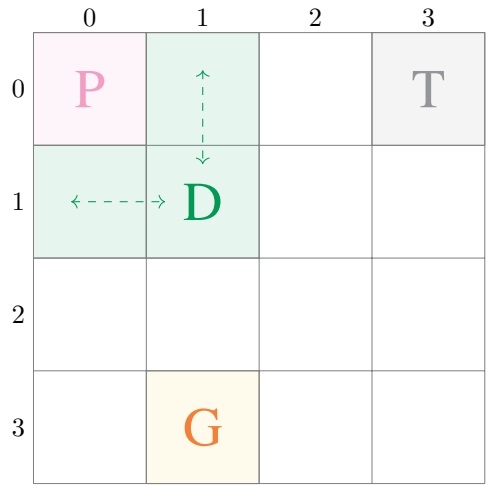        

\textbf{Knight.} The knight is the only player of the game.
He moves in the environment using the four cardinal actions (\ie, \emph{right}, \emph{down}, \emph{left} and \emph{up}) plus an action to keep the current position (\emph{stay}). We refer to these $5$ actions as \emph{movement actions}.
Additionally, the knight can collect the gold (action \emph{CG}), buy a key (action \emph{BK}) or buy an armour (action \emph{BA}).

\textbf{State representation, action effect and reward.}
The state $s_t$ of the game is represented by the following elements:
\begin{itemize}
        \item Knight position: coordinates of the grid $(row, col)$, $row,col \in {0,1,2,3}$;
        \item Gold level: the amount of gold own by the knight, $g \in \{0,1\}$;
        \item Dragon position: $d \in \{0,1,2\}$;
        \item Object identifier: the object(s) own by the knight, $o = \{0,1,2,3\}$ where $0:=$ nothing, $1:=$ key, $2 :=$ armour and $3 := $ key and armour.
\end{itemize}
Now we can finally explain the effects of the actions, \ie how states $s_{t+1}$ is generated.
The movement actions have the trivial effect of changing the knight position.
The action CG changes the state only when the knight is at the mine.
In this case the level of gold is incremented by one, formally $g_{t+1} = \min\{1, g_{t} + 1\}$.
Actions BK and BA alter the state only when are executed in the town with gold-level equal to $1$, \ie
\begin{align*}
        a_t = BK \wedge (row_t,col_t) = T \wedge g_t = 1  \implies o_{t+1} = \begin{cases}
        1 & \text{if } o_t = 0\\
        3 & \text{otherwise}
\end{cases}\\
a_t = BA \wedge (row_t,col_t) = T \wedge g_t = 1  \implies o_{t+1} = \begin{cases}
        2 & \text{if } o_t = 0\\
        3 & \text{otherwise}
\end{cases}
\end{align*}

All the actions are deterministic when the knight does not own the armour.
When the knight has the armour:
\begin{itemize}
        \item The movement actions result in a normal (correct) transition with probability $0.5$, otherwise the current position is kept;
        \item The CG action fails with probability $0.99$, \ie with probability $0.01$ the gold level is increment by $1$.
        \item Actions BK and BA are not modified.
\end{itemize}
Notice that when the knight is equipped with the armour it cannot be killed by the dragon (\ie knight and dragon can occupy the same cell).
However, due to the weight of the armour, knight's gait is unsteady. At the same time, the armour makes the collection of the goal very challenging (\ie success probability is $0.01$). You can imagine that mining with a metal armour can be very difficult!

The basic reward signal is $-1$ at each time step. Nevertheless, the knight receives a reward of $-10$ when he executes CG, BK or BA outside the designed location (\ie mine and town). Finally, he obtains a reward of $20$ when he reaches the princess with the key and $-20$ when he is killed by the dragon (\ie knight and dragon are in the same cell and the knight does not have the armour). For the experiments, we have scaled the reward in the range $[0,1]$.

Finally, when the episode ends (\ie the knight reaches the princess with the key or he is killed), the knight is reset at town location with no gold or object ($g,o=0$) and the dragon position is randomly drawn ($d \sim \mathcal{U}(\{0,1,2\})$).

\textbf{Properties of the game.}
The state and action space size are $S=360$ and $A=8$, while the diameter of the MDP is $D \approx 250$.
The diameter is due to the following path: start from the town with no gold but the armour and reach the princess with one unit of gold, key and armour.
However, the optimal strategy is simply to collect the gold, buy the magic key from the town and rescue the princess.\footnote{Notice that there are deterministic strategies to the princess's location avoiding the dragon.}
The optimal policy is such that $g^* \approx 0.5$, $\SP{h^*} \approx 3.28$.

This game is challenging for OFU approaches since the policy suffering the diameter is orthogonal to the optimal one. 
This is due to the presence of actions that are not relevant to the final objective and simply mess up the navigability of the environment.
We think that this characteristic is shared by common real-world applications where the agent can face several choices (actions) and most of them are useless.
This property can be interpreted also as a hierarchical structure that has been proved to be at the core of many applications.

More practically, the high diameter induces \ucrl to explore remote states that are seen as promising states (in contrast with the real importance).
On contrary, \scal can leverage on the knowledge of ``simple'' game (\ie small span) in order to condition the exploration.
Notice that the span constraint $c$ can be interpreted as the difficulty of the game. This game becomes difficult only if I want reach extreme states that are
nevertheless useless to the final goal (rescue the princess). By giving small $c$ values to the algorithm, we are implicitly saying that the game is simple, do not trust states that are too promising (\ie generate a high span).

\textbf{Results.}
We have tested \ucrl and \scal with different constraints over an horizon $T = 4 \cdot 10^8$ with Bernstein's bound.
The code is available on GitHub (\url{https://github.com/RonanFR/UCRL}).
\scal is run with the reward augmented but no perturbation of the transition matrix ($\eta_k =0$). For the terminal condition of \regopt we set $\gamma_k = 0$.\footnote{Note that in EVI and \regopt the optimistic reward is truncated at $\rmaxbound$, \ie $\max_{\wt{r}}\{\wt{r}(s,a)\} = \min\left\{\rmaxbound, \wh{r}(s,a) + \max \{B^r(s,a)\}\right\}$.}
In order to speed up the learning we have set $\alpha_p = \alpha_r = 0.05$ (see Eq.~\ref{eqn:cb_proba} and~\ref{eqn:cb_reward}).
This still guarantees that the confidence intervals at $t_0$ are still bigger than $1$ and $\rmaxbound$, respectively.
Results are reported in Fig.~\ref{F:KQ_regret_span}.
We can notice that \scal is able to outperform \ucrl by a big margin.
It is interesting to notice that in the regret curves it is easy to identify the linear and logarithmic regimes, while the square-root one is almost absent.
This is due to the fact that once the algorithms discover that visiting extreme states is not relevant they have almost perfectly learnt the dynamics under the optimal policy.\footnote{The actions executed by the optimal policy are deterministic. This means that by using Bernstein's bound we have the term involving the variance equal to zero and the second term scales linearly with the number of visits.}

\begin{figure}[t]
 \centering
 \includegraphics[height=6.1cm]{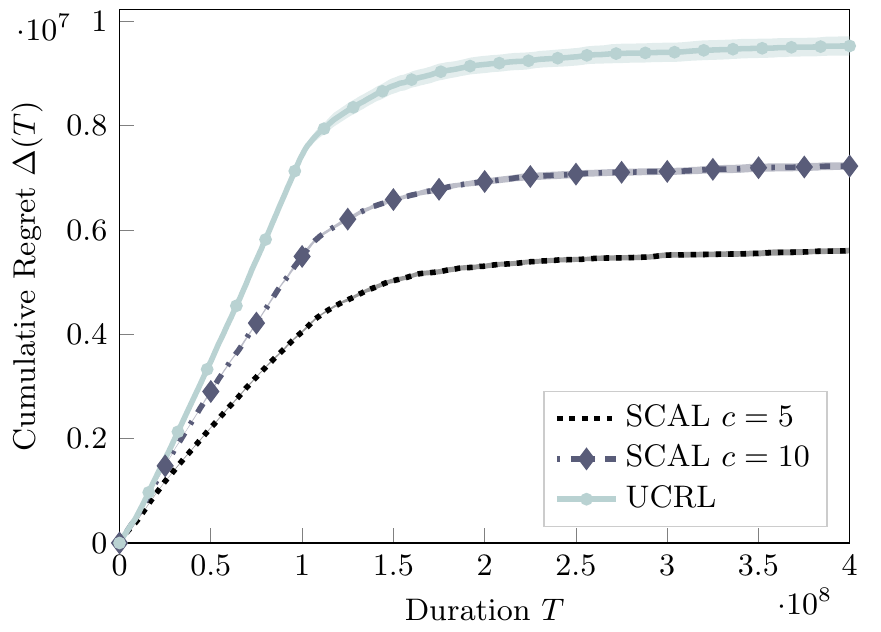}\hfill
 \includegraphics[height=6.1cm]{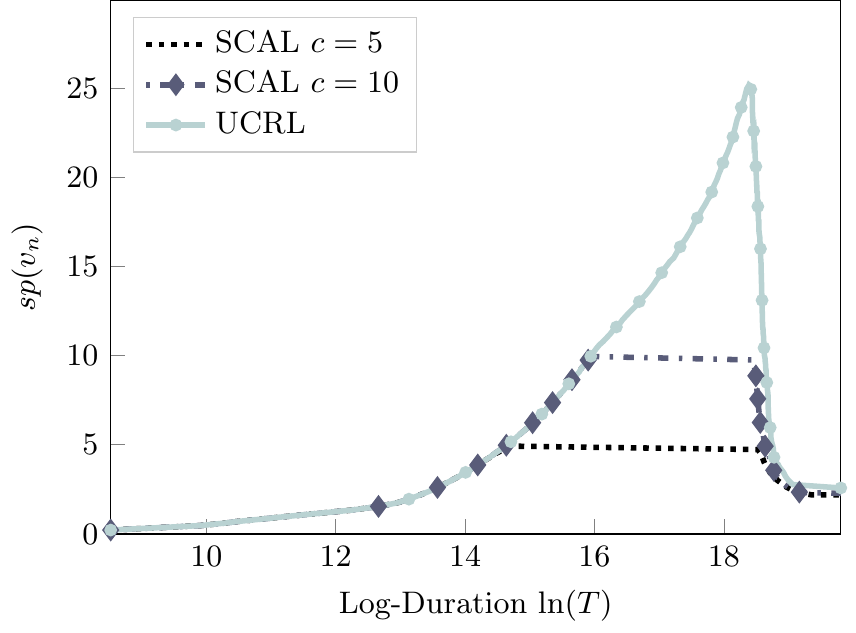}
 \caption{ Algorithm behaviour in the knight quest game.
Figures show the span of the optimistic bias \textit{(right)} and the cumulative regret \textit{(left)} as a function of $T$.
            Results are averaged over $15$ runs and $95\%$ confidence intervals of the mean are shown for the regret.
}
 \label{F:KQ_regret_span}
\end{figure}

\end{document}